\documentclass[twoside,11pt]{article}

\usepackage{jmlr2e}

\hypersetup{hidelinks}
\usepackage{color}

\usepackage[utf8]{inputenc}

\usepackage{bm}
\usepackage{relsize}

\usepackage{bm}
\usepackage{cancel}
\usepackage{caption}
\usepackage{subcaption}
\usepackage{amsmath}

\usepackage{cleveref} 

\DeclareMathOperator{\sinc}{sinc}

\newcommand{\rvs}[1]{\textcolor{black}{#1}}

\newcommand\qed{\BlackBox}
\newcommand\qedhere{\BlackBox}

\usepackage{lastpage}

 \jmlrheading{25}{2024}{1-\pageref{LastPage}}{1/23; Revised 1/24}{3/24}{23-0044}{Ayoub Belhadji and Rémi Gribonval}
\ShortHeadings{Revisiting RIP Guarantees for Sketching Operators on Mixture Models}{Belhadji and Gribonval}

\firstpageno{1}

\begin{document}

\title{Revisiting RIP Guarantees\\ for Sketching Operators on Mixture Models}
\author{\name Ayoub Belhadji \email ayoub.belhadji@gmail.com\\
\name Rémi Gribonval \email remi.gribonval@inria.fr\\
\addr Univ Lyon, ENS de Lyon,  Inria, CNRS, UCBL, LIP UMR 5668, Lyon, France}

\editor{Florence d'Alché-Buc}

 \maketitle

\begin{abstract}%
In the context of sketching for compressive mixture modeling, we revisit existing proofs
 of the Restricted Isometry Property of sketching operators with respect to certain mixtures models.
 After examining the shortcomings of existing guarantees, we propose an alternative analysis that circumvents the need to assume importance sampling when drawing random Fourier features to build random sketching operators.
 Our analysis is based on new deterministic bounds on the restricted isometry constant that depend solely on the set of frequencies used to define the sketching operator; then we leverage these bounds to establish concentration inequalities for random sketching operators that lead to the desired RIP guarantees. 
 Our analysis also opens the door to theoretical guarantees for structured sketching with frequencies associated to fast random linear operators.
 \end{abstract}




\section{Introduction}

Building up linear operators that preserve the distances between two sets is at heart of many problems in the field of inverse problems. The fetch for such linear operators gave birth to a rich literature at the intersection of signal processing and machine learning  \citep{Ach2000,BiMa01,Sar06,AiCh06,Can08,MaMu09,FoRa13}. Recently, a new family of inverse problems emerged in the context of compressive learning, also called sketched learning \citep{KeBoGrPe18,GrBlKeTr21,GrBlKeTr20}. These inverse problems are tailored to be used in the field of mixture modeling. In a nutshell, sketched learning is a paradigm aiming to scale up these learning tasks by conducting the learning task on a low dimensional vector, also called a \emph{sketch}, that contains a "\emph{gist}" of the initial dataset and that is suited for a specific learning task: \emph{sketching} is the procedure that outputs the sketch for a given dataset.
In practice, sketching boils down to embed a probability distribution $\pi$ typically on $\mathcal{X} = \mathbb{R}^{d}$ into $\mathbb{C}^{m}$ by considering a \emph{sketching operator} $\mathcal{A}$ such that\footnote{Integrability is treated informally in this introduction and will be more formally discussed in Section~\ref{sec:skop}.}
\begin{equation}\label{eq:DefSketchingOperator}
        \mathcal{A}\pi :=         \int_{\mathcal{X}}\Phi(x)\mathrm{d}\pi(x) \in  \mathbb{C}^m
\end{equation}
where $\Phi$ is a $\mathbb{C}^{m}$-valued function defined on $\mathcal{X}$ called the \emph{feature map}. As shown
 in \citep{GrBlKeTr21}, building up the linear operator $\mathcal{A}$ is indirectly constrained by the targeted learning task (e.g.: $k$-means clustering, or Gaussian Mixture Modeling). This constraint can be expressed using the Maximum Mean Discrepancy (MMD) \citep{GrBoRaScSm12} defined as follows: considering a positive definite kernel\footnote{In the rest of this paper when we write  \emph{kernel} we implicitly assume a positive definite kernel.}  \citep{BerTh11}  $\kappa: \mathcal{X} \times \mathcal{X} \rightarrow \mathbb{R}$, the MMD with respect to $\kappa$ between two probability distributions $\pi$ and $\pi'$ on $\mathcal{X}$ is defined using the norm $\|\cdot\|_{\kappa}$ on the reproducing kernel Hilbert space 
 (RKHS) $\mathcal{H}$ associated to $\kappa$ by
\begin{equation}\label{eq:MMD}
\|\pi - \pi'\|_{\kappa} := \sup\limits_{f \in \mathcal{H}, \|f\|_{\kappa} \leq 1} \bigg| \int_{\mathcal{X}} f(x) \mathrm{d}\pi(x)- \int_{\mathcal{X}} f(x) \mathrm{d}\pi'(x) \bigg|.
\end{equation}
Equiped with the MMD, the general theory of \citep{GrBlKeTr21} suggests to
look for sketching operators that satisfy\footnote{\rvs{The astute reader will notice that \eqref{eq:MMD_preservation} is analog to the usual RIP property in compressive sensing \cite{FoRa13}, see e.g. \citep{GrChKeScJaSc21}.}}
\begin{equation}\label{eq:MMD_preservation}
\forall \pi,\pi' \in \mathfrak{G}, \:\: (1-\delta)\|\pi-\pi'\|_{\kappa}^{2} \leq \|\mathcal{A}\pi - \mathcal{A}\pi'\|_{2}^2 \leq (1+\delta)\|\pi-\pi'\|_{\kappa}^2,
\end{equation}
where $\mathfrak{G}$ is a particular set of probability measures on $\mathcal{X}$ and $\delta \in [0,1)$. 

When the kernel $\kappa$ is shift-invariant, it is possible to build up a sketching operator satisfying~\eqref{eq:MMD_preservation} by considering random Fourier features \citep{ZhMe15,KeBoGrPe18,GrBlKeTr20}. Initially, random Fourier features were introduced to scale up kernel methods \citep{RaRe07}. This family of approximations is suitable for a shift-invariant kernel $\kappa$ for which Bochner's theorem \citep{Wen04,Rud17} holds:
\begin{equation}\label{eq:bochner_decomposition}
\forall x,y \in \mathcal{X} = \mathbb{R}^{d}, \:\: \kappa(x,y) = \int_{\mathbb{R}^d}  e^{\iota \omega (x-y)} \hat{\kappa}(\omega) \mathrm{d}\omega,
\end{equation}
with $\hat{\kappa}$ a non-negative function.
Although the framework of \citep{GrBlKeTr21} holds for more general kernels, the focus of this paper is indeed on shift-invariant kernels, for which we have $\kappa(x,x) = \kappa(0,0) = \int_{\mathbb{R}^d} \hat{\kappa}(\omega) \mathrm{d}\omega$ for every $x \in \mathcal{X} = \mathbb{R}^{d}$. Moreover,  we often simplify the analysis by assuming a normalized kernel, {\em i.e.}, $\kappa(0,0) = 1$. The results are easily extended to the non-normalized case. With the normalization assumption, the function $\hat{\kappa}$ satisfies $\int_{\mathbb{R}^d} \hat{\kappa}(\omega) \mathrm{d}\omega =1$ and can be interpreted as a probability density function on the \emph{frequency vector} $\omega \in \mathbb{R}^{d}$. Based on the identity~\eqref{eq:bochner_decomposition}, the \emph{random Fourier feature map}
is constructed as follows:  let $\omega_1, \dots, \omega_m \in \mathbb{R}^{d}$ to be  i.i.d. random variables with probability density $\hat{\kappa}(\omega)$ with respect to the Lebesgue measure on $\mathbb{R}^d$, and define the random feature map 
\begin{equation}
\label{eq:DefGenericRFMap}
\Phi(x) := \frac{1}{\sqrt{m}} (\phi_{\omega_i}(x))_{i \in [m]} \in \mathbb{C}^{m},
\end{equation}
where here $\phi_{\omega}(x) = e^{\imath \omega^{\top} x} \in \mathbb{C}$ for each $\omega \in \mathbb{R}^{d}$, and for any integer $n \in \mathbb{N}$ we denote $[n] := \{1,\ldots,n\}$.
With this design, the (random) empirical kernel $\kappa_{\Phi}(x,y) := \langle \Phi(x),\Phi(y)\rangle$ (where for any $u,v \in \mathbb{C}^{m}$ $\langle u,v\rangle := u^{\top}\overline{v}$, with $\overline{v}$ the complex conjugate of $v$) satisfies for every pair of vector $ x,y \in \mathbb{R}^d$
\begin{equation}\label{eq:ExpectedKernelProp}
\mathbb{E} \kappa_{\Phi}(x,y)  = \kappa(x,y).
\end{equation}
The study of the approximation $\kappa_{\Phi}(x,y) \approx \kappa(x,y)$ is a well established area of research. Indeed, since the publication of~\citep{RaRe07}, many works followed tackling various aspects of this class of approximations.
In particular, sharp uniform error bounds on compact sets for RFF were derived in~\citep{SuSc15,SrSz15}, and connexions with kernel-based quadrature were established in~\citep{Bac17}. Moreover, the quest to various designs of the frequencies were proposed: frequencies based on quasi-Monte Carlo sequences~\citep{YaSiAvMa14}, frequencies based on structured matrices~\citep{LeSaSm13,ChSi16,ChRoSaSiTuWe18}. We refer the reader to~\citep{LiHuChSu21} for an exhaustive review on this topic.

As it was shown in \citep{SrGrFuScLa10}, the identity~\eqref{eq:ExpectedKernelProp} somehow extends to pairs of probability distributions $\pi,\pi'$ on $\mathcal{X}$
\begin{equation*}\label{eq:expected_MMD_MMD}
\mathbb{E} \| \pi-\pi' \|_{\kappa_\Phi}^2 = \| \pi-\pi' \|_{\kappa}^2.
\end{equation*}
This formula was crucial in the study of \emph{characteristic} kernels carried in \citep{SrGrFuScLa10}. The study of the fluctuations of $\| \pi-\pi' \|_{\kappa_\Phi}^2$ around its expected value $\| \pi-\pi' \|_{\kappa}^2$ was carried in~\citep{ZhMe15} and~\citep{SuSc15}. Nevertheless, these results do not imply~\eqref{eq:MMD_preservation}: the established guarantees  have an additive form $|\| \pi-\pi' \|_{\kappa_\Phi}^2 - \| \pi-\pi' \|_{\kappa}^2| \leq t$ for a given pair $(\pi,\pi')$ of probability distributions, while~\eqref{eq:MMD_preservation} have a multiplicative form and holds uniformly on $\mathfrak{G} \times \mathfrak{G}$. 
A study of conditions under which~\eqref{eq:MMD_preservation} holds was undertaken in \citep{GrBlKeTr20} with a 
focus on the case where $\mathfrak{G} = \mathfrak{G}_k$ is a set of mixtures with $k$ components that depend on  parameters that belong to $\mathbb{R}^d$. In this context, it was shown that \eqref{eq:MMD_preservation} holds with high probability provided that:
\begin{enumerate}
\item the sketch dimension, $m$, is large enough: a sufficient sketch size was proved to satisfy $m = \mathcal{O}(k^2d)$ (up to logarithmic factors); and
\item a variant of the RFF \eqref{eq:DefGenericRFMap} is used, with an appropriate \emph{importance sampling scheme}.
\end{enumerate}
The dependency of this provably good sketch size in $k$ and $d$ does not match the empirical simulations carried out (\emph{without} importance sampling) in \citep{KeTrTrGr17}, which suggest that \eqref{eq:MMD_preservation} can hold with high probability for some $m = \mathcal{O}(kd)$ (again, possibly up to logarithmic factors). 

To bridge the gap between theory and practice, one would ideally like to both get rid of the importance sampling assumption, which does not seem needed in practice, and to achieve guarantees for sketch sizes with a better dependency in $k$. 
In this work, we revisit the analysis of \citep{GrBlKeTr20} with two main contributions: 
\begin{itemize}
\item we prove that $m = \mathcal{O}(k^{2}d)$ remains a provably good size \emph{even without importance sampling}; 
\item we explain why the high level structure of the proof technique of  \citep{GrBlKeTr20}, as well as the structure of our new approach, prevents them from achieving better dependencies in $k$ of provably good sketch sizes. 
\end{itemize}
As we shall see later, our analysis is based on tools inspired from the literature of sparse recovery with incoherent dictionaries~\citep[see][Chapter 5]{FoRa13}. This quadratic dependency on number $k$ of mixture components, which plays the role of a  sparsity level, is thus not surprising. Whether the ``right'' dependency in $k$ of provably good sketch sizes remain an open question.

These contributions are obtained by introducing several technical ingredients. First, we provide deterministic sufficient conditions on the sketching operator $\mathcal{A}$ so that the Restricted Isometry Property (RIP)~\eqref{eq:MMD_preservation} holds for mixture models using weighted Fourier features under some conditions. This is achieved thanks to a parametrization of the so-called \emph{normalized secant set} of the set $\mathfrak{G}_k$ with respect to the MMD. This parametrization uses the notion of \emph{dipoles} defined and used in \citep{GrBlKeTr20}. We extend the use of these tools and show how the proof of~\eqref{eq:MMD_preservation} boils down to the study of suprema of functions defined on $\mathbb{R}^d$. The benefit of this approach is to reduce the study from the normalized secant set, which is a set of signed measures with a geometry that is hard to grasp, to the much easier study of a few explicit functions defined on subsets of $\mathbb{R}^d$. 
Second, we leverage these deterministic sufficient conditions to establish the RIP when the sketching operator is random. This result is instantiated to carry out the proof of~\eqref{eq:MMD_preservation} in the case of a sketching operator built using i.i.d. frequencies. 
To constrast our result, which \emph{does not} 
require the use of importance sampling with the analysis given in \citep{GrBlKeTr20}, we establish that the high level structure of the latter \emph{requires importance sampling}. The technique we propose is thus both more general and much closer to practice. 
Moreover, we establish lower bounds showing the impossibility to achieve sharper estimates of provably good sketch sizes sufficient for a sketching operator based on Fourier feature maps: these impossibility results hold both for our analysis and existing one \citep{GrBlKeTr20}. 
Finally, we discuss the few steps that remain open to exploit our analysis to prove the RIP even when the frequencies are not necessarily independent, e.g. in the context of structured random Fourier features \citep{LeSaSm13,ChSi16,ChRoSaSiTuWe18}.

This article is structured as follows. In Section~\ref{sec:main_tools} we recall some notions and results from \citep{GrBlKeTr21,GrBlKeTr20} that are relevant to our study, as well as their main limitations which motivate this work.
In Section~\ref{sec:main_results} we present our results. We conclude and discuss some perspectives in Section~\ref{sec:conclusion}.

 \section{Main tools}\label{sec:main_tools}

We recall some results and definitions relevant to position our contributions.

\subsection{Sketching operator}\label{sec:skop}
The (random) feature maps that we will consider will bear special relations with the considered kernel $\kappa$, this will soon be discussed. For the moment we observe that
\begin{itemize}
\item  for \emph{any} bounded vector-valued function $\Phi: \mathcal{X} \to \mathbb{C}^{m}$, one can define a {\em sketching operator} 
\begin{equation*}\label{eq:DefSketchOp}
\mathcal{A}:  \left\{
    \begin{array}{ll}
       \mathcal{P}(\mathcal{X}) &\rightarrow \mathbb{C}^m  \\
        \pi & \mapsto \int_{\mathcal{X}}\Phi(x)\mathrm{d}\pi(x),
    \end{array}
\right.
\end{equation*}
with $\mathcal{P}(\mathcal{X})$ the set of probability distributions on $\mathcal{X}$. Jordan decomposition \citep{Hal50} allows to extend $\mathcal{A}$ to the set $\mathcal{M}(\mathcal{X})$ of finite signed measures~\citep[see][Appendix A.2.]{GrBlKeTr21}. 
\item for \emph{any} bounded kernel $\kappa$, one can define for every probability distributions $\pi,\pi' \in \mathcal{P}(\mathcal{X})$
\begin{equation}\label{eq:MMDInnerProd}
\langle \pi,\pi'\rangle_{\kappa} := \mathbb{E}_{X \sim \pi} \mathbb{E}_{Y \sim \pi'} \kappa(X,Y).
\end{equation}
 and extend this ``inner product'', as well as the definition~\eqref{eq:MMD} of the MMD,  to all finite signed measures in $\mathcal{M}(\mathcal{X})$. These can be manipulated as usual inner products and norms\footnote{Note that $\mathcal{M}(\mathcal{X})$ equipped with $\|\cdot\|_{\kappa}$ is not necessarily a Hilbert space, since  $\mathcal{M}(\mathcal{X})$ is not necessarily complete with respect to $ \|\cdot\|_{\kappa}$~\citep[see][Theorem 3.1 for details.]{StZi21}} thanks to a polarization identity: $\langle \nu,\nu'\rangle_{\kappa} = \frac{1}{4}\left(\|\nu+\nu'\|_{\kappa}^{2}-\|\nu-\nu'\|_{\kappa}^{2}\right)$.
\end{itemize}
We will write $\langle f,\pi\rangle := \mathbb{E}_{X \sim \pi}f(X)$, with implicit integrability assumption of function $f$ with respect to the probability distribution $\pi$. This is extended by a Jordan decomposition to $\langle f,\nu \rangle$ with $\nu \in \mathcal{M}(\mathcal{X})$. It should always be clear whether the bracket notation $\langle \cdot,\cdot\rangle$ stands for this shorthand or for the classical Hermitian inner product between vectors in $\mathbb{C}^{m}$.
\subsection{Separated mixture model, normalized secant set, and dipoles}
We focus our analysis on mixture modeling with a location-based family \citep[Definition 6.1]{GrBlKeTr20}: given a base probability distribution $\pi_{0}$ on $\mathbb{R}^{d}$ (for example $\pi_{0}$ may be the Dirac at zero, or a centered Gaussian) and a family $\Theta \subseteq \mathbb{R}^{d}$
of translation parameters, we consider $(\pi_\theta)_{\theta \in \Theta}$ where $\pi_{\theta}$ is the distribution of $X+\theta$ where $X \sim \pi_{0}$ and  observe that the map $\mathcal{I}: \theta \mapsto \mathcal{I}(\theta) = \pi_{\theta}$ is injective. 
Given a translation invariant metric $\rho$ on $\Theta \subseteq \mathbb{R}^{d}$ (for example, $\rho(\theta,\theta')$ may be the Euclidean distance between $\theta$ and $\theta'$, or $\|\theta'-\theta\|$ with any norm $\|\cdot\|$ on $\mathbb{R}^{d}$), we 
denote $\mathcal{T}:= (\Theta, \rho, \mathcal{I})$ and consider the set of {\em $2$-separated $k$-mixtures}~\citep{GrBlKeTr20} defined as
\begin{equation}\label{eq:DefMixtureModel}
\mathfrak{G}_{k} = \Bigg\{ \sum\limits_{i =1}^k u_i \pi_{\theta_i}; u_i \geq 0, \sum\limits_{i \in [k]} u_i =1, \theta_i \in \Theta, \forall i \neq i' \in [k], \rho(\theta_i, \theta_{i'}) \geq 2 \Bigg\}.
\end{equation}
More general separated mixture models of the form~\eqref{eq:DefMixtureModel} can be defined \citep[Section 5.2]{GrBlKeTr20} with $\mathcal{T}:= (\Theta, \rho, \mathcal{I})$ for any metric space $(\Theta,\rho)$ and injective map $\mathcal{I}:  \Theta \to \mathcal{P}(\mathcal{X})$, in which case we also denote $\pi_{\theta} := \mathcal{I}(\theta)$.
The study of \eqref{eq:MMD_preservation}  for general separated mixture models motivates the introduction of the normalized secant set $\mathcal{S}_{k}$ defined as follows
\begin{equation*}
\mathcal{S}_{k}
 :=  \Bigg\{ \frac{\nu-\nu'}{\| \nu-\nu'\|_{\kappa}}; \nu,\nu' \in \mathfrak{G}_{k}, \| \nu-\nu'\|_{\kappa} >0 \Bigg\}.
\end{equation*}
Indeed, \eqref{eq:MMD_preservation} is equivalent to $\sup_{\nu \in 
\mathcal{S}_{k}
} \big|\|\mathcal{A}\nu\|_{2}^2-1\big| \leq \delta$. In the following,
for every set $\mathfrak{T} \subset \{\frac{\nu}{\|\nu\|_{\kappa}}: \nu \in \mathcal{M}(\mathcal{X}),\ \|\nu\|_{\kappa}>0\}$ of normalized finite signed measures
and every sketching operator $\mathcal{A}$ we denote 
 \begin{equation}\label{eq:DefDelta}
 \delta(\mathfrak{T}|\mathcal{A}) := \sup_{\nu \in \mathfrak{T}} \big|\|\mathcal{A}\nu\|_{2}^2-1\big|.
 \end{equation}
  As we shall see, the elements of the normalized secant set may be approximated as a mixture of elementary measures called normalized dipoles.

\begin{definition}[Dipoles { \citep[Definitions 5.3, 5.6]{GrBlKeTr20} }]\label{def:Dipoles}
A finite signed measure $\iota \in \mathcal{M}(\mathcal{X})$ is a \emph{dipole} w.r.t. $\mathcal{T}= (\Theta, \rho, \mathcal{I})$ if 
$\iota = \alpha_{1} \pi_{\theta_1} - \alpha_{2} \pi_{\theta_2}$,
where $\theta_1, \theta_2 \in \Theta$, $\rho(\theta_1,\theta_2) \leq 1$ and $\alpha_{1}, \alpha_2 \geq 0$.
Two dipoles $\iota, \iota'$, are \emph{$1$-separated} if 
$\iota = \alpha_1\pi_{\theta_1} - \alpha_2 \pi_{\theta_2}, \:\: \iota' = \alpha_1'\pi_{\theta_1'} - \alpha_2' \pi_{\theta_2'}$, 
where $\rho(\theta_{i}, \theta_{j}') \geq 1$ for $i,j \in \{1,2\}$. The set of \emph{normalized dipoles} (with respect to kernel $\kappa$) is denoted
\begin{equation*}\label{eq:DefNormalizedDipoleSet}
\mathfrak{D} = \mathfrak{D}(\mathcal{T}) := \Big\{ \iota = \tilde{\iota}/\|\tilde{\iota}\|_{\kappa}, \tilde{\iota} \text{ is a dipole such that}\ \|\tilde{\iota}\|_{\kappa}>0  \Big\},
\end{equation*} 
and $\mathfrak{D}_{\neq}^{2} \subseteq \mathfrak{D} \times \mathfrak{D}$ denotes the set of pairs of 1-separated normalized dipoles.
\end{definition}
Dipoles offer a convenient parametrization of the (un-normalized) secant set.
\begin{lemma}[{\citep[Lemma 5.4]{GrBlKeTr20}}]
Let $\pi, \pi' \in \mathfrak{G}_{k}$. There exist $\ell \leq 2k$ nonzero dipoles $(\iota_{l})_{l \in [\ell]}$ that are pairwise $1$-separated and satisfy 
\begin{equation}\label{eq:dipole_decomposition}
\pi - \pi' = \sum\limits_{i=1}^{\ell} \iota_i.
\end{equation}
\end{lemma}
In other words, every element of the (unnormalized) secant set $\pi-\pi'$ is the sum of at most $2k$ dipoles. The decomposition \eqref{eq:dipole_decomposition} is convenient when calculating the squared MMD norm $\|\pi-\pi'\|_{\kappa} ^2 = \sum_{i=1}^{\ell} \|\iota_{i}\|_{\kappa}^2 + \sum_{i\neq i'} \langle \iota_{i}, \iota_{i'} \rangle_{\kappa}$ . In particular, under some additional assumptions on the kernel that we now discuss, the cross scalar products $\langle \iota_{i}, \iota_{i'}\rangle_{\kappa}$ are close to $0$ so that $\|\pi-\pi'\|_{\kappa}^2$ can be approximated by $\sum_{i=1}^{\ell} \|\iota_{i}\|_{\kappa}^2$.

\subsection{Kernel coherence}\label{sec:kernel_coherence}
To conduct our analysis, we require further assumptions on the compatibility of the positive definite
kernel $\kappa$ with the parameterized family of distributions $\mathcal{T}:= (\Theta, \rho, \mathcal{I})$. 
\begin{definition}\label{def:kernel_assumptions}
Given a family $\mathcal{T}:= (\Theta, \rho, \mathcal{I})$, 
a kernel $\kappa$ is said to be
\begin{enumerate}
\item \emph{non-degenerate}  with respect to $\mathcal{T}$ if $\|\pi_\theta\|_{\kappa} >0 $ for every $\theta \in \Theta$. This allows to define the \emph{$\mathcal{T}$-normalized kernel} $\bar{\kappa}$ as 
\begin{equation}\label{eq:DefNormalizedKernel}
\forall \theta, \theta' \in \Theta, \:\: \bar{\kappa}(\theta, \theta') := \frac{\langle\pi_\theta, \pi_{\theta'}\rangle_{\kappa}}{\|\pi_{\theta}\|_{\kappa}\|\pi_{\theta'}\|_{\kappa}}.
\end{equation}
 \item \emph{locally characteristic} with respect to $\mathcal{T}$ \citep[Definition 5.5]{GrBlKeTr20} if it is non-degenerate with respect to $\mathcal{T}$ and $ |\bar{\kappa}(\theta, \theta')| < 1$ for every $\theta, \theta' \in \Theta$ such that $0< \rho(\theta,\theta') \leq 1$. 
 This ensures that $\| \pi_{\theta}-\alpha\pi_{\theta'}\|_{\kappa}>0$ for every $\alpha \in \mathbb{R}$ whenever $0<\rho(\theta,\theta')\leq 1$.

\end{enumerate}
\end{definition}

\begin{definition}[Coherence { \citep[Definition 5.7]{GrBlKeTr20}}]\label{eq:DefCoherence}
Given an integer $\ell \geq 1$ the \emph{$\ell$-coherence of $\kappa$ with respect to $\mathcal{T}= (\Theta, \rho, \mathcal{I})$}, denoted $c_{\ell} = c_{\ell}(\kappa)$ is the smallest $c \geq 0$ such that, for any pairwise $1$-separated dipoles $(\iota_{i})_{i \in [\ell]}$ such that $\sum_{i = 1}^{\ell} \|\iota_{i}\|_{\kappa}^2 >0$, we have
\begin{equation}\label{eq:almost_pythagorian_formal}
1-c \leq \frac{\|\sum_{i = 1}^{\ell} \iota_i\|_{\kappa}^2}{\sum_{i = 1}^{\ell}\| \iota_i\|_{\kappa}^2} \leq 1+c.
\end{equation}
The kernel $\kappa$ has \emph{mutual coherence $\mu$ with respect to $\mathcal{T}$} if it is locally characteristic
wrt $\mathcal{T}$ and
\begin{equation}\label{eq:DefMutualCoherenceKernel}
\mu = \mu(\mathfrak{D}_{\neq}^{2}|\kappa) := \sup\limits_{(\iota,\iota') \in \mathfrak{D}_{\neq}^{2}}|\langle \iota, \iota' \rangle_{\kappa}|
\end{equation}
\end{definition}
By analogy with the kernel coherence we define the coherence of a sketching operator:
\begin{definition}[Operator Coherence]\label{eq:DefCoherenceSkOp}
For any sketching operator $\mathcal{A}$ and any set $\mathfrak{T} \subseteq \mathfrak{D}_{\neq}^{2}$ 
\begin{equation}\label{eq:DefGamma}
\mu(\mathfrak{T}|\mathcal{A}) := \sup\limits_{(\iota,\iota') \in \mathfrak{T}}|\mathfrak{Re} \langle \mathcal{A}\iota, \mathcal{A}\iota' \rangle|.
\end{equation}
\end{definition}
As shown in  \citep[Lemma 5.8]{GrBlKeTr20}, if the kernel $\kappa$ has mutual coherence $\mu$ with respect to $\mathcal{T}$ then $\kappa$ has $\ell$-coherence bounded by $c' := \mu(\ell -1)$. This is reminiscent (and indeed inspired by) classical results on incoherent dictionaries in sparse recovery~\citep[see][Chapter 5]{FoRa13}. In particular, if $\kappa$ has mutual coherence bounded by $\mu < 1/(2k-1)$ then the quasi-Pythagorean property~\eqref{eq:almost_pythagorian_formal} holds for $\ell = 2k$ with 
\begin{equation}\label{eq:LCoherenceFromMutualCoherence}
c = c_{2k} \leq (2k -1)\mu < 1.
\end{equation}
This implies that the normalized secant set $\mathcal{S}_{k}$ is made of ``nice'' mixtures of separated dipoles.

\begin{proposition}[{\citep[Lemma B.1]{GrBlKeTr20}}]\label{prop:embedding_normalized_secant_set}
Let $k \geq 1$ be an integer, and denote by $c = c_{2k}$ the $2k$-coherence of the kernel $\kappa$ with respect to $\mathcal{T}$. Under the assumption that $c <1$, we have
\begin{equation*}
\mathcal{S}_{k}
\subset \Bigg\{ \sum\limits_{i =1}^{2k} \alpha_{i} \iota_{i} :  (1+c)^{-1} \leq \sum\limits_{i = 1}^{2k} \alpha_{i}^2 \leq (1-c)^{-1}, \:\: \alpha_{i} \geq 0,\:\: (\iota_{i},\iota_{j}) \in \mathfrak{D}^{2}_{\neq}, \:\: 1 \leq i \neq j \leq 2k \Bigg\}.
\end{equation*}

\end{proposition}
In other words, the normalized secant set $\mathcal{S}_{k}$ is made of mixtures of $2k$ normalized dipoles with weights of controlled $\ell^{2}$ norm. This decomposition comes in handy when looking for an upper bound of high order moments  
as we will study soon for measure concentration arguments.

\subsection{Location-based families and shift-invariant kernels}

In most of this paper we focus on shift-invariant kernels, generally assumed to be normalized ($\kappa(0,0)=1$). As often we use the abuse of notation $\kappa(x,y) = \kappa(x-y)$. When the family $(\pi_{\theta})_{\theta \in \Theta}$ used to define the mixture model~\eqref{eq:DefMixtureModel} is location-based, we have   \citep[Proposition 6.2]{GrBlKeTr20}
\begin{equation}\label{eq:ShiftInvariantKernelNorm}
\forall \theta \in \Theta,\quad \|\pi_{\theta}\|_{\kappa} = \|\pi_{0}\|_{\kappa},
\end{equation}
hence $\kappa$ is non-degenerate with respect to $\mathcal{T}$ if, and only if, $\|\pi_0\|_{\kappa} >0 $. Moreover,  the $\mathcal{T}$-normalized kernel $\bar{\kappa}$ itself is also shift-invariant. We also abuse notations and denote
$\bar{\kappa}(\theta-\theta') = \bar{\kappa}(\theta, \theta') = \frac{1}{\|\pi_{0}\|_{\kappa}^{2}} \kappa(\pi_\theta, \pi_{\theta'})$.
The low-coherence property is satisfied by classical shift-invariant kernels and location-based families $\mathcal{T}$ \citep{GrBlKeTr20}. 

\begin{example}{(\textbf{Mixtures of Diracs and the Gaussian kernel {\citep[Definition 6.9]{GrBlKeTr20}}})} \label{ex:DiracMixture}
In this case, $\pi_{0}$ is the Dirac distribution at $0$, $\rho = \|\cdot\|_{2}/\epsilon$ where $\epsilon >0$, and $\kappa$ is the Gaussian kernel:
\begin{equation*}\label{eq:def_kappa_diracs}
\kappa(x,x') := \mathrm{exp}\Big(- \frac{\|x-x'\|_{2}^2}{2s^2}\Big),
\end{equation*}
with $s>0$ a scale parameter.  The normalized kernel writes \citep[Section 6.3.1]{GrBlKeTr20}
\begin{equation*}\label{eq:def_bar_kappa_diracs}
\bar{\kappa}(\theta-\theta') = \exp\Big(-\frac{  \|\theta - \theta'\|_{2}^2}{2s^2}\Big).
\end{equation*}
By \citep[Theorem 5.16, Lemma 6.10, Theorem 6.11]{GrBlKeTr20}, $\kappa$ is locally characteristic and its mutual coherence with respect to $\mathcal{T}$ is smaller than or equal to $12/(16(2k-1))$ as soon as
\begin{equation*}\label{eq:SeparationConditionDirac}
\epsilon \geq s/s_{k}^{\star}\quad \text{with}\ s_{k}^{*}:= (4\sqrt{\log(ek)})^{-1}.
\end{equation*}
\end{example}

\begin{example}{(\textbf{Mixtures of Gaussians and the Gaussian kernel  {\citep[Definition 6.9]{GrBlKeTr20}}})}\label{ex:KernelExamplesIncoherent} In this case, 
we consider the Mahalanobis norm $\|\cdot\|_{\bm{\Sigma}}$, defined by $\|x\|_{\bm{\Sigma}} := \sqrt{\langle x, \bm{\Sigma}^{-1}x \rangle} = \|\bm{\Sigma}^{-1/2}x\|_{2}$ for $x \in \mathbb{R}^d$,  where $\bm{\Sigma} \in \mathbb{R}^d$ is a positive definite matrix, $\rho = \|\cdot\|_{\bm{\Sigma}}/\epsilon$ where $\epsilon >0$, and $\pi_{0} = \mathcal{N}(0,\bm{\Sigma})$.  Finally $\kappa$ is the Gaussian kernel:
\begin{equation*}\label{eq:def_kappa_gaussians}
\kappa(x,x') := \mathrm{exp}\Big(-\frac{\|x-x'\|_{\bm{\Sigma}}^2}{2s^2}\Big),
\end{equation*}
with $s>0$ a scale parameter. The normalized kernel writes \citep[Section 6.3.1]{GrBlKeTr20}
\begin{equation*}\label{eq:def_bar_kappa_gaussians}
\bar{\kappa}(\theta-\theta') = \exp\Big(-\frac{  \|\theta - \theta'\|_{\bm{\Sigma}}^2}{2(2+s^2)}\Big).
\end{equation*}
By \citep[Theorem 5.16, Lemma 6.10, Theorem 6.11]{GrBlKeTr20}, $\kappa$ is locally characteristic and its mutual coherence with respect to $\mathcal{T}$ is smaller than or equal to $12/(16(2k-1))$ as soon as 
\begin{equation*}\label{eq:SeparationConditionGaussian}
\epsilon \geq \frac{\sqrt{s^{2}+2}}{s_{k}^{\star}}\quad \text{with}\ s_{k}^{*}:= (4\sqrt{\log(ek)})^{-1}.
\end{equation*}

\end{example}

\subsection{Random Fourier features}\label{sec:RFF}

The analysis in~\citep{KeBoGrPe18,GrBlKeTr20} is conducted using a variant of the random Fourier feature map described in the introduction, using importance sampling. It also uses the more general notion of a $\kappa$-compatible random sketching operator.
\begin{definition}\label{def:KernelCompatibleRFF}
Consider a kernel $\kappa$ on $\mathcal{X} \times \mathcal{X}$ and a random feature map $\Phi$ defined as in~\eqref{eq:DefGenericRFMap} from a parametric family $\{\phi_{\omega}: \mathcal{X} \to \mathbb{C}\}_{\omega \in \Omega}$ and random parameters $(\omega_{1},\ldots,\omega_{m})$, i.i.d. or not.
The feature map (and by extension the resulting sketching operator $\mathcal{A}$) is said to be $\kappa$-compatible if the expected value (with respect to the draw of frequencies) of the hermitian inner-product $\langle \Phi(x),\Phi(y)\rangle$ is exactly $\kappa(x,y)$, cf~\eqref{eq:ExpectedKernelProp}. 
\end{definition}

\begin{definition}\label{def:RFFsketching}
Given a weight function $w: \mathbb{R}^{d} \to (0,+\infty)$, a sketching operator $\mathcal{A}$ is a $w$-weighted Fourier feature ($w$-FF) sketching operator if it is built from a feature map $\Phi$ as in \eqref{eq:DefGenericRFMap} with some frequency vectors $\omega_{1}, \ldots,\omega_{m}$ and individual components defined as
 \begin{equation}\label{eq:FFComponent}
\phi_{\omega} = \phi_{\omega}^{w}: x \mapsto e^{\imath \omega^{\top} x}/w(\omega).
\end{equation}

If the frequency components are jointly drawn (i.i.d. or not) from some probability distribution then $\mathcal{A}$ is called a  $w$-weighted \emph{random} Fourier feature ($w$-RFF) sketching operator. 

We will often drop the the dependency of $\phi_{\omega}$ in $w$ for brevity of notation, and call $\mathcal{A}$ a WFF (or RFF) sketching operator when there is no need to specify the corresponding $w$.

\end{definition}

\begin{definition}\label{def:KappaCompatible}
Consider a normalized shift-invariant kernel $\kappa$ on $\mathcal{X} = \mathbb{R}^{d}$. A weight function
$w: \mathbb{R}^{d} \rightarrow (0,\infty)$ is said to be $\kappa$-compatible if 
 \begin{equation}
\label{eq:DefGammaNormalized}
\int_{\mathbb{R}^d} w^{2}(\omega) \hat{\kappa}(\omega) \mathrm{d}\omega =1.
\end{equation} 
\end{definition}

\begin{example}\label{ex:KCompatibleRFF}
Given a normalized shift-invariant kernel $\kappa$ on $\mathcal{X} = \mathbb{R}^{d}$, if $w$ is $\kappa$-compatible then
\begin{equation}\label{eq:ImportanceSampling}
\Lambda(\omega) := w^{2}(\omega)  \hat{\kappa}(\omega)
\end{equation}
defines a probability density function. For frequency vectors drawn (i.i.d. or not) with common marginal probability density function $\Lambda$ and any $x,y$ we have
\begin{align}\label{eq:kappa_generalized_phi}
\mathbb{E}_{\omega \sim \Lambda} \phi_{\omega}^{w}(x)\overline{\phi_{\omega}^{w}(y)}
&=  \int w^{-2}(\omega) e^{\imath \omega^{\top}(x-y)} \Lambda(\omega)\mathrm{d}\omega
=  \int e^{\imath \omega^{\top}(x-y)} \hat{\kappa}(\omega)d\omega = \kappa(x-y),
\end{align}
hence the random feature map $\Phi$
is $\kappa$-compatible (\Cref{def:KernelCompatibleRFF})
\begin{align*}
\mathbb{E}_{\omega_{1},\ldots,\omega_{m}} \langle \Phi(x),\Phi(y)\rangle
& = \frac{1}{m} \sum_{j=1}^{m}\mathbb{E}_{\omega_{1},\ldots,\omega_{m}} \phi_{\omega_{j}}^{w}(x)\overline{\phi_{\omega_{j}}^{w}(y)}
=  \frac{1}{m} \sum_{j=1}^{m}\mathbb{E}_{\omega_{j} \sim \Lambda} \phi_{\omega_{j}}^{w}(x)\overline{\phi_{\omega_{j}}^{w}(y)} = \kappa(x,y),
\end{align*}
moreover, 
\begin{equation}\label{eq:UnbiasedEmpiricalMMD2}
 \forall \nu,\nu' \in \mathcal{S}_{k}, \:\: 
 \mathbb{E}_{\omega \sim \Lambda}  \langle \phi_{\omega}^{w}, \nu \rangle \overline{\langle \phi_{\omega}^{w},\nu\rangle} =\langle \nu,\nu'\rangle_{\kappa}.
\end{equation}
Specializing this to $\nu = \nu'$ yields 
$\mathbb{E}_{\omega \sim \Lambda}  |\langle \phi_{\omega}^{w},\nu \rangle|^{2} =\| \nu\|_{\kappa}^{2}$.

With possibly distinct marginal densities $\omega_{j} \sim \Lambda_{j}$, we similarly get that the expectation 
of $ \langle \Phi(x),\Phi(y)\rangle$ is $\int w^{-2}(\omega) e^{\imath \omega^{\top}(x-y)} \frac{1}{m}(\sum_{j=1}^{m}\Lambda_{j}(\omega)) d\omega$ hence the same conclusion holds if, and only if, the ``average marginal density'' satisfies almost everywhere 
\begin{equation}\label{eq:ImportanceSamplingStructured}
\frac{1}{m}(\sum_{j=1}^{m}\Lambda_{j}(\omega)) = w^{2}(\omega) \hat{\kappa}(\omega).
\end{equation}
\end{example}

It can also occur that the frequencies $\omega_{1}, \dots, \omega_m$ are not independent, for example using \emph{structured random features} \citep{LeSaSm13,ChSi16,ChRoSaSiTuWe18}. Assuming here for simplicity that $m$ is a multiple of $d$, the construction of such frequencies is such that the matrix $\bm{\Omega} \in \mathbb{R}^{d \times m}$ with columns $\omega_{j}$, $1 \leq j \leq m$ is defined as the concatenation of $m/d$ i.i.d. random matrices $\bm{B}_i \in \mathbb{R}^{d\times d}$, $1 \leq i \leq m/d$. This is advantageous when each $\bm{B}_{i}$ is structured in such a way that the product $\bm{B} z$ for $z \in \mathbb{R}^d$ costs $\mathcal{O}(d\log(d))$ instead of $\mathcal{O}(d^2)$, e.g. when $d$ is a power of two and $\bm{B} = \bm{D}_{1}\bm{H}\bm{D}_{2}\bm{H}\bm{D}_{3}\bm{H}$ with $\bm{H}$ the matrix associated to the (fast) Hadamard transform, and $\bm{D}_{\ell}$, $1 \leq \ell \leq 3$ appropriate random diagonal matrices. When $d$ is not a power of $2$ and/or $m$ is not a multiple of $d$ the construction can be adapted using padding techniques. We refer the reader to \citep[Chapter 5]{Cha20}, where an overview of such techniques is summarized. It can also be proved that under appropriate conditions the resulting random feature map $\Phi$ is still $\kappa$-compatible (\Cref{def:KernelCompatibleRFF}), see e.g. \citep[Lemma 7]{LeSaSm13} and \citep[Lemma 5.6]{Cha20} for results when $\kappa$ is a Gaussian kernel.

\subsection{Existing results and their limitations}\label{sec:existing_results}
To establish a bound of the type~\eqref{eq:MMD_preservation} with a general mixture model, or equivalently to bound the constant $ \delta(\mathcal{S}_{k}|\mathcal{A})$ of~\eqref{eq:DefDelta} the strategy deployed in \citep{GrBlKeTr20} exploits covering numbers, pointwise concentration, and a deterministic bound on a certain Lipschitz constant. Indeed, if a family $(\nu_i)_{i \in [\mathcal{N}]}$ of elements of the secant set $\mathcal{S}_{k}$ satisfies
\begin{equation}\label{eq:delta_tau_tau_over_two_explanation}
\forall \nu \in \mathcal{S}_{k}
, \:\: \exists i \in [\mathcal{N}], \:\: 
\big|
\|\mathcal{A}\nu\|_{2}^{2}-\|\mathcal{A}\nu_{i}\|_{2}^{2}
  \big| \leq \frac{\tau}{2}, 
\end{equation}
then
\begin{equation*}\label{eq:prip_implies_rip}
\sup\limits_{\nu \in \mathcal{S}_{k}} 
\big| 
\|\mathcal{A}\nu\|_{2}^{2}
 -1 \big| \leq \sup\limits_{i \in [\mathcal{N}]}\big| 
\|\mathcal{A}\nu_{i}\|_{2}^{2}
  -1 \big| +\frac{\tau}{2}.
\end{equation*}
hence proving that $\delta(\{\nu_1, \dots, \nu_{\mathcal{N}}\}|\mathcal{A}) \leq \tau/2$ holds
is sufficient to deduce that $\delta(\mathcal{S}_{k}|\mathcal{A}) \leq \tau$.
Moreover,  assuming there is $v>0$ such that for every sketch size $m \geq 1$ the corresponding RFF sketching operator $\mathcal{A}$ satisfies a punctual concentration estimate of the form
\begin{equation}\label{eq:punctual_concentration_generic} 
\sup_{\nu \in \mathcal{S}_{k}}\mathbb{P}\big( \big|\|\mathcal{A}\nu\|_{2}^2 - 1\big| > \frac{\tau}{2} \big) 
\leq 2
 \exp\bigg(-\frac{m}{v \rvs{(\tau)}} \bigg)\, ,
\end{equation}
a union bound allows to deduce that $\delta(\{\nu_1, \dots, \nu_{\mathcal{N}}\}|\mathcal{A}) \leq \tau/2$ holds with probability at least $1-2\mathcal{N} \cdot \exp(-m/v\rvs{(\tau)})$ on the draw of $\mathcal{A}$. These arguments show that, for any $0<\eta<1$,  if the sketch size satisfies $m \geq v\rvs{(\tau)} \log 2\mathcal{N}/\eta$ then $\delta(\mathcal{S}_{k}|\mathcal{A}) \leq \tau$ with probility at least $1-\eta$.

The smallest size of a family satisfying \eqref{eq:delta_tau_tau_over_two_explanation} is a covering number of $X = \mathcal{S}_{k}$ with the pseudo-metric $d(\nu,\nu') := | \|\mathcal{A}\nu\|_{2}^{2}-\|\mathcal{A}\nu'\|_{2}^{2}|$ (see e.g. \citep{CuSm01} for the well-known definition of coverings in a pseudo-metric space $(X,d)$, and covering numbers, denoted $\mathcal{N}(X,d,\epsilon)$ at scale $\epsilon>0$). However, in the case of random sketching, this pseudo-metric \emph{depends on the random feature map $\Phi$} (or equivalently on the random sketching operator $\mathcal{A}$). To circumvent this difficulty the approach of  \citep[Proof of Lemma B.4]{GrBlKeTr20}  is to observe that 
\begin{equation}\label{eq:MetricInequalityExploitingImportanceSampling}
\forall \nu, \nu' \in \mathcal{S}_{k}
, \:\: \big|\|\mathcal{A}\nu\|_{2}^2 -\|\mathcal{A}\nu'\|_{2}^2\big|  \leq M  \|\nu - \nu' \|_{\mathcal{F}} ,
\qquad \text{with}\ M := 2 \sup_{\nu \in  \mathcal{S}_{k}} \|\nu\|_{\mathcal{F}},
\end{equation}
where $\|\nu \|_{\mathcal{F}} := \sup_{\omega \in \mathbb{R}^{d}} |\langle\phi_{\omega}, \nu \rangle|$ defines a \emph{deterministic} pseudo-norm on finite signed measures. 
Therefore, a covering $(\nu_i)$ of $\mathcal{S}_k$ with respect to $d'(\nu,\nu') := \|\nu-\nu'\|_{\mathcal{F}}$ at scale $\tau/2M$ satisfies  \eqref{eq:delta_tau_tau_over_two_explanation}, with $\mathcal{N} = \mathcal{N}(\mathcal{S}_{k},d',\|\cdot\|_{\mathcal{F}},\tau/2M)$. Inequality~\eqref{eq:MetricInequalityExploitingImportanceSampling} means that the (random) function $\nu \mapsto \|\mathcal{A}\nu\|_{2}^{2}$ is Lipschitz with respect to the metric $d'$, with (deterministic) Lipschitz constant $M$. 

For RFF sketching with location-based families, we will see that getting a finite $M$ highly constrains the function $w$ (cf~\eqref{eq:FFComponent}). \emph{A primary objective of this paper is to relax this constraint. }

For the concentration estimate~\eqref{eq:punctual_concentration_generic}, the approach of \citep{GrBlKeTr20} is generic for general mixture models and general sketching operators defined with a feature map as in \eqref{eq:DefGenericRFMap} using a parameterized family $\phi_{\omega}$ where $\omega \sim \Lambda$ are i.i.d. parameters. It
combines a Bernstein inequality with an assumption on higher order moments on normalized dipoles.  Indeed, a consequence of Proposition~\ref{prop:embedding_normalized_secant_set} is that if the $2k$-coherence of $\kappa$ is bounded by $c$ then\footnote{See before Eq. (109) in the proof of Theorem 5.11 in \citep[Section B.1]{GrBlKeTr20}} 
\begin{equation*}\label{eq:upper_bound_moment_dipole}
\forall q \geq 2,\:\: \sup_{\nu \in \mathcal{S}_{k}}\mathbb{E}_{\omega \sim \Lambda} | \langle \phi_{\omega},\nu \rangle |^{2q} \leq \Big(\frac{2k}{1-c} \Big)^{2q} \sup\limits_{\iota \in \mathfrak{D}} \mathbb{E}_{\omega \sim \Lambda} | \langle \phi_{\omega},\iota \rangle |^{2q},
\end{equation*} 
i.e., controlling the moments on normalized dipoles is enough to control them on the normalized secant.
Assuming a $\kappa$-compatible random feature map, this was shown to imply a concentration of the squared norm of the sketch $\|\mathcal{A}\nu\|_{2}^2$ around its expected value  \citep[Theorem 5.11]{GrBlKeTr20}.

The following is the specialization of  \citep[Theorem 5.11]{GrBlKeTr20} to the case of a $\kappa$-compatible RFF sketching operator with i.i.d.  frequencies $\omega_{j} \sim \Lambda$, $1 \leq j \leq m$.
\begin{theorem}[{ \citet[Theorem 5.11]{GrBlKeTr20}}]\label{thm:punctual_fluctuation_result}
Let $\mathcal{T} := (\Theta,\rho,\mathcal{I})$ 
be a location-based family with base distribution $\pi_{0}$ on $\mathbb{R}^{d}$, with $\Theta \subseteq \mathbb{R}^{d}$ a bounded subset, $\rho(\cdot,\cdot) := \|\cdot-\cdot\|$ where
 $\|\cdot\|$ is some norm on $\mathbb{R}^{d}$. Consider a normalized shift-invariant kernel $\kappa$ with an integer $k \geq 1$ such that $\kappa$ has its 
$2k$-coherence with respect to the location-based family $\mathcal{T}$ bounded by $0 \leq c \leq 3/4$.

Consider $w$ a $\kappa$-compatible weight function, $\Lambda=w^{2}\hat{\kappa}$, and assume that there exist $\beta_{1}>0$ and $\beta_{2} \geq 1$ such that\begin{equation}\label{eq:dipole_moments_condition_0}
\forall q \geq 2, \:\: \sup_{ \iota \in \mathfrak{D}}\mathbb{E}_{\omega \sim \Lambda}| \langle \phi_{\omega}^{w}, \iota \rangle |^{2q} \leq \frac{q!}{2} \beta_{1} \beta_{2}^{q-1}.
\end{equation}
Then for every $m \geq 1$ the RFF sketching operator $\mathcal{A}$ built with i.i.d. random frequencies $\omega_{j} \sim \Lambda$ is $\kappa$-compatible and
\begin{equation}\label{eq:punctual_concentration} 
\forall \tau>0,\  \sup_{\nu \in \mathcal{S}_{k}} \mathbb{P}\big( \big|\|\mathcal{A}\nu\|^2_{2} - 1\big| > \tau \big) 
\leq 2
 \exp\bigg(-\frac{m}{v} \bigg),
\end{equation}
where $v = v(k,\tau) := 2v_{k} (1+\tau/3)/\tau^2$, with $v_k = 16ek \beta_{2} \log^2(4ek \beta_{1})$.
\end{theorem}
The quantity $v(k,\tau) >0$ is reminiscent of a variance and depends on $k,\tau$ (as displayed by the notation) but also more implicitly on $\sup_{\iota \in \mathfrak{D}} \mathbb{E}_{\omega \sim \Lambda} | \langle \phi_{\omega}^{w}, \iota \rangle |^{2q}$, $q \geq 2$ via the constants $\beta_{1},\beta_{2}$. 
As will shortly see, the dependence in $k$ is something we pay when estimating the sketch size, and it is natural to wonder whether this is due to the analysis or intrinsic. \emph{A side contribution of this paper is to show that this is somehow inevitable for this type of result.}

The above arguments from  \citep{GrBlKeTr20} lead to the following result.
\begin{theorem}\label{thm:old_fluctuation_result} Consider $\mathcal{T}$ $\kappa$, $w$, $\Lambda$, $\beta_{1}>0,\beta_{2} \geq 1$ as in~\Cref{thm:punctual_fluctuation_result}. Assume that~\eqref{eq:dipole_moments_condition_0} holds,  
and that the constant $M$ defined in~\eqref{eq:MetricInequalityExploitingImportanceSampling} is finite.
Then, with $v(k,\cdot)$ as in~\Cref{thm:punctual_fluctuation_result}, we have
\begin{equation*}
\forall \tau >0,\ 
\mathbb{P}\big( \delta(\mathcal{S}_k|\mathcal{A}) > \tau \big) 
\leq 2  \mathcal{N}\Big(\mathcal{S}_k, \|\cdot\|_{\mathcal{F}}, \frac{\tau}{2M}\Big)
 \exp\bigg(-\frac{m}{v(k,\tau/2)} \bigg).
\end{equation*}

\end{theorem}
In particular, under the assumptions of~\Cref{thm:old_fluctuation_result}, the property $\delta(\mathcal{S}_k|\mathcal{A})\leq \tau$ holds with probability at least $1-\eta$ where $0<\eta<1$ as soon as the sketch size satisfies
\begin{equation*}
m \geq v(k,\tau/2)\log \Big(2\mathcal{N}\Big(\mathcal{S}_k, \|\cdot\|_{\mathcal{F}}, \frac{\tau}{2M}\Big)/\eta\Big).
\end{equation*}

However, this estimate on a sufficient sketch size is only relevant when the constant $M$ defined  in~\eqref{eq:MetricInequalityExploitingImportanceSampling} is finite
and $\mathcal{N}(\mathcal{S}_k, \|\cdot\|_{\mathcal{F}}, \tau/(2M))<+\infty$.
By \citep[Theorem 5.12, Theorem 5.15, Lemma 6.4, Lemma 6.7]{GrBlKeTr20}, this is the case when $\kappa$ is {\em strongly locally characteristic} with respect to $\mathcal{T}$ (a property satisfied in \Cref{ex:DiracMixture,ex:KernelExamplesIncoherent}, cf \citep[Theorem 6.11]{GrBlKeTr20} ) under the assumption that
\begin{equation}\label{eq:phi_omega_conditions}
\sup_{\omega \in \mathbb{R}^{d}} 
|\langle  \phi_{\omega}^{w},\pi_{0} \rangle| \cdot \max(1,\|\omega\|_{\star},\|\omega\|_{\star}^2)<+\infty
\end{equation}
with $\|\cdot\|_{\star}$ the dual norm of $\|\cdot\|$ defined by $\|u\|_{\star} := \sup_{\|v\|\leq 1} u^\top v$. 
Then, by \citep[Theorem 5.12, Lemma 6.7]{GrBlKeTr20}, we have  for fixed $\epsilon>0$, $\log\mathcal{N}(\mathcal{S}_{k}, \|\cdot\|_{\mathcal{F}},\epsilon) = \mathcal{O}(kd)$ up to logarithmic factors in $k,d,1/\epsilon$, so that a sufficient sketch size to have the desired result with high probability can be shown to satisfy $m = \mathcal{O}(k^2d)$ up to logarithmic factors\footnote{If the sup in~\eqref{eq:phi_omega_conditions} is not only finite but at most polynomial in $k,d$ then $\log M/\tau$ is also logarithmic in $k,d$.}  in $k,d,M/\tau$.

Condition~\eqref{eq:phi_omega_conditions} constrains the choice of weight function $w$ for models such as mixtures of Diracs. Indeed, for this family of mixtures $|\langle \phi_{\omega}^{w},\pi_{0} \rangle| = 1/w(\omega)$ so that~\eqref{eq:phi_omega_conditions} implies 
$w(\omega) \gtrsim \max(1,\|\omega\|_{\star},\|\omega\|_{\star}^{2})$ (which imposes constraints on the behavior of $w$ both around $\omega \to 0$ and $\|\omega\|_{\star} \to +\infty$). 
This is the main reason why the analysis in \citep{GrBlKeTr20} is limited to random Fourier features \emph{with importance sampling}, while plain sampling seems enough in practical experiments. \emph{The main contribution of this paper is to establish results valid without assuming~\eqref{eq:phi_omega_conditions}.}

In summary, besides a (strongly locally characteristic) kernel with bounded $2k$-coherence and a $\kappa$-compatible weight function $w$, existing results are built on the following assumptions:
\begin{enumerate}
\item moment conditions~\eqref{eq:dipole_moments_condition_0} on $\Lambda = w^{2}\hat{\kappa}$, to establish punctual concentration;
\item growth conditions~\eqref{eq:phi_omega_conditions} on $w$, to control the Lipschitz constant $M$ from~\eqref{eq:MetricInequalityExploitingImportanceSampling} and the associated covering numbers;
\item i.i.d. frequencies $\omega_{j} \sim \Lambda$, $1 \leq j \leq m$.
\end{enumerate}
Under these assumptions, a sketch size of the order of $k^{2}d$ (up to logarithmic factors) is proved to be sufficient.
As experiments conducted in \citep{KeTrTrGr17} suggested that the same result should hold with a smaller sufficient sketch size $m = \mathcal{O}(kd)$, this raises the question whether the theoretical bound of Theorem~\ref{thm:old_fluctuation_result} can be refined.  
An approach which would provide an easy fix would be if we could simply remove a $k$ factor in the punctual concentration estimate of~\Cref{thm:punctual_fluctuation_result} via a more subtle analysis. 
This would indeed naturally insert in the analysis of either \citep{GrBlKeTr20} or of this work to yield the desired order of magnitude of $m$.
A second contribution of this work is  to show that such a uniform improvement of concentration estimates is simply not possible.

\section{Main results}\label{sec:main_results}

This work aims to overcome some shortcomings of the theoretical analysis given in~\citep{GrBlKeTr20}. First, we show in Section~\ref{sec:Mfinite} that some of the growth conditions~\eqref{eq:phi_omega_conditions} are necessary to exploit the analysis given in \citep{GrBlKeTr20}. Then, in Section~\ref{sec:generic_result_random_sketching} we give our main contribution: we provide an alternative analysis that allows to completely relax the growth conditions~\eqref{eq:phi_omega_conditions} on $w$. This yields results (still with a sketch size $m$ of the order of $k^{2}d$) for a much less constrained family of importance sampling schemes, including plain sampling $w \equiv 1$. This is primarily achieved by circumventing the deterministic control of the Lipschitz constant $M$ from~\eqref{eq:MetricInequalityExploitingImportanceSampling}: instead, we provide a stochastic control of a ``typical'' Lipschitz constant, thanks to a reduction of the stochastic control of $\delta(\mathcal{S}_{k}|\mathcal{A})$ to a stochastic control of its equivalent on dipoles, $\delta(\mathfrak{D}|\mathcal{A})$, and of the coherence of the sketching operator, $\mu(\mathfrak{D}_{\neq}^{2}|\mathcal{A})$. This yields a substantial streamlining of the analysis which is then further reduced to the equivalent quantities for so-called normalized monopoles and balanced normalized dipoles. This is achieved thanks to deterministic bounds on $\delta(\mathcal{S}_k|\mathcal{A})$ established in Section~\ref{sec:sharp_determistic_bounds}. Finally, we show in \Cref{sec:lower_bound} that both the analysis given in \citep{GrBlKeTr20} and the one given in this work cannot be fixed by a simple improvement of concentration estimates to close the gap between sufficient sketch sizes endowed with theoretical guarantees, which scale essentially as $O(k^{2}d)$, and practically observed sketch sizes, which scale as $O(kd)$~\citep{KeTrTrGr17}.

\subsection{On the necessity of conditions~\eqref{eq:phi_omega_conditions}} \label{sec:Mfinite}
As mentioned in Section~\ref{sec:existing_results}, the analysis of~\citep{GrBlKeTr20} assumes that condition~\eqref{eq:phi_omega_conditions} holds to obtain that
$M := \sup_{\nu \in \mathcal{S}_k} \|\nu\|_{\mathcal{F}}$ and the covering numbers $\mathcal{N}(\mathcal{S}_k, \|\cdot\|_{\mathcal{F}}, \epsilon)$ are finite. Here we establish a partial converse.
The following proposition is proved in Appendix~\ref{proof:lower_bound_L_constant}.
\begin{proposition}\label{prop:lower_bound_L_constant}
Consider a normalized shift-invariant kernel  $\kappa$.
Consider a location-based family $\mathcal{T} = (\Theta,\varrho,\mathcal{I})$ with base distribution $\pi_{0}$ where $\Theta$ contains a neighborhood of zero and $\varrho(\theta,\theta') := \|\theta-\theta'\|$ for some arbitrary norm $\|\cdot\|$ on $\mathbb{R}^{d}$. Assume that $\bar{\kappa}$, as defined in~\eqref{eq:DefNormalizedKernel}, is $C^{2}$ at zero, and assume that $\nabla^2\bar{\kappa}(0) \in \mathbb{R}^{d\times d}$, the Hessian matric of $\bar{\kappa}$ at $0$, is non-zero. Then, for weighted Fourier features, we have for every integer $k \geq 1$, and separated $k$-mixture model $\mathfrak{G}_{k}$ from~\eqref{eq:DefMixtureModel} 
\begin{equation*}
\sup_{\nu \in \mathcal{S}_k} \|\nu\|_{\mathcal{F}}  \geq \sup_{\omega \in \mathbb{R}^{d}} \frac{|\langle \phi^{w}_{\omega},\pi_{0} \rangle|}{\|\pi_{0}\|_{\kappa}}
\max\left(1, \frac{1}{\sqrt{\|\nabla^2 \bar{\kappa}(0)\|_{\mathrm{op}}}} \|\omega\|_{\star}
\right).
\end{equation*}
\end{proposition}
A direct consequence of  Proposition~\ref{prop:lower_bound_L_constant} is that if $\sup_{\nu \in \mathcal{S}_k} \|\nu\|_{\mathcal{F}} <+\infty$ then
\begin{equation*}
\sup_{\omega \in \mathbb{R}^{d}} |\langle \phi^{w}_{\omega},\pi_{0} \rangle|
\max\left(1, \|\omega\|_{\star}\right) <+\infty
\end{equation*}
which is reminiscent of~\eqref{eq:phi_omega_conditions} and plays the role of a partial converse.

In the setting of mixtures of Diracs defined in Example~\ref{ex:DiracMixture}, we have $|\langle \phi^{w}_{\omega}, \pi_0 \rangle| = 1/w(\omega)$ (cf \eqref{eq:FFComponent}) and $\bar{\kappa}$ is $C^2$ at $0$ with $\nabla^2 \bar{\kappa}(0) \neq 0$, thus we can apply Proposition~\ref{prop:lower_bound_L_constant}, and we get that there exists a constant $C>0$ such that
\begin{equation*}
\forall \omega \in \mathbb{R}^{d}, \:\: w(\omega) \geq C
\max\left(1, \|\omega\|_{\star}\right).
\end{equation*}
Thus, the proof technique of  \citep{GrBlKeTr20}, which is summarized in \Cref{thm:old_fluctuation_result}, indeed \emph{requires} the weight functions $w(\omega)$ to grow with $\|\omega\|_{\star}$ to provide non-trivial results. It is in particular inapplicable to the ``flat'' weight function $w(\omega)=1$. In contrast, this weight function is covered by our \Cref{cor:bounds_dirac}.
It is an interesting challenge left to future work to determine if~\eqref{eq:phi_omega_conditions} is in fact fully necessary to have both $\sup_{\nu \in \mathcal{S}_{\kappa}}\|\nu\|_{\mathcal{F}} < +\infty$ \emph{and}   $\mathcal{N}(\mathcal{S}_k, \|\cdot\|_{\mathcal{F}}, \epsilon) < \infty$ for some $\epsilon>0$.

\subsection{Sharp deterministic bounds on $\delta(\mathcal{S}_{k}|\mathcal{A})$}\label{sec:sharp_determistic_bounds}

In light of Proposition~\ref{prop:lower_bound_L_constant}, we propose an alternative analysis to control $\delta(\mathcal{S}_{k}|\mathcal{A})$ that does not require condition~\eqref{eq:phi_omega_conditions}.
This subsection focuses on the deterministic part of this analysis: first, we upper bound  $\delta(\mathcal{S}_{k}|\mathcal{A})$ (which is defined as a supremum over the normalized secant set) using quantities defined on simpler sets made of dipoles (\Cref{prop:generic_upper_bound_worst_sketching_error}); then the latter are themselves controlled in terms of even simpler, quantities defined in terms of monopoles and balanced dipoles (\Cref{thm:upper_bound_sup_monopoles_dipoles}); finally all considered quantities are explicitly written as suprema of empirical averages over frequency vectors $\omega_{j}$ (\Cref{thm:parametrization_functions_cross_dipoles}). 
This will be used in the next subsection to control all quantities in the context of a random sketching operator $\mathcal{A}$. As we will see, the main price to pay for this alternative analysis is that (unlike in~\Cref{thm:old_fluctuation_result}, and more generally in results of the same flavor inspired by compressive sensing) $\delta(\mathcal{S}_{k}|\mathcal{A})$ is no longer proved to be \emph{arbitrarily small} with high probability when the sketch size $m$ is large enough, but only \emph{arbitrarily close to a quantity (smaller than one) depending on the $2k$-coherence} of the kernel $\kappa$.

\subsubsection{From the normalized secant set to normalized monopoles and balanced dipoles}

As a first step we bound the targeted quantity, which is defined as a supremum on the normalized secant set, in terms of two suprema defined on simpler sets of normalized dipoles.
\begin{proposition}[From the secant set to normalized dipoles]\label{prop:generic_upper_bound_worst_sketching_error}
Consider a kernel $\kappa$, a family $\mathcal{T}$, and an integer $k \geq 1$ such that $\kappa$ has its 
$2k$-coherence with respect to $\mathcal{T}$ bounded by $0 \leq c< 1$.
Consider a sketching operator $\mathcal{A}$ defined via \eqref{eq:DefSketchingOperator} with any feature map $\Phi(\cdot)$ such that $\mathcal{A}\pi_{\theta}$ is well-defined for every probability distribution in the family $\mathcal{T}$. We have
\begin{equation}\label{eq:upper_bound_worst_sketching_error_two_sides}
\delta(\mathcal{S}_k| \mathcal{A}) \leq \frac{1}{1-c}\Big( c +\delta(\mathfrak{D}| \mathcal{A}) + (2k-1) \mu(\mathfrak{D}_{\neq}^2| \mathcal{A}) \Big).
\end{equation}

\end{proposition}
The proof, which is given in~\Cref{sec:proof dipole decomp of secant bound}, is essentially an adaptation of 
a bound of the restricted isometry constant for incoherent dictionaries in sparse recovery, see e.g. \citep[Chapter 5]{FoRa13}. The minor technicality is to take into account deviations to the normalization of dictionary columns, which is captured by the term $\delta(\mathfrak{D}|\mathcal{A})$.

The upper bound~\eqref{eq:upper_bound_worst_sketching_error_two_sides} reduces the study of $\delta(\mathcal{S}_{k}| \mathcal{A})$ to that of $\delta(\mathfrak{D}| \mathcal{A})$ and $\mu(\mathfrak{D}_{\neq}^2| \mathcal{A})$. Note that this bound is only relevant if we can ensure that $\delta(\mathcal{S}_{k}|\mathcal{A}) < 1$ when $\delta(\mathfrak{D}| \mathcal{A})$ and $\mu(\mathfrak{D}_{\neq}^2| \mathcal{A})$ are sufficiently small, i.e., if  $c/(1-c) <1$, which is possible to achieve in practice by a proper selection of the parameters of the kernel (see~\Cref{ex:KernelExamplesIncoherent}).

We now push the analysis further to scrutinize $\delta(\mathfrak{D}|\mathcal{A})$ and $\mu(\mathfrak{D}_{\neq}^{2}| \mathcal{A})$.  As these two quantities are defined as suprema of a function defined on $\mathfrak{D}$ and $\mathfrak{D}_{\neq}^{2}$ respectively, which are abstract sets of measures for which the topology is hard to grasp intuitively, we show that both $\delta(\mathfrak{D}|\mathcal{A})$ and $\mu(\mathfrak{D}_{\neq}^{2}| \mathcal{A})$ boil down to suprema of functions defined on subsets of $\mathbb{R}^d$. 

\emph{From now on we specialize to a location-based family $\mathcal{T}$
and a shift-invariant kernel $\kappa$ that is locally characteristic with respect to $\mathcal{T}$}. In this setting we have the following property of normalized dipoles  \citep[Lemma C.1]{GrBlKeTr20}
\begin{equation*}\label{prop:ndipole_param}
\mathfrak{D} = \Bigg\{\frac{\nu}{\|\nu\|_{\kappa}}, \:\: \nu = \|\pi_0\|_{\kappa}^{-1}s( \pi_{\theta'}- \alpha \pi_{\theta}); s \in \{-1,1\}, 0 \leq \alpha \leq 1, 0 <\varrho(\theta,\theta')\leq 1 \Bigg\}\, ,
\end{equation*}
where since $\kappa$ is locally characteristic we have $\|\nu\|_{\kappa}>0$.
In other words, for such $\kappa$ and $\mathcal{T}$, a normalized dipole is characterized by a sign $s \in \{-1,1\}$, the two nodes $\theta, \theta' \in \Theta$ that satisfies $0<\varrho(\theta,\theta') \leq 1$ and a parameter $\alpha \in [0,1]$ that characterizes how balanced is the normalized dipole. This suggests the following definitions.
\begin{definition}\label{def:NormMonoBalancedDipo}
Given a location-based family $\mathcal{T}$ and a shift-invariant kernel $\kappa$ that is locally characteristic with respect to $\mathcal{T}$, the set of normalized monopoles is defined by
\begin{equation}\label{prop:nmonopole_param}
\mathfrak{M} = \Bigg\{\frac{\nu}{\|\nu\|_{\kappa}}, \:\: \nu = \|\pi_0\|_{\kappa}^{-1}s \pi_{\theta}; s \in \{-1,1\}, \theta \in \Theta \Bigg\}.
\end{equation}
The set of balanced normalized dipoles is defined by
\begin{equation}\label{prop:bdipole_param}
\hat{\mathfrak{D}} = \Bigg\{\frac{\nu}{\|\nu\|_{\kappa}}, \:\: \nu = \|\pi_0\|_{\kappa}^{-1}s( \pi_{\theta'}- \pi_{\theta}); s \in \{-1,1\}, 0 <\varrho(\theta,\theta') \leq 1 \Bigg\}.
\end{equation}
\end{definition}
In a nutshell, normalized dipoles correspond to $\alpha = 0$, while normalized balanced dipoles correspond to $\alpha = 1$. 
Moreover, with a slight abuse of notation we define shorthands to denote the sets of all elements $(\iota,\iota') \in \mathfrak{D}_{\neq}^2$ (i.e., of pairs of separated normalized dipoles) where each elements is resticted to be either a monopole or a balanced dipole:
\begin{equation}\label{eq:DefineMDfamilies}
\mathfrak{M}_{\neq}^{2} := \mathfrak{M}^2\cap \mathfrak{D}_{\neq}^2, \:\: \mathfrak{M}\times\hat{\mathfrak{D}}_{\neq}:= (\mathfrak{M}\times\hat{\mathfrak{D}})\cap \mathfrak{D}_{\neq}^2, \:\: \hat{\mathfrak{D}}_{\neq}^{2} := \hat{\mathfrak{D}}^2\cap \mathfrak{D}_{\neq}^2.
\end{equation}
Now, we are ready to state the following result which is proved in~\Cref{app:pf_Th3-4}.

\begin{proposition}{(\textbf{From normalized dipoles to normalized monopoles and balanced dipoles})} \label{thm:upper_bound_sup_monopoles_dipoles}
Consider $\mathcal{T}= (\Theta, \rho, \mathcal{I})$ a location-based family with base distribution $\pi_{0}$ where $\rho(\cdot,\cdot) = \|\cdot-\cdot\|$ for some norm $\|\cdot\|$, and $\kappa$ a normalized shift-invariant kernel that is locally characteristic with respect to $\mathcal{T}$. Considering the sets of (normalized) monopoles and dipoles associated to $\mathcal{T}$,
and  $\mathcal{A}$ a WFF sketching operator (cf~\Cref{def:RFFsketching}) with arbitrary frequencies $\omega_{1},\ldots,\omega_{m}$, we have
\begin{equation}\label{eq:bound_diagonal_sketch_simplification_2}
\delta(\mathfrak{D}|\mathcal{A}) = \max\big(\delta(\mathfrak{M}|\mathcal{A}), \delta(\hat{\mathfrak{D}}|\mathcal{A})\big).
\end{equation}
If in addition $\kappa \geq 0$ we also have
\begin{equation}\label{eq:bound_cross_sketch_simplification}
1 \leq \frac{\mu(\mathfrak{D}_{\neq}^2| \mathcal{A})}{\max(\mu(\mathfrak{M}_{\neq}^2| \mathcal{A}), \mu(\hat{\mathfrak{D}}_{\neq}^2| \mathcal{A}), \mu(\mathfrak{M}\times\hat{\mathfrak{D}}_{\neq}| \mathcal{A}))}  \leq 3.
\end{equation}
The lower bound holds regardless of the assumption $\kappa \geq 0$.

\end{proposition}
Inspecting the proof shows that 
$\delta(\mathfrak{D}|\mathcal{A}) \geq \max\big(\delta(\mathfrak{M}|\mathcal{A}), \delta(\hat{\mathfrak{D}}|\mathcal{A})\big)$
is valid for any sketching operator such that $\mathcal{A} \pi_{\theta}$ is well defined for any distribution in the family $\mathcal{T}$. Similarly the lower bound in~\eqref{eq:bound_cross_sketch_simplification} holds under this relaxed assumption. It remains open whether the converse bounds extend (possibly with weaker constants) beyond the case of WFF operators and location-based families.
It also remains open whether the upper bound in~\eqref{eq:bound_cross_sketch_simplification} (or a qualitatively equivalent but larger bound) still holds without the assumption that $\kappa \geq 0$. This is left to future work, as this assumption is satisfied by all concrete kernels we will work with.

\subsubsection{Expression using the supremum of certain empirical processes}\label{sec:introducing_suprema}

The main overall consequence of~\Cref{prop:generic_upper_bound_worst_sketching_error} and~\Cref{thm:upper_bound_sup_monopoles_dipoles} is that under the appropriate assumptions we have
\begin{equation}\label{eq:MainDeterministicBound}
\delta(\mathcal{S}_{k}|\mathcal{A}) \leq \frac{1}{1-c} \Big(
c+ \max\big(\delta(\mathfrak{M}|\mathcal{A}),\delta(\hat{\mathfrak{D}}|\mathcal{A})\big)
+(6k-3)\max\big(\mu(\mathfrak{M}_{\neq}^2| \mathcal{A}),\mu(\hat{\mathfrak{D}}_{\neq}^2| \mathcal{A}),\mu(\mathfrak{M}\times\hat{\mathfrak{D}}_{\neq}| \mathcal{A}))\big)
\Big)
\end{equation}
As we now show, the advantage behind this dissection is that the study of the quantities $\delta(\hat{\mathfrak{D}}|\mathcal{A}), \mu(\mathfrak{M}_{\neq}^2|\mathcal{A}), \mu(\mathfrak{M}\times \hat{\mathfrak{D}}|\mathcal{A}), \mu(\hat{\mathfrak{D}}_{\neq}^2|\mathcal{A})$ boils down to the study of suprema of the absolute value of auxiliary functions  defined as empirical means.
We prove in~\Cref{proof:cor_parametrization_functions_cross_dipoles} the following result.

\begin{proposition}\label{thm:parametrization_functions_cross_dipoles}
Consider $\kappa$ a normalized shift-invariant kernel, $\mathcal{T}= (\Theta, \rho, \mathcal{I})$ a location-based family with base distribution $\pi_{0}$ where $\rho(\cdot,\cdot) = \|\cdot-\cdot\|$ for some norm $\|\cdot\|$, and assume that  $\kappa$ is locally characteristic with respect to $\mathcal{T}$. Let $\bm{\Omega} \in \mathbb{R}^{d \times m}$ be a matrix with arbitrary columns $\omega_{1},\ldots,\omega_{m}$ and
$\mathcal{A} =\mathcal{A}_{\bm{\Omega}}$ be a WFF sketching operator (cf~\Cref{def:RFFsketching}) with frequencies $\omega_{1},\ldots,\omega_{m}$. With $\phi_{\omega}$ defined as in~\eqref{eq:FFComponent},
define for $\omega \in \mathbb{R}^d$, $x,x' \in \mathbb{R}^{d}$ such that $\bar{\kappa}(x) < 1$ and $y \in \mathbb{R}^{d}$  
\begin{align}
\psi(\omega) &:= \frac{|\langle \pi_0,\phi_\omega \rangle|^2}{\|\pi_0\|_\kappa^2} \label{eq:DefRho}\\
 f_{\mathrm{d}}(x|\omega)&:= \frac{2 \sin^{2}\big(\langle \omega, x \rangle/2 \big) }{1-\bar{\kappa}(x)}   \label{eq:DefPsiD} \\
 f_{\mathrm{mm}}(y|\omega)&:= \cos\big(\langle \omega, y \rangle \big)  \label{eq:DefPsiMM} \\
 f_{\mathrm{md}}(x,y|\omega)&:= 2 \frac{\sin \big(\langle \omega, x\rangle/2 \big)\sin \big(\langle \omega, y+x/2\rangle \big) }{ \sqrt{2(1-\bar{\kappa}(x))}} \label{eq:DefPsiMD} \\
  f_{\mathrm{dd}}(x_1,x_2,y|\omega)&:= 
4 \frac{\sin\big(\langle \omega, x_1/2 \rangle \big)\sin\big(\langle\omega, x_2/2\rangle \big)\cos\big( \langle \omega, y+x_2/2-x_1/2 \rangle \big)}{ \sqrt{2(1-\bar{\kappa}(x_1))}\sqrt{2(1-\bar{\kappa}(x_2))}}. \label{eq:DefPsiDD}
\end{align}
Denote $\bm{\Omega} \in \mathbb{R}^{d \times m}$ the matrix with columns $\omega_{j}$, $1 \leq j \leq m$ and $\Psi_{\mathrm{m}}(\bm{\Omega}):= \frac{1}{m}\sum\limits_{j=1}^{m} \psi(\omega_j)$ and for $\ell \in \{\mathrm{d,mm,md,dd}\}$ 
\begin{equation}\label{eq:DefPsiEll}
\Psi_{\ell}(\cdot|\bm{\Omega}):= \frac{1}{m}\sum\limits_{j=1}^{m} \psi(\omega_{j}) f_{\ell}(\cdot|\omega_j).
\end{equation}
Denote $\Theta-\Theta := \{x-x': (x,x') \in \Theta^{2}\}$.
With the sets of (normalized) monopoles and dipoles associated to $\mathcal{T}$ as defined in \eqref{prop:nmonopole_param}, \eqref{prop:bdipole_param} and ~\eqref{eq:DefineMDfamilies}, we have 
\begin{align}\label{eq:bound_gammma_m}
\delta(\mathfrak{M}|\mathcal{A})& = 
\big|1-\Psi_{\mathrm{m}}(\bm{\Omega}) \big|,\\
\label{eq:bound_gammma_d}
\text{with}\:\: \Theta_{\mathrm{d}} := \Big\{ x \in \Theta-\Theta, \:\: 0< \|x\| \leq 1  \Big\}, \:\:\:\: \text{we have}  \:\: \delta(\hat{\mathfrak{D}}|\mathcal{A}) & = \sup\limits_{x \in \Theta_{\mathrm{d}}} \big| 1-\Psi_{\mathrm{d}}(x|\bm{\Omega})\big|,\\ 
\label{eq:bound_gammma_mm}
\text{with} \:\: \Theta_{\mathrm{mm}} := \Big\{ y \in 
\Theta-\Theta, \:\: 1\leq \|y\| \Big\}, \:\:\:\: \text{we have} \:\: \mu(\mathfrak{M}_{\neq}^2|\mathcal{A}) &  = \sup\limits_{y \in \Theta_{\mathrm{mm}}} \big|\Psi_{\mathrm{mm}}(y|\bm{\Omega})\big|,\\
\label{eq:bound_gammma_md}
\text{there exists a set}\:\:  \Theta_{\mathrm{md}} \subset \Theta_{\mathrm{d}} \times \Theta_{\mathrm{mm}}, \:\: \text{s.t.} \:\: \mu(\mathfrak{M}\times\hat{\mathfrak{D}}_{\neq}|\mathcal{A}) & 
 =\sup\limits_{(x,y) \in \Theta_{\mathrm{md}}}\big|\Psi_{\mathrm{md}}(x,y|\bm{\Omega})\big|, \\
\label{eq:bound_gammma_dd}
\text{there exists a set}\:\: \Theta_{\mathrm{dd}} \subset \Theta_{\mathrm{d}} \times \Theta_{\mathrm{d}} \times \Theta_{\mathrm{mm}}, \:\: \text{s.t.} \:\: \mu(\hat{\mathfrak{D}}_{\neq}^2|\mathcal{A}) &= \sup\limits_{(x_1,x_2,y)\in \Theta_{\mathrm{dd}}}\big|\Psi_{\mathrm{dd}}(x_1,x_2,y|\bm{\Omega})\big|. 
\end{align}
\end{proposition}
NB: Since $\kappa$ is locally characteristic, $\bar{\kappa}(x) < 1$ for $x \in \Theta_{\mathrm{d}}$ hence all of the above functions are well defined.

\subsubsection{Lipschitz property and covering numbers}

The study of the suprema of functions $\Psi_{\ell}(z|\bm{\Omega})$ (as defined in~\eqref{eq:DefPsiEll}) for random i.i.d. frequencies $\omega_{j}$ is classical and fits within the general theory of empirical processes.  It typically relies on establishing pointwise concentration inequalities for $\Psi_{\ell}(z|\bm{\Omega})$ and showing that with high probability on the draw of frequencies $\bm{\Omega}$ the function $\Psi_{\ell}(\cdot|\bm{\Omega})$ is Lipschitz with respect to a metric $\Delta_{\ell}$ on $\Theta_{\ell}$ such that the covering numbers of $\Theta_{\ell}$ with respect to $\Delta_{\ell}$ are well controlled.

The following result establishing a Lipschitz bound is proved in~\Cref{sec:proofthemdeltabounds}.

\begin{theorem}[Lipschitz bound]\label{thm:delta_bounds}
Let $\mathcal{T} := (\Theta,\rho,\mathcal{I})$ 
be a location-based family with base distribution $\pi_{0}$ on $\mathbb{R}^{d}$, with $\Theta \subseteq \mathbb{R}^{d}$ a bounded subset, $\rho(\cdot,\cdot) := \|\cdot-\cdot\|$ where
 $\|\cdot\|$ is some norm on $\mathbb{R}^{d}$.
Let  $\kappa$ be a non-degenerate normalized shift-invariant kernel on $\mathbb{R}^{d}$. Assume that there is some norm $\|\cdot\|_{a}$ on $\mathbb{R}^{d}$ and a function $\tilde{\kappa}: [0,+\infty) \to \mathbb{R}$ such that, with $R := \sup_{x \in \Theta_{\mathrm{d}}} \|x\|_{a}$, the normalized kernel $\bar{\kappa}(x)$ defined in~\eqref{eq:DefNormalizedKernel} satisfies for every $x \in \mathbb{R}^{d}$ such that $\|x\|_{a} \leq R$
\begin{equation}\label{eq:DefRadialKernel}
\bar{\kappa}(x) = \tilde{\kappa}(\|x\|_{a})\, .
\end{equation}
Assume that the following function is of class $\mathcal{C}^1$ on $(0,R)$
\begin{equation}
\alpha: r >0 \mapsto \alpha(r) := \frac{r}{\sqrt{1-\tilde{\kappa}(r)}}\label{eq:DefAlphaFn}
\end{equation}
 and that
\begin{equation}
\label{eq:conditions_on_alpha}
C_{\kappa}
 := 
 \sup_{0<r\leq R} \max(1,\alpha^{2}(r),|\alpha'(r)|^{2}) < \infty.
\end{equation}
Consider $\psi(\cdot)$ defined as in~\eqref{eq:DefRho} with $\phi_{\omega}$ defined as in~\eqref{eq:FFComponent}, $\bm{\Omega} \in \mathbb{R}^{d \times m}$ with arbitrary columns $\omega_1, \dots, \omega_m$, and $\Psi_{\ell}(\cdot|\bm{\Omega}), \Theta_{\ell}$ defined as in~\Cref{thm:parametrization_functions_cross_dipoles}  for $ \ell \in \{\mathrm{d,mm,md,dd}\}$.
Then, we have for each $ \ell \in \{\mathrm{d,mm,md,dd}\}$
\begin{equation}\label{eq:MetricDomination}
\forall z,z' \in \Theta_{\ell}, \:\:
\left| \Psi_{\ell}(z|\bm{\Omega})-\Psi_{\ell}(z'|\bm{\Omega})\right|
 \leq 
6 
\Psi_{0}(\bm{\Omega})
\cdot C_{\kappa} \cdot\ 
 \Delta_{\ell}(z,z'),
\end{equation}
where 
\begin{equation}\label{eq:M_Omega_def}
\Psi_{0}(\bm{\Omega})
:= 
\frac{1}{m} \sum_{j=1}^{m} \psi(\omega_{j}) 
f_{0}(\omega_{j}),\quad
f_{0}(\omega) :=  \sum_{i=1}^{3} \|\omega\|_{a,\star}^{i}\, .
\end{equation}
and the metrics $\Delta_{\ell}$ on the domains $\Theta_{\ell}$, $\ell \in \{\mathrm{d,mm,md,dd}\}$
are defined as
\begin{align}
 \Delta_{\mathrm{d}}(x,x') & := |\|x\|_{a}-\|x'\|_{a}|
 + \bigg\|\frac{x}{\|x\|_{a}} - \frac{x'}{\|x'\|_{a}} \bigg\|_{a}, \label{eq:DefMetricD}\\  
 \Delta_{\mathrm{mm}}(y,y') & := \|y-y'\|_{a}, \label{eq:DefMetricMM}\\
 \Delta_{\mathrm{md}}((x,y),(x',y')) & := \Delta_{\mathrm{d}}(x,x') + \Delta_{\mathrm{mm}}(y,y'), \label{eq:DefMetricMD}\\
 \Delta_{\mathrm{dd}}((x_1,x_2,y),(x_{1}',x_{2}',y')) & := \Delta_{\mathrm{d}}(x_1,x_1') + \Delta_{\mathrm{d}}(x_2,x_2') + \Delta_{\mathrm{mm}}(y,y'), \label{eq:DefMetricDD}
\end{align}

\end{theorem}

Covering numbers are controlled using the following result established in~\Cref{pf:prop:covering_delta_metric}.

\begin{proposition}[Covering numbers]\label{prop:covering_delta_metric}

Define $D := \operatorname{diam}_{a}(\Theta) = \sup_{x \in \Theta-\Theta} \|x\|_{a}$. 
For every $\tau>0$, and for each $\ell \in \{\mathrm{d,mm,md,dd}\}$, we have
\begin{equation*}
 \mathcal{N}_{\ell}(\tau):= \mathcal{N}( \Theta_{\ell}, \Delta_{\ell},\tau) \leq 
  \big(1+64(D+1)/\tau\big)^{3d+2} ,
\end{equation*}
\rvs{where $\mathcal{N}( \Theta_{\ell}, \Delta_{\ell},\tau)$ is the covering number of $\Theta_{\ell}$ at scale $\tau$ with respect to $\Delta_{\ell}$.}

\end{proposition}

\subsection{Results for random sketching}\label{sec:generic_result_random_sketching}

In this section, we leverage~\Cref{sec:sharp_determistic_bounds} to establish RIP results for random sketching. 
We first establish a generic result before exploiting it for specific examples and showing that it allows to improve upon and to extend existing work from the literature.

\begin{theorem}\label{thm:RIP_without_weights}
Consider $\mathcal{T}= (\Theta, \rho, \mathcal{I})$, $\kappa \geq 0$ and $\|\cdot\|_{a}$ satisfying the assumptions of~\Cref{thm:delta_bounds}, and $C_{\kappa}$ as in~\eqref{eq:conditions_on_alpha}. 
Assume that $\kappa$ has its mutual coherence with respect to $\mathcal{T}$ bounded by $0<\mu<1/10$.
Let $k \geq 1$ be an integer such that $1 \leq k <\frac{1}{10 \mu}$, and denote $c := (2k-1)\mu$.
Let $w$ be a $\kappa$-compatible weight function (cf \Cref{def:KappaCompatible}). 
Consider an integer $m \geq 1$  and $\bm{\Omega} \in \mathbb{R}^{d \times m}$  a random matrix (possibly with non i.i.d. columns $\omega_{1},\ldots,\omega_{m}$), such that the average marginal density of the $\omega_{j}$'s satisfies~\eqref{eq:ImportanceSamplingStructured}. Denote $\Psi_{\ell}(\cdot|\bm{\Omega}), \Theta_{\ell}$ as in~\Cref{thm:parametrization_functions_cross_dipoles}  for $ \ell \in \{\mathrm{d,mm,md,dd}\}$, and $\Psi_{0}(\bm{\Omega})$ as in~\Cref{thm:delta_bounds}. 

Given any $M>0$, $0<\tau<1-5c$, $v>0$, $\gamma >0$, if the following inequalities hold
\begin{align}\label{eq:MOmega_leq_M}
\mathbb{P}\Big(\Psi_{0}(\bm{\Omega})  > M\Big) & \leq \gamma \exp\Big(-\frac{m}{v}\Big)\, , \\ \label{eq:assumption_psi_m_Omega}
\mathbb{P}\Big(|\Psi_{\mathrm{m}}(\bm{\Omega})-\mathbb{E}\Psi_{\mathrm{m}}(\bm{\Omega})| > \frac{\tau}{4} \Big) & \leq 2\exp\Big(-\frac{m}{v}\Big)\, ,\\
\label{eq:assumption_psi_d_Omega}
\forall z \in \Theta_{\mathrm{d}}, \:\:  \mathbb{P}\Big(|\Psi_{\mathrm{d}}(z|\bm{\Omega})-\mathbb{E}\Psi_{\mathrm{d}}(z|\bm{\Omega})| > \frac{\tau}{8} \Big) 
&\leq 2\exp\Big(-\frac{m}{v}\Big)\, ,\\
\label{eq:assumption_psi_ell_Omega}
\forall \ell \in \{\mathrm{mm,md,dd}\}, \:\: \forall z \in \Theta_{\ell}, \:\:  
\mathbb{P}\Big(|\Psi_{\ell}(z|\bm{\Omega}) - \mathbb{E}\Psi_{\ell}(z|\bm{\Omega})| > \frac{\tau}{16k}\Big) 
& \leq 2\exp\Big(-\frac{m}{v}\Big)\, ,
\end{align}
then the $w$-FF sketching operator  $\mathcal{A}$ (cf~\Cref{def:RFFsketching}) with frequencies $\omega_{1},\ldots,\omega_{m}$ satisfies
\begin{equation*}\label{eq:MainProbaControl}
\mathbb{P} \bigg( \delta(\mathcal{S}_{k}|\mathcal{A}) >
\frac{4c+\tau}{1-c}
\bigg) \leq 
(10+\gamma) \cdot
  \exp\Big(-\frac{m}{v}\Big) \big(1+C/\tau)^{3d+2} ,
\end{equation*}
where
\begin{equation*}\label{eq:ConstantGenericRandomSketch}
C:= 6144 C_{\kappa}M \cdot k(1+\operatorname{diam}_{a}(\Theta)).
\end{equation*}
\end{theorem}
Under the assumptions of the theorem, a consequence is that for any $0<\eta<1$
\begin{equation*}
m \geq v \Big( (3d+2) \log(1+C/\tau) + \log((10+\gamma)/\eta) \Big) \implies \mathbb{P} \bigg( \delta(\mathcal{S}_{k}|\mathcal{A}) >  \frac{4c+\tau}{1-c} \bigg) \leq \eta.
\end{equation*}
hence estimating the order of magnitude of $v$, $C$ and $\tau$ satisfying the assumptions of the theorem is key to estimate a sufficient sketch size. \rvs{The assumptions \eqref{eq:assumption_psi_m_Omega}, \eqref{eq:assumption_psi_d_Omega} and \eqref{eq:assumption_psi_ell_Omega} in \Cref{thm:RIP_without_weights} hold under assumption $(66)$ of \citep{GrBlKeTr21}[Theorem 5.13]. The assumptions on the kernel $\kappa$ are similar too. Assumption \eqref{eq:MOmega_leq_M} does not appear as such in \citep{GrBlKeTr21}[Theorem 5.13], but it does hold when the weighted RFFs of \citep{GrBlKeTr21} are used. Moreover, since the quantity $\Psi_{0}(\bm{\Omega})$ appearing in Assumption \eqref{eq:MOmega_leq_M} is essentially a Lipschitz constant (cf~\eqref{eq:MetricDomination}), this assumption is weaker than assumption $\sup_{\nu \in \mathcal{S}_k} \|\nu\|_{\mathcal{F}} <+\infty$ required in \citep{GrBlKeTr21}, which ensures a {\em deterministic control} on the Lipschitz constant.  Overall, \Cref{thm:RIP_without_weights} thus recovers  \citep{GrBlKeTr21}[Theorem 5.13]. As \eqref{eq:MOmega_leq_M} also holds {\em without} weighting the RFFs, \Cref{thm:RIP_without_weights} indeed improves over  \citep{GrBlKeTr21}[Theorem 5.13].}  The proof of Theorem~\ref{thm:RIP_without_weights} is given in Appendix~\ref{sec:proof_RIP_without_weights}.

 Our next result establishes the concentration inequalities~\eqref{eq:assumption_psi_m_Omega}-\eqref{eq:assumption_psi_d_Omega}-\eqref{eq:assumption_psi_ell_Omega} under a sub-exponentiality assumption on functions associated to the random frequencies 
$\omega_{j}$, $1 \leq j \leq m$. \begin{definition}\label{eq:DefSubExp}
A real-valued random variable is sub-exp($\nu,\beta$), where $\nu,\beta \geq 0$,  if 
\begin{equation*}\label{eq:sub_exp_def}
\mathbb{E}e^{\lambda(X-\mathbb{E}X)} \leq e^{\frac{\lambda^2\nu^2}{2}}, \:\: \forall |\lambda| \leq \frac{1}{\beta}.
\end{equation*}
The case $\beta = 0$ corresponds to a sub-Gaussian variable.
\end{definition}
If $X$ is sub-exp($\nu,\beta$) then by the standard Cramér-Chernoff method\footnote{See e.g. \citep[Section 2.2]{BoLuMa13} and the proof of Theorem 2.8.1. in \citep{Ver18}.}
\begin{equation}\label{eq:ConcentrationSubExp}
\forall t > 0, \:\: \mathbb{P}\Big(\big|X-\mathbb{E}X\big|>t\Big) \leq 2\max\Big(e^{-\frac{t^2}{2\nu^2}}, e^{-\frac{t}{\beta}} \Big)
\leq 2 \exp\left(-\frac{t^{2}}{2\nu^{2}+\beta t}\right).
\end{equation} 
We establish the inequalities~\eqref{eq:assumption_psi_m_Omega}-\eqref{eq:assumption_psi_d_Omega}-\eqref{eq:assumption_psi_ell_Omega} by showing that $\Psi_{\mathrm{m}}(\bm{\Omega})$ and $\Psi_{\mathrm{\ell}}(z|\bm{\Omega})$, $\ell \in \{\mathrm{d},\mathrm{mm},\mathrm{md},\mathrm{dd}\}$, $z \in \Theta_{\ell}$, are sub-exponential with controlled expectations. A well-known property of sub-exp random variables is that if $X_{1}, \dots, X_{m}$ are independent sub-exp($\nu,\beta$)  then  $\frac{1}{m}\sum_{j=1}^{m}X_j$ is sub-exp($\nu/\sqrt{m},\beta/m$). Thus, when the frequencies $\omega_{j} \sim \Lambda$, $1 \leq j \leq m$ are i.i.d. random variables, in order to prove that $\Psi_{m}(\bm{\Omega})$ is sub-exp($\nu/\sqrt{m},\beta/m$), it is enough to prove that the random variables $\psi(\omega_j)$, $1 \leq j \leq m$, are sub-exp($\nu,\beta$). Similarly, for $\ell \in \{ \mathrm{d},\mathrm{mm},\mathrm{md},\mathrm{dd}\}$ and $z \in \Theta_{\ell}$, in order to prove that $\Psi_{\ell}(z|\bm{\Omega})$ is sub-exp($\nu/\sqrt{m},\beta/m$) it is enough to prove that the random variables $\psi(\omega) f_{\ell}(z|\omega)$, $\omega \sim \Lambda$, are sub-exp($\nu,\beta$). For this purpose, the following lemma (proved in  \Cref{proof:BoundFellOmegaP}) will be crucial.
\begin{lemma}\label{lem:BoundFellOmegaP}
Under the assumptions of  \Cref{thm:RIP_without_weights_random_bis} (see below), for each
 $\ell \in \{\mathrm{d},\mathrm{mm},\mathrm{md}\}$ and $z \in \Theta_{\ell}$, there exists  $x'_{0} \in \mathbb{R}^{d}$ satisfying $\|x'_{0}\|_{a}=1$ such that
\begin{equation}\label{eq:PfThm5MainStep}
|f_{\ell}(z|\omega)| \leq (\sqrt{C_\kappa} |\langle \omega, x'_{0}\rangle|)^{p_{\ell}}, \quad \forall \omega \in \mathbb{R}^{d},\quad
\text{with}\ p_{\mathrm{d}}=2,\ p_{\mathrm{mm}}=0,\ p_{\mathrm{md}}=1.
\end{equation}
Moreover, for each $z \in \Theta_{\mathrm{dd}}$, there are $x'_{i}\in \mathbb{R}^{d}$ such that $\|x'_{i}\|_{a}=1$, $i =1,2$ and
\[
|f_{\mathrm{dd}}(z|\omega)| \leq \frac{C_{\kappa}}{4} \Big(\langle \omega,x'_1\rangle^2+\langle \omega,x'_2 \rangle^2\Big), \quad \forall \omega \in \mathbb{R}^{d}.
\]
\end{lemma}
Using \Cref{lem:BoundFellOmegaP}, we obtain that the random variables $\psi(\omega)f_{\ell}(z|\omega)$ are (almost surely) bounded by random variables of the form $\psi(\omega)(\sqrt{C_{\kappa}}|\langle \omega, x \rangle|)^{p}$, with $x \in \mathbb{R}^d$ and $p \in \{0,1,2\}$, allowing to leverage the following lemma proved in \Cref{sec:proofsubexps}.
\begin{lemma}\label{lemma:sub_exp_2}
Consider real-valued random variables $X,Y$ where $Y$ is sub-exp($\nu,\beta$) and
$|X| \leq Y$ almost surely.
Then $X$ is sub-exp($\nu',\beta$), where
\begin{equation*}\label{eq:SubExpXfromBoundY}
\nu':= \sqrt{2\nu^{2}+16(\mathbb{E}(Y))^{2}}.
\end{equation*}
\end{lemma}
The following theorem considers a slightly generalized case with block-i.i.d. variables, covering structured random features, \rvs{ which are mentioned at the end of \Cref{sec:RFF}}.

\begin{theorem}\label{thm:RIP_without_weights_random_bis}
Consider $\mathcal{T} := (\Theta,\rho,\mathcal{I})$  a location-based family with base distribution $\pi_{0}$ on $\mathbb{R}^{d}$, with $\Theta \subseteq \mathbb{R}^{d}$ a bounded subset, $\rho(\cdot,\cdot) := \|\cdot-\cdot\|$ where
 $\|\cdot\|$ is some norm on $\mathbb{R}^{d}$.
Let  $\kappa \geq 0$ be a non-degenerate normalized shift-invariant kernel on $\mathbb{R}^{d}$, and assume that there is some norm $\|\cdot\|_{a}$ on $\mathbb{R}^{d}$ and a function $\tilde{\kappa}: [0,+\infty) \to \mathbb{R}$ such that, with $R := \sup_{x \in \Theta_{\mathrm{d}}} \|x\|_{a}$, the normalized kernel $\bar{\kappa}(x)$ defined in~\eqref{eq:DefNormalizedKernel} satisfies 
$\bar{\kappa}(x) = \tilde{\kappa}(\|x\|_{a})$ for every $x \in \mathbb{R}^{d}$ such that $\|x\|_{a} \leq R$. Assume that the function $\alpha$ defined in \eqref{eq:DefAlphaFn} is of class $\mathcal{C}^1$ on $(0,R)$ and the constant $C_{\kappa}$ defined in \eqref{eq:conditions_on_alpha} is finite.
Moreover, assume that $\kappa$ has its mutual coherence with respect to $\mathcal{T}$ bounded by $\mu$ where $0<\mu <\frac{1}{10}$.
Let $1 \leq k < \frac{1}{10\mu}$ and define $c := (2k-1)\mu$.

Let $w$ be a $\kappa$-compatible weight function and $m$ be an integer multiple of $b \in \mathbb{N}^{*}$, and consider $\mathcal{A}$ a $w$-FF sketching operator (Definition~\ref{def:RFFsketching}) associated to the frequencies $(\omega_{1},\ldots,\omega_{m})$ that are block-i.i.d. corresponding to $m/b$ i.i.d. $d\times b$ random matrices $\bm{B}_i$, $1 \leq i \leq m/b$ such that~\eqref{eq:ImportanceSamplingStructured} holds.
Let $\tau \in (0,1-5c)$, and assume that 
\begin{enumerate}
\item \label{cond:new_thm_0} there exists $\nu,B>0$ and $\beta \geq 0$ such that for each $x \in \mathbb{R}^{d}$ satisfying $\|x\|_{a}=1$, the following random variables are sub-exp($\nu,\beta$) with $|\mathbb{E}(Z_{p})| \leq B$
\begin{equation}\label{eq:Z_def_structured}
Z_{p} := \frac{1}{b}\sum_{j=1}^{b}\psi(\omega_{j})(\sqrt{C_{\kappa}}|\langle \omega_{j},x \rangle|)^{p}\, , \quad
p \in \{0,1,2\}. 
\end{equation} \label{cond:new_thm_1}
\item there exists $\gamma, M>0$ such that \begin{equation}\label{eq:assumption_on_psi0_thm5}
\mathbb{P}\Big(\Psi_{0}(\bm{\Omega})  > M\Big)  \leq \gamma\exp\Big(-\frac{m}{v}\Big).
\end{equation}
\begin{equation}\label{eq:variance_structured}
v:= \frac{256k^2b(2\nu'^2+\beta\tau)}{\tau^2}, \:\: \textit{where} \:\: \nu' := \sqrt{2}\sqrt{\nu^2 + 8 B^2}.
\end{equation} \label{cond:new_thm_2}
\end{enumerate}
Then $\mathcal{A}$ satisfies
\begin{equation}\label{eq:the_RIP}
\mathbb{P} \bigg( \delta(\mathcal{S}_{k}|\mathcal{A}) >
\frac{4c+\tau}{1-c}
\bigg) \leq 
(10+\gamma) \cdot
  \exp\Big(-\frac{m}{v}\Big) \big(1+C/\tau)^{3d+2},
\end{equation}
where
\begin{equation}\label{eq:C_def_thm5}
C:= 6144 C_{\kappa}M \cdot k(1+\operatorname{diam}_{a}(\Theta)).
\end{equation}
When $B_{\psi}:=\sup_{\omega \in \mathbb{R}^d} \psi(\omega) <+\infty$, \eqref{eq:the_RIP} also holds with $v$ from~\eqref{eq:variance_structured} by
replaced with
\begin{equation}\label{eq:variance_structured_bis}
v':= \frac{256k^2b(2B_{\psi}^{2}\nu'^2+B_{\psi}\beta\tau)}{\tau^2}
\end{equation} 
if we assume~\eqref{eq:assumption_on_psi0_thm5} with $v'$ instead of $v$ and replace
\Cref{cond:new_thm_0}
by the same assumption on the random variables
\begin{equation}\label{eq:Z_def_structured_bis}
Z'_{p} := \frac{1}{b}\sum_{j=1}^{b}(\sqrt{C_{\kappa}}|\langle \omega_{j},x \rangle|)^{p}\, , \quad
p \in \{0,1,2\},
\end{equation}
\end{theorem}

\Cref{thm:RIP_without_weights_random_bis} is obtained by applying \Cref{thm:RIP_without_weights}, see \Cref{proof:RIP_without_weights_random_bis}.

Next we give two examples: for mixtures of Gaussians, \Cref{cond:new_thm_2} is established using sub-exponentiality; for mixtures of Diracs, \Cref{cond:new_thm_2} requires a bit more work.

\paragraph{The case of a mixture of Gaussians.}
We consider the kernel $\kappa$ and the overall setting of Example~\ref{ex:KernelExamplesIncoherent}, and a $w$-RFF sketching operator $\mathcal{A}$ with ``flat'' $\kappa$-compatible weight function $w \equiv 1$ and i.i.d. frequencies $\omega \sim \mathcal{N}(0,\bm{\Sigma}^{-1}/s^{2})$. In this setting, the function $\psi$ defined in~\eqref{eq:DefRho} satisfies\footnote{See  \citep[Section 6.3.1]{GrBlKeTr20}.}
\begin{equation}\label{eq:psi_expression_gaussian}
\forall \omega \in \mathbb{R}^{d},\:\: \psi(\omega) = \frac{|\langle \pi_0, \phi_{\omega} \rangle|^2}{\|\pi_{0}\|_{\kappa}^2} = \frac{e^{-\omega^{\mathrm{T}} \bm{\Sigma} \omega}}{(1+2s^{-2})^{-\frac{d}{2}}}.
\end{equation}
Given the definition~\eqref{eq:M_Omega_def} of $f_{0}(\cdot)$, we deduce that $\psi(\omega)$ and  $\psi(\omega) f_{0}(\omega)$ are bounded, so that Hoeffding's inequality yields~\Cref{cond:new_thm_2} for an explicit $M$ independent of $m$ and for any $v>0$, while the variant of \Cref{cond:new_thm_1} with $Z'_{p}$ follows from the sub-exponentiality of $|\langle\omega,x\rangle|^{2}$, $s \in \{0,1,2\}$, since $\omega \sim \mathcal{N}(0,\bm{\Sigma}^{-1}/s^{2})$. Details are given in \Cref{proof:bounds_gaussian}, including sub-exponentiality constants and a proof that  $\mathcal{T}$ and $\kappa$ satisfy the assumptions of \Cref{thm:RIP_without_weights_random_bis}.
When all is said and done, we obtain the following result.

\begin{corollary}\label{cor:bounds_gaussian}
Consider $\Theta \subseteq \mathbb{R}^{d}$, an integer $k \geq 1$, a scale $s>0$, and
\begin{equation}\label{eq:SepVsScaleGMM}
\epsilon \geq \sqrt{2+s^2} (4\sqrt{\log(5ek)}).
\end{equation}
With $\mathcal{T}$, $\kappa$, $\bm{\Sigma}$  as in \Cref{ex:KernelExamplesIncoherent}  and $\mathcal{A}$ the $w$-RFF sketching operator  with ``flat'' $\kappa$-compatible weight function $w \equiv 1$ and $m$ i.i.d. frequencies $\omega_j \sim \Lambda := \mathcal{N}(0,\bm{\Sigma}^{-1}/s^2)$, the mutual coherence of $\kappa$ with respect to $\mathcal{T}$ is bounded by $\mu$ where $0<\mu <\frac{1}{10k}$.\\
Moreover, for each $0<\tau<1-5c$, where $c := (2k-1)\mu$, we have
\begin{equation}\label{eq:cor_bound_delta_s_k_gaussian}
\mathbb{P} \bigg( \delta(\mathcal{S}_{k}|\mathcal{A}) >
\frac{4c+\tau}{1-c}
\bigg) \leq 11  \exp\Big(-\frac{m}{v}\Big) \big(1+C/\tau)^{3d+2} ,
\end{equation}
where $v = v_{k}(\tau)$ with
\begin{align}\label{eq:DefVKTau}
v_{k}(\tau) := 512k^2  &\left((C_{0}/\tau)^{2}+\frac{1}{3}(C_{0}/\tau)\right), \:\: \textit{with} \:\:  C_{0}  \leq 7\sqrt{3}\epsilon^{2}s^{-2}(1+2s^{-2})^{d/2}, \\
\label{eq:DefCofVKTau}C& \leq \left(43000\epsilon^2(1+2s^{-2})^{d/2}\right) \cdot k(1+\operatorname{diam}_{a}(\Theta)).
\end{align}

\end{corollary}

In contrast to \citep[Theorem 6.11]{GrBlKeTr20}, the RIP constant here is not guaranteed to be (with high probability) arbitrarily close to zero, but only smaller than the quantity $(4c+\tau)/(1-c)$, which can be made arbitrarily close to $4c/(1-c) < 1$ (since $c=(2k-1)\mu<(2k-1)/10k < 1/5$). 
The assumption \eqref{eq:SepVsScaleGMM} relating the separation parameter
 $\epsilon$ and the scale parameter $s$ is essential to guarantee that $10k \mu <1$.
This technical condition is important since our bounds are only valid under the assumption that $5c = 5(2k-1)\mu<1$. In particular, we deduce that the probability that the event $\big\{\delta(\mathcal{S}_{k}|\mathcal{A}) \leq (4c+\tau)/(1-c)\big\}$ holds is larger than $1-\eta$, with $\eta \in ]0,1]$, whenever
$$m \geq (3d+2)v_{k}(\tau)\log \big(1+C/\tau\big)  + \log(11/\eta).  $$ 
In other words, a sufficient sketch size $m$ scales as $\mathcal{O}(dv_{k}(\tau)\log(C/\tau))$. Typically, we seek to determine the dependency of the sketch size in terms of the sparsity $k$ and the dimension $d$. Considering a near minimum separation parameter $\epsilon$ according to \eqref{eq:SepVsScaleGMM}, a fixed diameter, and $1 \leq \log k = \mathcal{O}(d)$, let us highlight two regimes: 
\begin{enumerate}
\item the regime $\sqrt{d} \lesssim s = \mathcal{O}(\mathtt{poly}(d))$: then $(1+2s^{-2})^{d/2} = \mathcal{O}(1)$ so that $C_{0} = \mathcal{O}(\log k)$, $v_{k}(\tau) = \mathcal{O}((\tau^{-1} k\log k)^2)$ and $\log(C)= \mathcal{O}(\log(kd))$ and the the sufficient sketch size scales as $\mathcal{O}( (\tau^{-1} k \log k)^{2} \log kd)$. 
\item  the regime where $s$ is of the order of one: then $(1+2s^{-2})^{d/2} = \mathcal{O}(e^{c d/2})$, with $c = \log(1+2s^{-2})$ so that $v_{k}(\tau) = \mathcal{O}((\tau^{-1}k \log k)^2 d e^{cd})$ and $\log(C)= \mathcal{O}(d)$ so that the sufficient sketch size scales as $\mathcal{O}((\tau^{-1}k \log k)^{2}d^2 e^{cd})$;
\end{enumerate}
In both regimes, we obtain similar results as in \citep[Table 1]{GrBlKeTr20}. 
There exists an intermediate regime, $c_{1} d^{1/4} \leq s^2 \leq c_{2}\sqrt{d}$, for large $d$, where we expect that  \Cref{thm:RIP_without_weights_random_bis} can be leveraged to obtain better sketch size estimates, that would not be achievable with the techniques of \citep{GrBlKeTr20}. A closer inspection of our proof techniques indeed suggests that better dependencies can be obtained 
by relying on \Cref{cond:new_thm_0} of \Cref{thm:RIP_without_weights_random_bis} (with $Z_{p}$) rather than on its variant with $Z'_{p}$.
Concretely, this means obtaining better sub-exponentiality constants for the random variables $\psi(\omega)|\langle \omega,x\rangle|^{p}$ . 
 
 As an example, for $p=0$, the proof given in \Cref{proof:bounds_gaussian} only relies on the crude uniform deterministic bound $\psi(\omega) \leq B_{\psi} := \sup_{\omega \in \mathbb{R}^d} \psi(\omega)  = 
 (1+2s^{-2})^{d/2}$, hence it cannot yield a better result than Hoeffding's bound \citep{Hoef94}
 \begin{equation*}
\forall \epsilon>0, \:\: \mathbb{P}\Big(\psi(\omega) -\mathbb{E}\psi(\omega) > \epsilon\Big) \leq \exp\Big(-\frac{2\epsilon^2}{B_{\psi}^2}\Big).
\end{equation*}
It is well known that exploiting the variance of $\psi(\omega)$, $V_{\psi}:=\mathbb{V}\psi(\omega)$ can lead to better results using Bernstein's concentration inequality~\citep[Theorem 2.10]{BoLuMa13}
\begin{equation*}
\forall \epsilon>0, \:\: \mathbb{P}\Big(\psi(\omega) -\mathbb{E}\psi(\omega) > \epsilon\Big) \leq \exp\Big(-\frac{\epsilon^2}{2(V_{\psi}+B_{\psi}\epsilon)}\Big).
\end{equation*}
 In the setting of \Cref{cor:bounds_gaussian}, 
since $\omega \sim \Lambda$ is Gaussian, we can compute explicitly $\mathbb{E}_{\omega \sim \Lambda}\psi(\omega) =1$ and $V_{\psi}:=\mathbb{V}\psi(\omega) \leq 
 \mathbb{E}_{\omega \sim \Lambda} \psi^{2}(\omega)= (1+4s^{-4}(1+4s^{-2})^{-1})^{d/2}$. Thus $\log(V_{\psi}) = d\log(1+4s^{-4}(1+4s^{-2})^{-1})/2$, while $\log(B_{\psi}) = d \log(1+2s^{-2})/2$, hence
 in the regime $c_{1} d^{1/4} \leq s^2 \leq c_{2}\sqrt{d}$ we have $V_{\psi} = \mathcal{O}(1)$ while $B_{\psi} \geq e^{c_{3}\sqrt{d}}$. 
Empirical experiments further suggest that even this Bernstein bound remains crude, and even essentially vacuous (on the order of one): this is illustrated on \Cref{fig:psi_histogram}, as well the behavior of the following \emph{conjectured} upper bound 
\begin{equation}
\mathbb{P}\Big(\psi(\omega) -\mathbb{E}\psi(\omega) > \epsilon\Big) \leq \exp\Big(-\frac{\epsilon^2}{2(V_{\psi}+\sqrt{V_{\psi}}\epsilon)}\Big).
\tag{Conjecture}
\end{equation}
which remains to be proved for a range of $0<\epsilon<\epsilon_{\max}$ to be determined. Proving this conjecture and similar ones for $\langle \omega,x\rangle^{p} \psi(\omega)$, $p=1,2$ would lead to definite improvements to the bounds derived in \Cref{cor:bounds_gaussian} and consecrate the advantage of the analysis developed in this work over the analysis given in \citep{GrBlKeTr21}, justifying the supplementary assumption on the RIP constant (to be larger than $4c/(1-c)$) that our analysis requires.
This is however left to future work.

\begin{figure}[h]
\centering
\includegraphics[width= 0.45\textwidth]{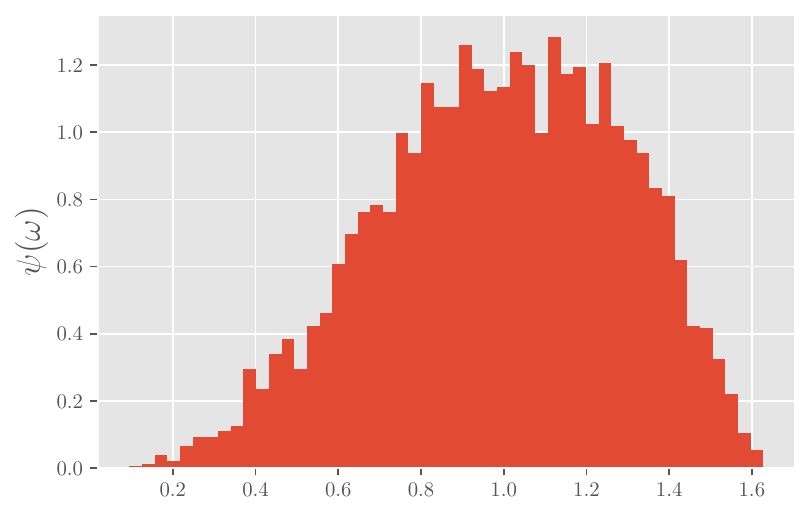}~\includegraphics[width= 0.45\textwidth]{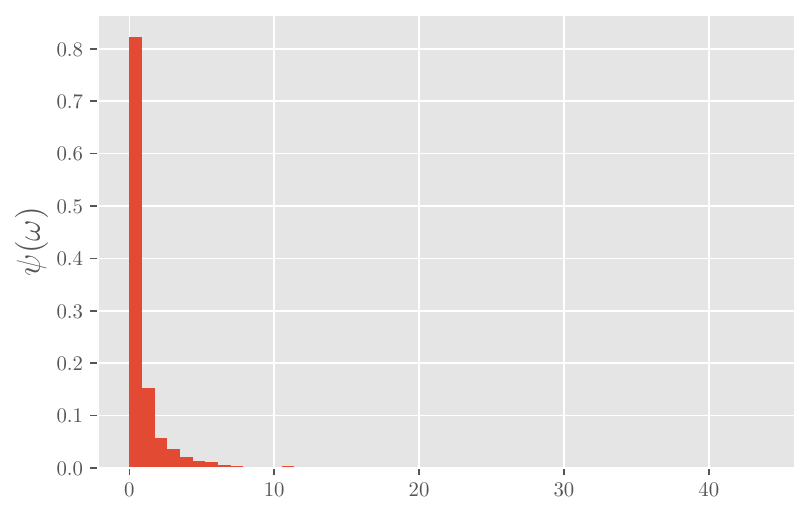}\\
\includegraphics[width= 0.45\textwidth]{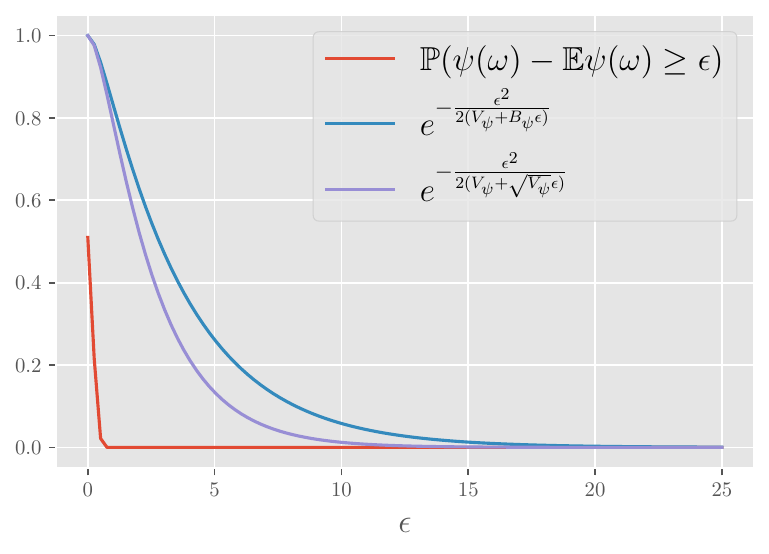}~\includegraphics[width= 0.45\textwidth]{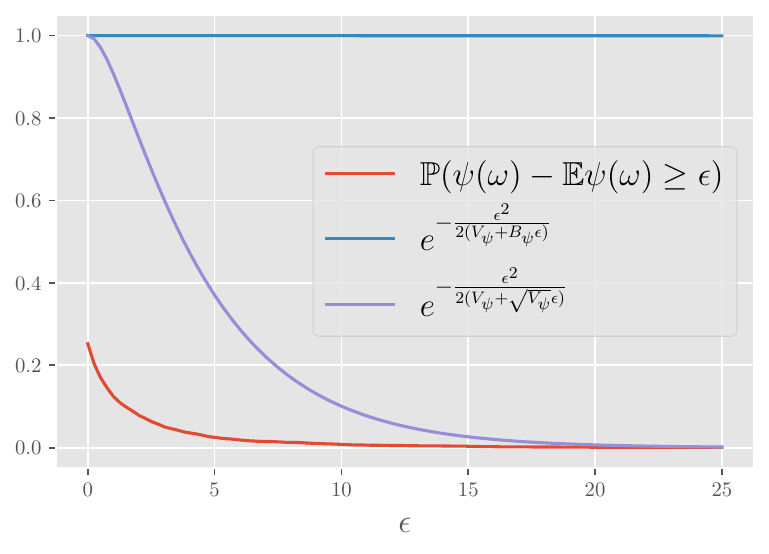}\\
\caption{(top) Histogram of $\psi(\omega)$ when $\omega \sim \mathcal{N}(0,s^{-2}\mathbb{I}_d)$; (bottom) Empirical graph of $\mathbb{P}(\psi(\omega) - \mathbb{E}\psi(\omega) \geq \epsilon)$ and two candidate
 analytic bounds for $s =3$ and $d = 5$ (left), $d=100$ (right).\label{fig:psi_histogram}}
\end{figure}

\paragraph{The case of a mixture of Diracs}
We now consider the setting of \Cref{ex:DiracMixture}, and a $w$-RFF sketching operator $\mathcal{A}$ corresponding to the ``flat'' $\kappa$-compatible weight function $w \equiv 1$ and i.i.d. frequencies $\omega \sim \mathcal{N}(0,s^{-2}\mathbb{I}_{d})$. In this setting, the function $\psi$ defined in~\eqref{eq:DefRho} satisfies\footnote{See again \citep[Section 6.3.1]{GrBlKeTr20}.} 
\begin{equation}\label{eq:psi_expression_dirac}
\forall \omega \in \mathbb{R}^{d},\:\: \psi(\omega) = \frac{|\langle \pi_0, \phi_{\omega} \rangle|^2}{\|\pi_{0}\|_{\kappa}^2} = 1.
\end{equation}
This is a setting where the analysis adopted in \citep{GrBlKeTr20} cannot be applied, since the condition~\eqref{eq:phi_omega_conditions} does not hold. Indeed, by~\eqref{eq:psi_expression_dirac} and by the fact that $\|\pi_0\|_{\kappa} =1$ we have $\sup_{\omega \in \mathbb{R}^{d}} 
|\langle  \phi_{\omega}^{w},\pi_{0} \rangle| \cdot \max(1,\|\omega\|_{\star},\|\omega\|_{\star}^2) = +\infty. $

Reasoning as in the case of a mixture of Gaussians allows to establish the variant of~\Cref{cond:new_thm_1} with $Z'_{p}$ required in \Cref{thm:RIP_without_weights_random_bis}.
However, in this setting, $\psi(\omega) f_{0}(\omega)$ is no longer sub-exponential hence establishing~\Cref{cond:new_thm_2} in \Cref{thm:RIP_without_weights_random_bis} requires more work and a choice of $M>0$ that depends at most polynomially on the sketch size $m$. As detailed in  \Cref{proof:bounds_gaussian}, we get the following result as a corollary of 
\Cref{thm:RIP_without_weights_random_bis}.

\begin{corollary}\label{cor:bounds_dirac}
Consider $\Theta \subseteq \mathbb{R}^{d}$, an integer $k \geq 1$, a scale $s>0$, and
\begin{equation}\label{eq:SepVsScaleDiracs}
\epsilon \geq s (4\sqrt{\log(5ek)}).
\end{equation}
With $\mathcal{T}$, $\kappa$, $\bm{\Sigma}$  as in \Cref{ex:DiracMixture}  and $\mathcal{A}$ the $w$-RFF sketching operator  with ``flat'' $\kappa$-compatible weight function $w \equiv 1$ and $m$ i.i.d. frequencies $\omega_j \sim \Lambda :=
\mathcal{N}(0,s^{-2} \mathbb{I}_{d})$,  where $\mathbb{I}_{d}$ is the identity matrix of dimension $d$, the mutual coherence of $\kappa$ with respect to $\mathcal{T}$ is bounded by $\mu$ where $0<\mu <\frac{1}{10k}$.\\
Moreover, for each $0<\tau<1-5c$, where $c := (2k-1)\mu$, we have

\begin{equation}\label{eq:bound_cor_diracs_delta}
\mathbb{P} \bigg( \delta(\mathcal{S}_{k}|\mathcal{A}) >
\frac{4c+\tau}{1-c}
\bigg) \leq 11 \exp\Big(-\frac{m}{v}\Big) \big(1+C(\tau,m)/\tau)^{3d+2} ,
\end{equation}
where
\begin{align}
v = v_{k}(\tau) &:= 512k^2  \left((C_{0}/\tau)^{2}+\frac{1}{3}(C_{0}/\tau)\right),
\ \text{with}\ C_{0} \leq 
7s^{-2}\max(1,\sqrt{3}\epsilon^{2}),
\label{eq:DefVKTauBis}\\
C(\tau,m) &:= \left(6144 (1+2s^{-1})^3 \max(1,\sqrt{3}\epsilon^{2})(2d^{3/2}+\sqrt{m}\tau^{3/2})\right) \cdot k(1+\operatorname{diam}_{a}(\Theta)).
\label{eq:DefCofVKTauBis}
\end{align}
\end{corollary}
As in the case of Gaussian mixtures, and in contrast to \citep[Theorem 6.11]{GrBlKeTr20}, the RIP constant is not guaranteed to be (with high probability) arbitrarily close to zero: it can only be made arbitrarily close to $4c/(1-c) < 1$. This is the price we pay for being able to handle plain importance sampling with $w \equiv 1$.
Observe that $1+C(\tau,m)/\tau \leq \sqrt{m} (1+C(\tau,1)/\tau) \leq m (1+C(\tau,1)/\tau)$ hence the r.h.s. of ~\eqref{eq:bound_cor_diracs_delta} is upper bounded by 
\begin{equation*}\label{eq:DiracSketchSize}
11 \exp\bigg(-\frac{m}{v} \Big( 1- (3d+2)v \frac{\log(m)}{m}\Big)+(3d+2)\log \big(1+C(\tau,1)/\tau\big) \Big)\bigg).
\end{equation*}
We deduce that for $\eta \in (0,1]$, and for the $w-RFF$ sketching operator described in Corollary~\ref{cor:bounds_dirac}, the probability that the event $\big\{\delta(\mathcal{S}_{k}|\mathcal{A}) \leq (4c+\tau)/(1-c)\big\}$ holds is larger than $1-\eta$, as soon as
\[
\frac{m}{\log m} \gtrsim (3d+2)v\log \big(1+C(\tau,1)/\tau\big)  + \log(11/\eta) =\Omega(k^2 d).  
\]
 In other words, a sufficient sketch size is $\mathcal{O}(k^2d)$: our analysis allows to obtain the same dependencies on $k$ and $d$ 
 as the analysis developed in~\citep{GrBlKeTr20} but \emph{without} assuming
  the conditions~\eqref{eq:phi_omega_conditions} that imposes constraints on the importance sampling weight $w$.
 It would be tempting to think that a judicious choice of the weight function $w$ would allow to further improve the dependency on $k$ of the sketch size from $\mathcal{O}(k^2d)$ to $\mathcal{O}(kd)$. Unfortunately, as we shall see in Section~\ref{sec:lower_bound}, our analysis does not allow us to make such an improvement. The investigation of the role of the weight function is deferred for future work. 
Finally, observe that the constant $\log(1+C(\tau,1)/\tau)$ depends logarithmically on $D$ (the diameter of $\Theta$), on $s^{-1}$, and on $d\log k$. 
Note that the logarithmic dependency on $D$ is well known in the Fourier features literature \citep{SrSz15}, while the dependency on $s^{-1}$ has empirical implications: the parameter $s$ should be choosen small enough so that the mutual coherence of $\kappa$ with respect to $\mathcal{T}$ is bounded by $\mu$, yet not too small. Interestingly, this phenomenon was documented in several empirical investigations~\citep{KeTrTrGr17,Cha20}.

\subsubsection{Towards bounds for structured random sketching}
A benefit of the theoretical analysis presented in this work is to pave the way to
 theoretical guarantees on the RIP of sketching operators based on structured features. 
Indeed, as evoked in \Cref{sec:RFF}, there are now many constructions of structured random Fourier features
 where $m$ is a multiple of $d$ and the frequencies $\omega_{1}, \dots, \omega_{m}$ are block-i.i.d. 
 with blocks of size $b=d$: the matrix $\bm{\Omega} \in \mathbb{R}^{d \times m}$ with columns $\omega_{j}$, $1 \leq j \leq m$ is the concatenation of i.i.d. random matrices $\bm{B}_i \in \mathbb{R}^{d\times d}$, $1 \leq i \leq m/d$. 
 
 Such constructions can be designed to lead to (non independent but) identically distributed Gaussian frequencies $\omega_{j} \sim \mathcal{N}(0,\sigma^{2}\mathbb{I}_{d})$, so that the average marginal density satisfies \eqref{eq:ImportanceSamplingStructured} with a Gaussian kernel and the simplest $\kappa$-compatible weight function $w \equiv 1$, see e.g. \citep[Chapter 5]{Cha20}. This allows to reuse the proof techniques used to establish \Cref{cor:bounds_gaussian}-\Cref{cor:bounds_dirac}, showing as an intermediate step that each $X_{j} := \psi(\omega_{j})(\sqrt{C_{\kappa}}|\langle \omega_{j},x \rangle|)^{p}$
 (resp. $X_{j} := (\sqrt{C_{\kappa}}|\langle \omega_{j},x \rangle|)^{p}$, when $B_{\psi}< +\infty$), with arbitrary $x \in \mathbb{R}^{p}$ satisifying $\|x\|_{a} = 1$ and $p \in \{0,1,2\}$, is sub-exp($\nu,\beta$) with $|\mathbb{E}X| \leq B$.

 \begin{corollary}\label{cor:bounds_str}
Let $m$ be a multiple of $d$. 
\begin{itemize}
\item Consider the same setting as in \Cref{cor:bounds_gaussian}, and assume that the frequencies $\omega_{1}, \dots, \omega_{m}$ are \emph{block-i.i.d.}
 with blocks of size $b=d$ such that the $\omega_{j}$ are identically (but not necessarily independently) distributed Gaussian frequencies $\omega_{j} \sim \mathcal{N}(0,\bm{\Sigma}^{-1}/s^{2})$. Then, for each $0<\tau<1-5c$, where $c := (2k-1)\mu$, we have
\begin{equation}\label{eq:cor_bound_delta_s_k_gaussian_str}
\mathbb{P} \bigg( \delta(\mathcal{S}_{k}|\mathcal{A}) >
\frac{4c+\tau}{1-c}
\bigg) \leq 11 \exp\Big(-\frac{m}{dv}\Big) \big(1+C/\tau)^{3d+2} ,
\end{equation}
where $v, C$ are defined as in \Cref{cor:bounds_gaussian}.

\item Consider the same setting as in \Cref{cor:bounds_dirac}, and assume that the frequencies $\omega_{1}, \dots, \omega_{m}$ are block-i.i.d. with blocks of size $b=d$ such that the $\omega_{j}$ are identically (but not necessarily independently) distributed Gaussian frequencies $\omega_{j} \sim \mathcal{N}(0,s^{-2}\mathbb{I}_{d})$. Then, for each $0<\tau<1-5c$, where $c := (2k-1)\mu$, we have
\begin{equation}\label{eq:bound_cor_str_delta}
\mathbb{P} \bigg( \delta(\mathcal{S}_{k}|\mathcal{A}) >
\frac{4c+\tau}{1-c}
\bigg) \leq (10+d)  \exp\Big(-\frac{m}{dv}\Big) \big(1+C(\tau,m/d)/\tau)^{3d+2} ,
\end{equation}
where $v,C$ are defined as in \Cref{cor:bounds_dirac}.
\end{itemize}

\end{corollary}

The proof of \Cref{cor:bounds_str} is given in \Cref{proof:bounds_str}. While this result does not assume the independence of the random variables $X_{j}$, it comes with a non-negligible cost: compared to~\eqref{eq:cor_bound_delta_s_k_gaussian} and~\eqref{eq:bound_cor_diracs_delta}, the "variance" term \eqref{eq:variance_structured} in~\eqref{eq:cor_bound_delta_s_k_gaussian_str} and~\eqref{eq:bound_cor_str_delta} worsens by a factor $d$ compared to the fully i.i.d. setting. There is also a more benign $\log d$ factor in sketch sizes due to the appearance of the $(10+d)$ factor in~\eqref{eq:bound_cor_str_delta}. Thus, a refined analysis of the expectation and of the sub-exponentiality constants of the random variables $Z_{p}$ and $Z'_{p}$ (cf \eqref{eq:Z_def_structured}-\eqref{eq:Z_def_structured_bis}) is required in order to prove competitive bounds on sketch sizes for structured sketching. 
Ideally, we may hope to prove that $Z_{p}$ and $Z'_{p}$ are sub-exp($\nu,\beta$) with $\nu = \mathcal{O}(1/\sqrt{d})$ and $\beta = \mathcal{O}(1/d)$ for some families of structured random matrices. 
Note however that this would require to slightly sharpen the bound obtained in \Cref{thm:RIP_without_weights_random_bis}. Indeed, the main bottleneck in the sketch size would be the constant $\nu'$ defined in \eqref{eq:variance_structured} that involves the term $B$ which is a constant independent of $d$. Improving \Cref{thm:RIP_without_weights_random_bis} would either require to refine \Cref{lemma:sub_exp_2} (using more taylored assumptions) in order to circumvent the presence of this constant in $\nu'$, or to more directly rely on \Cref{thm:RIP_without_weights} and in establishing autonomous concentration bounds. This is left to future work.

 \subsection{Lower bounds}
\label{sec:lower_bound}
To conclude, we provide several lower bounds that complete the picture established in this section. 

First, we show that condition~\eqref{eq:phi_omega_conditions}, which is known to be sufficient to control the covering numbers appearing in \Cref{thm:old_fluctuation_result}, is indeed close to necessary for these covering numbers to be well-defined and finite. This shows that existing theory (such as \citep[Theorem 5.13]{GrBlKeTr20}) is simply too restrictive to provide guarantees for perhaps the most natural setting where there is ``no'' importance sampling, i.e., $w \equiv 1$, which is in contrast covered by our new results. 

Second, we investigate the gap between sufficient sketch sizes endowed with theoretical guarantees, which scale as $O(k^{2}d)$, and practically observed sketch sizes, which scale as $O(kd)$. We demonstrate that a proof route which could seem natural to bridge this gap is in fact a dead-end, leaving possible improvements to further work.

\subsubsection{Lower bounds on variance terms}\label{sec:lower_bound_variances}
The empirical investigations in~\citep{KeTrTrGr17} showed that a practically sufficient sketch size scales as $\mathcal{O}(dk)$ compared to the theoretically sufficient sketch size $\mathcal{O}(dk^2)$ obtained by the analysis given in \citep{GrBlKeTr20} and the analysis given in this work. This suggest that there is still room for improvement on the theoretical bounds of sketching. We investigate below theoretical approaches that may seem natural ways to improve the proof techniques respectively introduced in \citep{GrBlKeTr20} and in this work. Our main conclusion is that these approaches \emph{cannot} lead to the desired explanation of the empirical findings of~\citep{KeTrTrGr17}.

 \paragraph{Limits of the proof technique of~\citep{GrBlKeTr20}}  After a careful examination of the proof given in \citep{GrBlKeTr20}, 
it may be tempting to improve the concentration inequality~\eqref{eq:punctual_concentration} and target one of the form
\begin{equation*}
 \forall \tau>0,\  \:\: 
\sup_{\nu \in \mathcal{S}_k}  \mathbb{P} \Big( \big|\|\mathcal{A}\nu\|_{2}^2 -1\big| > \frac{\tau}{2} \Big) \leq 2\exp \bigg( -\frac{m}{v_{0}(\tau)} \bigg),
\end{equation*}
with $v_{0}(\tau)$ \emph{independent of $k$} (under appropriate incoherence assumptions on $\kappa$, that depend of $k$). This would indeed easily provide the desired result by combining the technical ingredients as in the proof of~\Cref{thm:old_fluctuation_result}, however under standard assumptions\footnote{A subgaussian tail or a sub-exponential tail.} on the growth of $v_{0}(\tau)$  when $\tau \to \infty$, it is a classical exercice\footnote{See  \citep[Theorem 2.3]{BoLuMa13} and \citep[Lemma 1]{BoKlZh20}.} to show that this implies bounded moments $\mathbb{E}\|\mathcal{A}\nu\|_{2}^{2q}$, for $q \geq 2$, depending only on $v_{0}(\tau)$ and $m$, in particular this would also imply the existence of a constant $C>0$, \emph{independent of} $k$, such that
\begin{equation*}
\forall \nu \in \mathcal{S}_{k},\:\: \mathbb{V}\|\mathcal{A} \nu\|_{2}^2 \leq \frac{C}{m}.
\end{equation*}
where $\mathbb{V}(\cdot)$ denotes the variance of a scalar random variable.
However, as we now show, under typical assumptions on the $2k$-coherence of the kernel, this variance grows linearly  with $k$.

We begin with a technical lemma proved in~\Cref{proof:tildekernel}. 
\begin{lemma}\label{lem:TildeKernel}
Consider a normalized shift-invariant kernel  $\kappa$ and $\pi_{0} \in \mathcal{P}(\mathbb{R}^{d})$. If $w: \mathbb{R}^{d}\to (0,+\infty)$ is $\kappa$-compatible and satisfies
\begin{equation}
\label{eq:FiniteVarianceAssumptionW}
\int
 \big|\langle\phi_{\omega}^{1}, \pi_0\rangle\big|^{4} w^{-2}(\omega)\hat{\kappa}(\omega) \mathrm{d}\omega < +\infty,
 \end{equation} 
then the following shift-invariant kernel is well-defined
 \begin{align}
\label{eq:tilde_kernel}
\kappa_w(\theta,\theta') &:= \int_{\mathbb{R}^d} \big|\langle\phi_{\omega}^{1}, \pi_0\rangle\big|^{4} 
w^{-2}(\omega)
\hat{\kappa}(\omega)  e^{ \imath \omega^{\top} (\theta-\theta')}  \mathrm{d}\omega,\quad \theta,\theta' \in \mathbb{R}^{d}
\end{align}
and satisfies $\kappa_w(0,0) \geq \|\pi_{0}\|_{\kappa}^{4}$. The following weight function
is $\kappa$-compatible and satisfies~\eqref{eq:FiniteVarianceAssumptionW} 
\begin{equation}\label{eq:DefBasicWeighting}
w_{0}(\omega) := \|\pi_{0}\|_{\kappa}^{-1} \cdot |\langle \phi_{\omega}^{1},\pi_{0}\rangle|.
\end{equation}
Moreover $\kappa_{w_{0}}(0,0) = \|\pi_{0}\|_{\kappa}^{4}$ and more generally
\begin{equation*}
\kappa_{w_{0}}(\theta,\theta') = \|\pi_{0}\|_{\kappa}^{2}\langle \pi_{\theta},\pi_{\theta'}\rangle_{\kappa}, \quad \theta,\theta' \in \mathbb{R}^{d}.
\end{equation*}
\end{lemma}
In the special case of a Dirac base distribution $\pi_{0}$,~\eqref{eq:DefBasicWeighting} simply defines a ``flat'' weight function $w_{0} \equiv 1$ since $\|\pi_{0}\|_{\kappa}^{2} = \langle \pi_{0},\pi_{0}\rangle_{\kappa} = \kappa(0,0) = 1 = |\langle \phi_{\omega}^{1},\pi_{0}\rangle|$.

\begin{theorem}\label{thm:initial_lower_bound}
Consider a normalized shift-invariant kernel  $\kappa$.
Consider a location-based family $\mathcal{T}$ with base distribution $\pi_{0}$, an integer $k \geq 1$, and the 
separated $k$-mixture model $\mathfrak{G}_{k}$ from~\eqref{eq:DefMixtureModel} where $\varrho(\theta,\theta') := \|\theta-\theta'\|$ for some arbitrary norm $\|\cdot\|$ on $\mathbb{R}^{d}$. Consider a vector $\theta^{*} \in \mathbb{R}^{d}$ such that $\|\theta^{*}\| \geq 1$, and observe that the following two $k$-mixtures are $1$-separated with respect to $\varrho$, i.e., $\nu_{k,1},\nu_{k,2} \in \mathfrak{G}_{k}$, so that $\nu_{k}:= (\nu_{k,1}- \nu_{k,2})/\|\nu_{k,1}- \nu_{k,2}\|_{\kappa} \in \mathcal{S}_{k}$:  
\begin{equation*}
\nu_{k,1} = \sum\limits_{i =1}^{k}\frac{1}{k}\pi_{(2i-2)\theta^{*}}, \:\: \nu_{k,2} = \sum\limits_{i =1}^{k}\frac{1}{k}\pi_{(2i-1)\theta^{*}}.
\end{equation*}

\begin{enumerate}
\item If the $2k$-coherence of $\kappa$ with respect to $\mathcal{T}$ is bounded by $0 \leq c<1$ then for any $\kappa$-compatible weight function $w$ that satisfies~\eqref{eq:FiniteVarianceAssumptionW}, 
the  $w$-RFF sketching operator $\mathcal{A}$ with $m$ i.i.d frequencies $\omega_{j} \sim \Lambda := w^{2} \hat{\kappa}$ satisfies
\begin{align}
\label{eq:MainThm2Var}
\mathbb{V}\| \mathcal{A} \nu_k \|^{2} 
&\geq
\frac{1}{m} \Big(
C_{w} k-1
\Big),\\ 
\label{eq:TildeKernelCoherence}
\text{where} \qquad
C_{w}  &:= \frac{\kappa_w(0,0)}{\|\pi_{0}\|_{\kappa}^{4}} \cdot \frac{4/3-2c_{w}}{(1+c)^{2}}\, \qquad
\text{with}\ 
c_{w} := 2k  \sup_{\theta,\theta': \|\theta-\theta'\| \geq 1} \frac{|\kappa_{w}(\theta,\theta')|}{\kappa_w(0,0)}.
\end{align}

\item   If the mutual coherence of $\kappa$ is bounded by $c/2k$ we have $c_{w_{0}} \leq c$ with $w_{0}$ as in~\eqref{eq:DefBasicWeighting}.  
\end{enumerate}
\end{theorem}
The proof of~\Cref{thm:initial_lower_bound} is given in~\Cref{proof:initial_lower_bound}. For a kernel with mutual coherence bounded by $c/2k$ with $c \leq 1/2$, we obtain  $\frac{4/3-2c_{w_{0}}}{(1+c)^{2}} \geq (4/3-1)/(3/2)^{2} = 4/27 \geq 1/7$. Since $\kappa_{w_{0}}(0,0)/\|\pi_{0}\|_{\kappa}^{4} = 1$ (by Lemma~\ref{lem:TildeKernel}) we get $C_{w_{0}} \geq 1/7$.
Finally, since the mutual coherence of $\kappa$ is bounded by $c/2k$, its $2k$-coherence is bounded by $c$ (cf~\eqref{eq:LCoherenceFromMutualCoherence}), and~\eqref{eq:MainThm2Var} implies
that the $w_{0}$-RFF sketching operator $\mathcal{A}$ with $m$ i.i.d frequencies $\omega_{j} \sim \Lambda_{0} := w_{0}^{2} \hat{\kappa}$ satisfies
\begin{equation}\label{eq:MainThm2VarDirac}
\mathbb{V}\| \mathcal{A} \nu_k \|^{2} 
\geq\frac{k/7-1}{m}.
\end{equation}

In the special case of a mixture of Diracs, since $w_{0} \equiv 1$, it is not difficult to check that $\kappa_{w_{0}} = \kappa$. For a non-negative kernel $\kappa \geq 0$, the fact that $\kappa_{w_{0}} \geq 0$ allows to improve an intermediate bound (in Equation~\eqref{eq:KernelWLowerBoundDetail} in the proof), leading to the same result where
 $4/3-2c_{w_{0}}$ is replaced with $4/3-c_{w_{0}}$ in the definition of $C_{w_{0}}$. This shows that~\eqref{eq:MainThm2VarDirac} is then valid even with a mutual coherence bounded by $c/2k$ with $c < 1$.

Under the assumptions of \Cref{thm:initial_lower_bound}, with $w = w_{0}$, we have $\mathbb{V}\|\mathcal{A} \nu_{k}\|^2 = \Omega(k/m)$. This implies that with this weight function, and even with the usual incoherence assumption, the term $v(k,\tau)$ in \eqref{eq:punctual_concentration} \emph{cannot} be bounded from above by a universal constant that is independent of $k$. 
This result highlights the difference between classical compressed sensing and sketching. Indeed, if we consider a random Gaussian matrix $A \in \mathbb{R}^{ m\times d}$ with i.i.d. entries $\mathcal{N}(0,1/m)$, it is well known that for every normalized vector $x \in \mathbb{R}^{d}$ (such that $\|x\|_{2}=1$) the variance of $\|A x\|_{2}^{2}$ \emph{does not depend on the sparsity} $k$ of the vector $x$; \citep[see e.g.][Lemma 9.8]{FoRa13}. 

\Cref{fig:rip_classical_vs_sketching} illustrates this claim: we compare the variance of $\|A x_{k}\|_{2}^{2}$ where $x_{k} \in \mathbb{R}^{d}$ is a normalized vector of sparsity $k$
to the variance of $\|\mathcal{A} \nu_{k}\|_{2}^{2}$ where $\mathcal{A}$ and $\nu_{k}$ are defined  in~\Cref{thm:initial_lower_bound}, with $\pi_{0}$ the Dirac distribution, $w =w_{0} \equiv 1$, and the $2k$-coherence of $\kappa$ is smaller than $1/2$.
We observe that the variance of $\|A x_{k}\|_{2}^{2}$ is practically flat as a function of $k$, while the variance of $\|\mathcal{A}\nu_{k}\|_{2}^{2}$ is linear in $k$. This observation shows that the study of the RIP in the set of mixtures of Diracs $\mathfrak{G}_{k}$ is not a mere extension of the existing RIP literature in Euclidean spaces.
\begin{figure}[h]
\centering
\includegraphics[width= 0.47\textwidth]{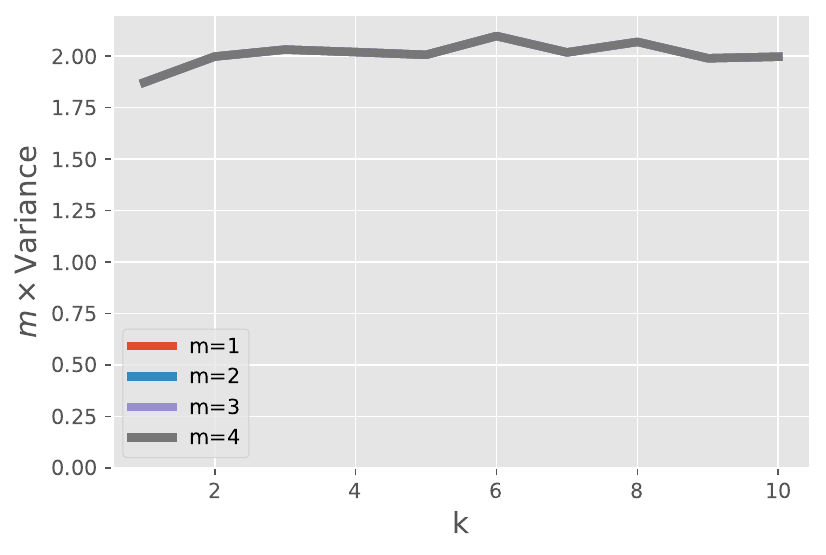}~\includegraphics[width= 0.47\textwidth]{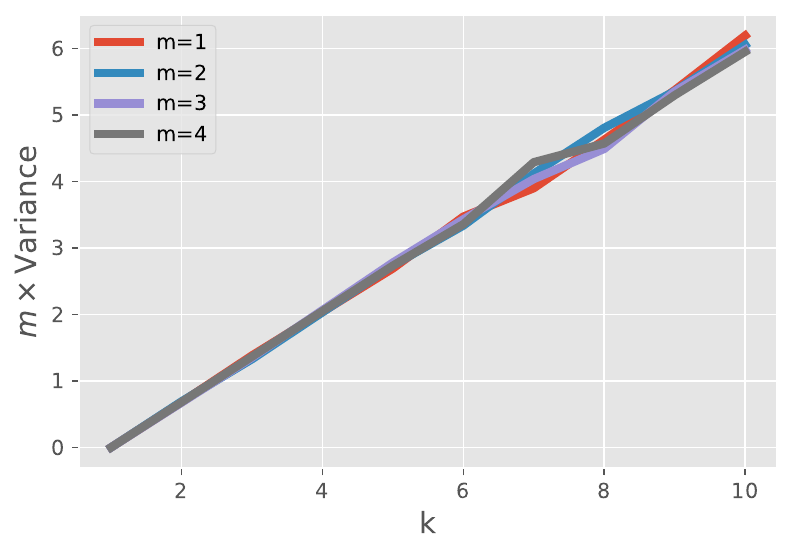}\\
\caption{The term $m \times \mathbb{V}\|Ax_{k}\|^2$ (left) compared to the term $m \times \mathbb{V}\|\mathcal{A}\nu_{k}\|^2$ (right).
}
\label{fig:rip_classical_vs_sketching}
\end{figure}

\paragraph{Limits of the proof technique proposed in this paper.}
Now, a careful examination of the analysis leading to \Cref{cor:bounds_gaussian} and \Cref{cor:bounds_dirac} suggests that the unwanted $O(k^{2}d)$ instead of $O(kd)$ behaviour of the sufficient sketch size results from the requirement of concentration inequality~\eqref{eq:assumption_psi_ell_Omega} in Theorem~\ref{thm:RIP_without_weights}. 
In particular, in order to improve the theoretical guarantees of sketching using i.i.d. random frequencies, it would be tempting to seek a weight function $w$ such that the random variable $$\psi_{\ell}(z|\omega):= \psi(\omega) f_{\ell}(z|\omega)$$ is sub-exp($\nu_\ell,b_{\ell}$)  with $\nu_{\ell} = \mathcal{O}(1/\sqrt{k})$, for each $\ell \in \{\mathrm{mm,md,dd}\}$ and $z \in \Theta_{\ell}$.

The reader can check that such an approach would indeed allow to establish guarantees with a sketch size $\mathcal{O}(kd)$. 
However, this would also imply that for $\ell \in \{\mathrm{mm,md,dd}\}$ and $ z \in \Theta_{\ell}$, the variance of $\psi_{\ell}(z|\omega)$ would satisfy $\mathbb{V} \psi_{\ell}(z|\omega) = \mathcal{O}(1/k)$.
Now, when $\kappa$ has mutual coherence bounded by $\mu<1/(2k-1)$ (this is a natural assumption in our context), the expectation of this variable satisfies $ \mathbb{E} 
\psi_{\ell}(z|\omega) = \mathcal{O}(1/k)$, hence we would obtain that 
\begin{equation}\label{eq:wishful_bound}
\mathbb{E}[\psi_{\ell}^{2}(z|\omega)] = \mathcal{O}(1/k).
\end{equation}
The following result shows that~\eqref{eq:wishful_bound} \emph{cannot hold} in the specific setting of mixture of Diracs.

\begin{proposition}\label{prop:lower_bound_psi_mm}
Consider $\mathcal{T}$ to be a location-based family  with the dirac in $0$ as a base distribution, and consider $\kappa$ to be a normalized shift-invariant kernel such that $\kappa \geq 0$. Consider $\phi_{\omega}$ as defined in~\eqref{eq:FFComponent}, then for any $\kappa$-compatible weight function $w$ and for $\omega \sim \Lambda:= w^2 \hat{\kappa}$, we have

\begin{equation*}\label{eq:lower_bound_exp_epsimm}
\forall y \in \Theta_{\mathrm{mm}}, \:\:\:\: \mathbb{E}[\psi^{2}_{\mathrm{mm}}(y|\omega)] \geq \frac{1}{4}.
\end{equation*}

\end{proposition}
This lower bound holds \emph{irrespective of how small the mutual coherence of $\kappa$ may be.}
\begin{proof}
Let $y \in \Theta_{\mathrm{mm}}$. Since $\|\pi_{0}\|_{\kappa}^{2} = \kappa(0,0) = 1$ and $|\langle \pi_{0},\phi_{\omega}\rangle| = 1/w(\omega)$, by~\eqref{eq:DefRho}-\eqref{eq:DefPsiMM} we have 
\begin{equation*}
\psi_{\mathrm{mm}}(y|\omega) := \psi(\omega)f_{\mathrm{mm}}(y|\omega) = \frac{|\langle \pi_0, \phi_{\omega}\rangle|^2}{\|\pi_0\|_{\kappa}^2} \cos(\langle \omega,y\rangle) = \frac{\cos(\langle \omega,y\rangle)}{w^2(\omega)}.
\end{equation*}
As $w$ is $\kappa$-compatible, we have $\int_{\mathbb{R}^d} w^{2}(\omega)\hat{\kappa}(\omega) \mathrm{d} \omega  = 1$, thus by Cauchy-Schwarz inequality we get
\begin{align*}
\mathbb{E}_{\omega \sim \Lambda}[\psi^{2}_{\mathrm{mm}}(y|\omega)]
= \int_{\mathbb{R}^{d}} \frac{ \cos^{2}(\langle\omega, y \rangle)}{w^{4}(\omega)} w^{2}(\omega)  \hat{\kappa}(\omega) \mathrm{d}\omega  
& = \int_{\mathbb{R}^{d}} w^{2}(\omega)  \hat{\kappa}(\omega) \mathrm{d}\omega \int_{\mathbb{R}^{d}} \frac{ \cos^{2}(\langle\omega, y \rangle)}{w^{2}(\omega)}  \hat{\kappa}(\omega) \mathrm{d}\omega \\
&  \geq \Big(\int_{\mathbb{R}^{d}} |\cos(\langle \omega,y\rangle)|  \hat{\kappa}(\omega) \mathrm{d}\omega\Big)^{2}  \\
&  \geq \Big(\int_{\mathbb{R}^{d}} \cos^{2}(\langle \omega,y\rangle)  \hat{\kappa}(\omega) \mathrm{d}\omega\Big)^{2}. 
\end{align*}
Finally, observe that $\cos(\langle \omega,y\rangle)^2 = (1+\cos(2\langle \omega,y\rangle))/2$, so that we have, using that $\kappa \geq 0$,
\begin{equation*}
\mathbb{E}_{\omega \sim \Lambda}[\psi^{2}_{\mathrm{mm}}(y|\omega)] \geq 
\Big [\frac{1}{2}\bigg( \int_{\mathbb{R}^{d}}   \hat{\kappa}(\omega) \mathrm{d}\omega  +  \int_{\mathbb{R}^{d}} \cos(2\langle \omega,y\rangle)  \hat{\kappa}(\omega) \mathrm{d}\omega \bigg)\Big]^{2}  =
\Big(\frac{1}{2}(1+\kappa(2y,0))\Big)^{2} \geq \frac{1}{4}.
\end{equation*}
\end{proof}

\begin{figure}[h]
\centering
\includegraphics[width= 0.45\textwidth]{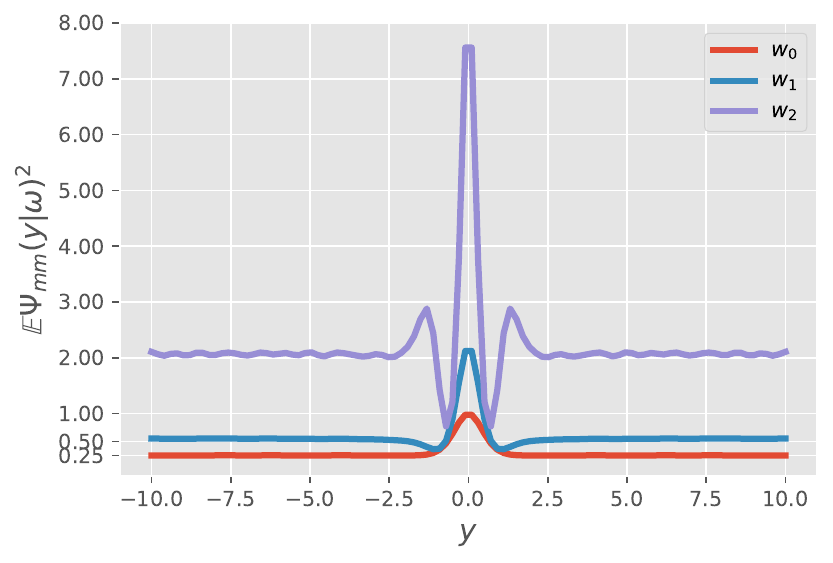}\includegraphics[width= 0.45\textwidth]{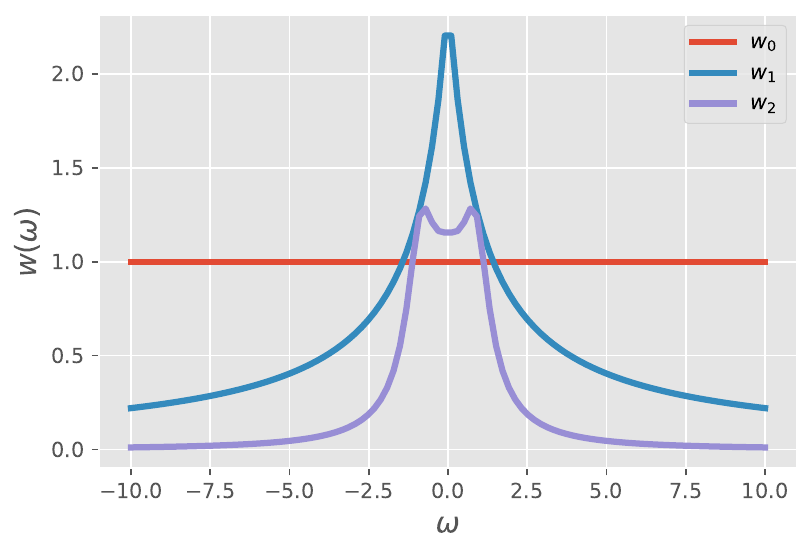}
\caption{An illustration of the lower bound of Proposition~\ref{prop:lower_bound_psi_mm} (left) for three choices of $w$ (right): $w_0(\omega) = 1, \:\: w_{1}(\omega) = (1+\|\omega\|)^{-1}, \:\: w_{2}(\omega) = (\|\omega\|^4+1)(\|\omega\|^6+1)^{-1}$. }
\end{figure}

Proposition~\ref{prop:lower_bound_psi_mm} shows that improving the dependency of the sketch size on $m$ cannot simply rely on improved concentration bounds:  obtaining sharper bounds on sketch sizes that reflect the empirical findings of~\citep{KeTrTrGr17} seems to require a substantially subtler analysis which is beyond the scope of this paper.

\section{Conclusion}\label{sec:conclusion}

In this work we revisited the theoretical analysis of the Restricted Isometry Property for sketching operators proposed  in \citep{GrBlKeTr21,GrBlKeTr20}. This property is crucial in the field of compressive learning: it measures how the sketching operator preserves the MMD distance between measures belonging to a model set of measures. In particular, the sketching operators proposed in \citep{GrBlKeTr20} are suited for models of mixtures and are based on Fourier features. Nevertheless, the proposed theoretical analysis makes some additional assumptions that are summarized by the conditions~\eqref{eq:phi_omega_conditions}.
After investigating the partial necessity of the conditions~\eqref{eq:phi_omega_conditions} in the analysis of~\citep{GrBlKeTr20}, we proposed an alternative analysis based on deterministic bounds of $\delta(\mathcal{S}_k|\mathcal{A})$, then we showed how to leverage these deterministic bounds to establish the Restricted Isometry Property for stochastic sketching operators restricted to sets of mixtures based on location based measures. In particular, we showed that our revisited analysis allows to deal with realistic settings not covered by \citep{GrBlKeTr20}. 

Beyond these contributions, this work opens the door to  further developments on the theoretical study of sketching operators used in the context of compressive learning. For instance, in the context of structured sketching introduced in~\citep{ChGrKe18}, the frequencies are rather block-i.i.d. samples but not i.i.d. samples (cf the end of Section~\ref{sec:RFF}). Theorem~\ref{thm:RIP_without_weights} remains valid in this context: indeed this result can be used without assuming that the frequencies $\omega_{1}, \dots, \omega_{m}$ are i.i.d.. The main hindrance remaining on this direction is to check that the punctual concentration  expressed by conditions~\eqref{eq:assumption_psi_m_Omega},~\eqref{eq:assumption_psi_d_Omega},~\eqref{eq:assumption_psi_ell_Omega}  hold even when using block-i.i.d. frequencies. For this purpose, the existing results on the literature may help \citep{LeSaSm13,ChSi16,ChRoWe17,MuKaBuOs18}.
Another setting where our results may be useful is the study of deterministic sketching operators. Indeed, as shown in Section~\ref{sec:main_results}, the core of our analysis is based on deterministic bounds of $\delta(\mathcal{S}_k|\mathcal{A})$ presented in Section~\ref{sec:sharp_determistic_bounds}, and recent years have witnessed an increased interest into the theoretical study of \emph{deterministic Fourier feature maps}~\citep{DaDeRe17,YaSiAvMa14}. Investigating whether deterministic sketching operators still satisfy the same guarantees as the stochastic ones is thus both a natural and challenging question. 

As shown in Section~\ref{sec:lower_bound}, neither the analysis of \citep{GrBlKeTr20} nor our analysis unfortunately achieves to explain the empirical findings of~\citep{KeTrTrGr17}, and there remains a gap between sufficient sketch sizes endowed with theoretical guarantees, which scale as $\mathcal{O}(k^2d)$, and practically observed sketch sizes, which scale as $\mathcal{O}(kd)$. On the one hand, the quadratic theoretical dependency on the 'sparsity' $k$ is not surprising given the known limits of sparse recovery guarantees exploiting dictionary coherence \citep[Chapter 5]{FoRa13}. Yet, the literature on compressive sensing manages to establish bounds essentially linear in the sparsity using random matrix techniques that do rely on mutual coherence \citep[Chapter 9]{FoRa13}. The proofs in this field exploit a fine study of the eigenvalues of random matrices, which was until now somehow overlooked in the community of compressive learning. Thus, an interesting direction of research is the study of the eigenvalues of the random matrices that appear in this context. Recent developments on the study of ridge kernel regression for random Fourier features may help \citep{AvKaMuVeZa17,LiToOgSe19}.
In particular, in this line of research, the authors investigated the impact of the frequency distribution in the quality of the approximations based on random Fourier features. The techniques developed in these works may be helpful to understand the impact of the frequency distribution in the design of sketching operators. In the same vein, alternative frequency distributions, that define other kernels than the Gaussian kernel, have manifested better empirical performance when used in sketching-based learning tasks such as mixture learning; see Section 4.2 in \citep{Cha20} for an example. This motivates to scrutinize the impact of the kernel on the design of the sketching operator.

\paragraph{Acknowledgement}

This project was supported by the AllegroAssai ANR project ANR-19-CHIA-0009. The authors
would like to thank Titouan Vayer for his constructive feedback on an early version of this work and for sketching an early version of \Cref{lemma:very_useful_bound_on_exp}. Rémi Gribonval would like to thank Felix Krahmer for interesting discussions at the Oberwolfach Workshop 2148.

 \newpage
\appendix
\section{Proofs}

\subsection{Proof of~\Cref{lem:TildeKernel}}\label{proof:tildekernel}

By~\eqref{eq:FiniteVarianceAssumptionW}, the integral in~\eqref{eq:tilde_kernel} converges hence 
the kernel $\kappa_{w}$ is well-defined and shift-invariant.
Consider arbitrary $\theta,\theta' \in \mathbb{R}^{d}$ and denote $\pi_{\theta},\pi_{\theta'}$ deduced from $\pi_{0}$ as in a location-based family. Recall that by definition, $\pi_{\theta}$ is the distribution of $X+\theta$ where $X \sim \pi_{0}$, and $\pi_{\theta'}$ is the distribution of $X+\theta'$. By~\eqref{eq:kappa_generalized_phi}, standard computations with the MMD\footnote{See \citep[Section 2.1]{MuFuSrSc17} and  \citep[proof of Proposition 6.2]{GrBlKeTr20}.} 
yield 
\begin{equation}\label{eq:MMDInnerProdLocBased}
 \langle \pi_{\theta},\pi_{\theta'}\rangle_\kappa = \int_{\mathbb{R}^d} \big|\langle \phi^{1}_{\omega}, \pi_0\rangle\big|^{2} \hat{\kappa}(\omega)  e^{ \imath \omega^{\top} (\theta-\theta')}  \mathrm{d}\omega.
\end{equation}
Specializing~\eqref{eq:MMDInnerProdLocBased} to $\theta=\theta'=0$  we get
by Cauchy-Schwarz' inequality,  since $w$ is $\kappa$-compatible
\begin{align*}
\|\pi_{0}\|_{\kappa}^{4} 
&= 
\left(\int_{\mathbb{R}^d} \big|\langle \phi^{1}_{\omega}, \pi_0\rangle\big|^{2} \hat{\kappa}(\omega)  \mathrm{d}\omega\right)^{2}
= 
\left(\int_{\mathbb{R}^d} \big|\langle \phi^{1}_{\omega}, \pi_0\rangle\big|^{2} w^{-1}(\omega) \sqrt{\hat{\kappa}(\omega)}
w(\omega) \sqrt{\hat{\kappa}(\omega)}
  \mathrm{d}\omega\right)^{2}\\
  & \leq 
\left(\int_{\mathbb{R}^d} \big|\langle \phi^{1}_{\omega}, \pi_0\rangle\big|^{4} w^{-2}(\omega) \hat{\kappa}(\omega)\mathrm{d}\omega\right)^{2} \cdot \left(\int_{\mathbb{R}^d}w^{2}(\omega) \hat{\kappa}(\omega)  \mathrm{d}\omega\right)^{2}
\stackrel{\eqref{eq:tilde_kernel} \& \eqref{eq:DefGammaNormalized}}{:=} \kappa_{w}(0,0)\, .
\end{align*}
The equality case of Cauchy-Schwarz is  when $w(\omega) \propto \big|\langle \phi^{1}_{\omega}, \pi_0\rangle\big|^{2} w^{-1}(\omega)$, i.e., $w(\omega) \propto w_{0}(\omega)$ with $w_{0}$ defined in~\eqref{eq:DefBasicWeighting}. We have
\[
\int w_{0}^{2}(\omega) \hat{\kappa}(\omega)\mathrm{d}\omega
\stackrel{\eqref{eq:DefBasicWeighting}}{=}
 \|\pi_{0}\|_{\kappa}^{-2} \cdot
\int |\langle \phi^{1}_{\omega},\pi_{0}\rangle|^{2} \hat{\kappa}(\omega)\mathrm{d}\omega
\stackrel{\eqref{eq:MMDInnerProdLocBased}}{=}
 \|\pi_{0}\|_{\kappa}^{-2} \cdot \langle \pi_{0},\pi_{0}\rangle_{\kappa}
=1\, ,
\]
hence $w_{0}$ is the only equality-case of Cauchy-Schwarz which is $\kappa$-compatible. The fact that $w_{0}(\omega)$ satisfies~\eqref{eq:FiniteVarianceAssumptionW} follows from $w_{0} \propto  \big|\langle \phi^{1}_{\omega}, \pi_0\rangle\big|^{2} w_{0}^{-1}(\omega)$. Finally we write
\begin{align*}
\kappa_{w_{0}}(\theta,\theta') 
& \stackrel{\eqref{eq:tilde_kernel}}{:=} 
\int_{\mathbb{R}^d} \big|\langle \phi^{1}_{\omega}, \pi_0\rangle\big|^{4} w_{0}^{-2}(\omega) \hat{\kappa}(\omega)  e^{ \imath \omega^{\top} (\theta-\theta')}  \mathrm{d}\omega\\ 
&\stackrel{\eqref{eq:DefBasicWeighting}}{=}
\|\pi_{0}\|_{\kappa}^{2} \cdot
\int_{\mathbb{R}^d} \big|\langle \phi^{1}_{\omega}, \pi_0\rangle\big|^{2} \hat{\kappa}(\omega)  e^{\imath \omega^{\top} (\theta-\theta')}  \mathrm{d}\omega
 \stackrel{\eqref{eq:MMDInnerProdLocBased}}{=} \|\pi_{0}\|_{\kappa}^{2} \cdot \langle \pi_{\theta},\pi_{\theta'}\rangle_\kappa.
\end{align*}

\subsection{Proof of~\Cref{thm:initial_lower_bound}}\label{proof:initial_lower_bound}

The proof relies on the following result which gives a closed formula of the variance of interest.
\begin{proposition}\label{prop:inside_proof_theorem_lower_bound}
Consider a normalized shift-invariant kernel  $\kappa$, a $\kappa$-compatible positive weight function $w$ satisfying~\eqref{eq:FiniteVarianceAssumptionW}, $\Lambda = w^{2}\hat{\kappa}$ and $\mathcal{A}$ a $\kappa$-compatible random $w$-FF sketching operator as in~\Cref{ex:KCompatibleRFF} with $m$ i.i.d. frequencies. Consider a location-based family $\mathcal{T}$ with base distribution $\pi_{0}$.
Given any location parameters $\theta_1, \dots, \theta_{2k} \in \Theta$ and weights $u_1, \dots, u_{2k} \in \mathbb{R}$, we have
\begin{equation}\label{eq:MainTensorProductLemmaLowerBound}
\mathbb{V} \| \mathcal{A} \sum_{i \in [2k]} u_i \pi_{\theta_{i}}\|^{2}  = \frac{1}{m} \Big( 
( \bm{u}^{\otimes 2} )^{\top}
\bm{K}_w(\theta) \bm{u}^{\otimes 2} - \big\|\sum\limits_{i=1}^{2k}u_{i}\pi_{\theta_i}\big\|_{\kappa}^4 \Big),
\end{equation}
where $ \bm{u}^{\otimes 2}  \in \mathbb{R}^{(2k)^2}$ is the tensor product with itself of the vector $\bm{u} \in \mathbb{R}^{2k}$ containing the $u_i$, and
\begin{equation}\label{eq:DefTensorProductKernel}
\bm{K}_w(\theta) = [\kappa_{w}(\theta_{i_1}-\theta_{i_2},\theta_{i_4}-\theta_{i_3} )]_{(i_1,i_2),(i_3,i_4) \in [2k]^2} \in 
\mathbb{R}^{(2k)^{2} \times (2k)^{2}}
\end{equation}
with $[2k] :=\{1,2,\dots, 2k\}$ and $\kappa_{w}$ is the shift-invariant kernel defined in~\eqref{eq:tilde_kernel}.
\end{proposition}

\begin{proof}[Proof of~\Cref{prop:inside_proof_theorem_lower_bound}]
As $\| \mathcal{A} \sum_{i=1}^{2k} u_i \pi_{\theta_i} \|^{2} = \frac{1}{m}\sum_{j \in [m]} \big|\langle \phi_{\omega_j}^{w}, \sum_{i=1}^{2k} u_i \pi_{\theta_i} \rangle\big|^2$
is the average of $m$ i.i.d. variables, we have
$\mathbb{V} \| \mathcal{A} \sum_{i=1}^{2k} u_i \pi_{\theta_i} \|^{2}  
=\frac{1}{m} \mathbb{V} \big|\langle \phi_{\omega}^{w}, \sum_{i=1}^{2k} u_i \pi_{\theta_i} \rangle\big|^2$
so we now characterize 
\begin{align*}
\mathbb{V}|\langle \phi_{\omega}^{w}, \sum_{i=1}^{2k} u_i \pi_{\theta_i} \rangle|^2 & = \mathbb{E}_{\omega \sim \Lambda}\big|\langle\phi_{\omega}^{w}, \sum_{i=1}^{2k} u_i \pi_{\theta_i} \rangle\big|^4 - \Big(\mathbb{E}_{\omega \sim \Lambda}\big|\langle\phi_{\omega}^{w}, \sum_{i=1}^{2k} u_i \pi_{\theta_i} \rangle\big|^2\Big)^2 \nonumber\\
&
 = \mathbb{E}_{\omega \sim \Lambda}\big|\langle\phi_{\omega}^{w}, \sum_{i=1}^{2k} u_i \pi_{\theta_i} \rangle\big|^4 - \big\| \sum_{i=1}^{2k} u_i \pi_{\theta_i} \big\|_{\kappa}^4
\end{align*}
where we used~\eqref{eq:UnbiasedEmpiricalMMD2} since $w$ is $\kappa$-compatible and $\Lambda = w^{2}\hat{\kappa}$ (see Example~\ref{ex:KCompatibleRFF}). Given the expression~\eqref{eq:FFComponent} of $\phi_{\omega} := \phi_{\omega}^{w} =\phi^{1}_{\omega}/w(\omega)$ and since $\{\pi_{\theta}\}_{\theta \in \Theta}$ is location-based, 
\[
\langle \phi_{\omega}, \pi_{\theta}\rangle 
=
\mathbb{E}_{X \sim \pi_{\theta}} \phi_{\omega}(X)
=
\mathbb{E}_{X' \sim \pi_{0}} \phi_{\omega}(X'+\theta)
=
e^{\imath\omega^{\top}\theta} 
\mathbb{E}_{X' \sim \pi_{0}} \phi_{\omega}(X')
= e^{\imath\omega^{\top}\theta} \langle \phi_{\omega}, \pi_{0} \rangle,
\]
so that we can develop \begin{align*}
 \big|\langle \phi_{\omega}, \sum_{i=1}^{2k} u_i \pi_{\theta_i} \rangle\big|^4 
 & = \sum_{i_1 =1}^{2k} \sum_{i_2 = 1}^{2k} \sum_{i_3 =1}^{2k} \sum_{i_4 = 1}^{2k}  u_{i_1}u_{i_2} u_{i_3}u_{i_4} \langle \phi_{\omega},  \pi_{\theta_{i_1}} \rangle \overline{\langle \phi_{\omega},  \pi_{\theta_{i_2}} \rangle}  \langle \phi_{\omega},  \pi_{\theta_{i_3}} \rangle \overline{\langle \phi_{\omega},  \pi_{\theta_{i_4}} \rangle} \nonumber\\
& = \sum_{i_1 =1}^{2k} \sum_{i_2 = 1}^{2k} \sum_{i_3 =1}^{2k} \sum_{i_4 = 1}^{2k}  u_{i_1}u_{i_2} u_{i_3}u_{i_4} |\langle \phi_{\omega},  \pi_{0} \rangle |^4  e^{\imath \big(\omega^{\top} (\theta_{i_1} - \theta_{i_2} + \theta_{i_3} - \theta_{i_4}) \big) }.
\end{align*}
Moreover, given the expression~\eqref{eq:ImportanceSampling} of the pdf $\Lambda(\omega)$ we have 
for every $i_1,i_2,i_3,i_4 \in [2k]$
\begin{align*}
\mathbb{E}_{\omega \sim \Lambda}|\langle \phi_{\omega},  \pi_{0} \rangle |^4  e^{\imath \big(\omega (\theta_{i_1} - \theta_{i_2} + \theta_{i_3} - \theta_{i_4}) \big) } 
& = \int_{\mathbb{R}^d}  |\langle \phi_{\omega},  \pi_{0} \rangle |^4 e^{\imath \big(\omega^{\top} (\theta_{i_1} - \theta_{i_2} + \theta_{i_3} - \theta_{i_4})\big)} w^{2}(\omega) \hat{\kappa}(\omega) d\omega \nonumber \\
& = \int_{\mathbb{R}^d}  |\langle \phi_{\omega}^{1},  \pi_{0} \rangle |^4w^{-4}(\omega) e^{\imath \big(\omega^{\top} (\theta_{i_1} - \theta_{i_2} + \theta_{i_3} - \theta_{i_4})\big)} w^{2}(\omega) \hat{\kappa}(\omega) d\omega \nonumber \\
& = \kappa_{w}(\theta_{i_1}-\theta_{i_2}, \theta_{i_4}-\theta_{i_3}),
\end{align*} 
where $\kappa_{w}$ is defined in \eqref{eq:tilde_kernel}. As a result, we have
\begin{equation*}\mathbb{E} \big|\langle \phi_{\omega}, \sum_{i=1}^{2k} u_i \pi_{\theta_i} \rangle\big|^4 
= \sum_{i_1 =1}^{2k} \sum_{i_2 = 1}^{2k} \sum_{i_3 =1}^{2k} \sum_{i_4 = 1}^{2k}  u_{i_1}u_{i_2} u_{i_3}u_{i_4} \kappa_{w}(\theta_{i_1}-\theta_{i_2}, \theta_{i_4}-\theta_{i_3})
= (\bm{u}^{\otimes 2})^\top\bm{K}_w(\theta) \bm{u}^{\otimes 2}
\end{equation*}
where, according to the notations of the proposition $\bm{u}^{\otimes 2} \in \mathbb{R}^{(2k)^{2}}$ is the vector with entries
$u^{\otimes 2}_{(i_1,i_2)} = u_{i_1}u_{i_2}$ for each pair $(i_{1},i_{2}) \in [2k] \times [2k]$, and $\bm{K}_w(\theta)$ is the square matrix of size $(2k)^{2} \times (2k)^{2}$ defined in~\eqref{eq:DefTensorProductKernel}.
Putting the pieces together yields~\eqref{eq:MainTensorProductLemmaLowerBound} as claimed.
\end{proof}

By Proposition~\ref{prop:inside_proof_theorem_lower_bound}
with $\nu_{k} := \mu_{k,1} - \mu_{k,2} = \sum\limits_{i =1}^{2k}u_{i}\pi_{\theta_i}$, where $u_i := \tfrac{(-1)^{i-1}}{k}$, $\theta_{i} := (i-1)\theta^{*}$, \begin{equation}\label{eq:VarThm2}
\mathbb{V}\| \mathcal{A}\nu_{k} \|^{2} = \frac{\mathbb{V}\| \mathcal{A}(\mu_{k,1}-\mu_{k,2}) \|^{2}}{\|\mu_{k,1}-\mu_{k,2}\|_{\kappa}^4} = \frac{1}{m}\Bigg(\frac{(\bm{u}^{\otimes 2})^{\top}\bm{K}_w(\theta) \bm{u}^{\otimes 2}}{\|\mu_{k,1}-\mu_{k,2}\|_{\kappa}^4} - 1 \Bigg).
\end{equation}
We now bound $\|\mu_{k,1}-\mu_{k,2}\|_{\kappa}^{4}$ and $(\bm{u}^{\otimes 2})^{\top}\bm{K}_w(\theta) \bm{u}^{\otimes 2}$ to get the first part of the Theorem.

{\bf Bounding $\|\mu_{k,1}-\mu_{k,2}\|_{\kappa}^4$.} 
Since $\kappa$ is shift-invariant and $\mathcal{T}$ is a location-based family we have $\|\pi_{\theta_i}\|_{\kappa}^2 = \|\pi_{0}\|_{\kappa}^{2}$ for each $i \in [2k]$. Since $\varrho(\theta_{i},\theta_{j}) = \|\theta_{i}-\theta_{j}\| = |i-j| \cdot \|\theta^{*}\| \geq 1$ for $1 \leq i \neq j \leq 2k$, the $2k$ (unnormalized) monopoles $\{u_{i}\pi_{\theta_i}\}_{i=1}^{2k}$ are pairwise $1$-separated dipoles with respect to $\rho$. As $\kappa$ has its $2k$-coherence with respect to $\mathcal{T}$ bounded by $c$ (cf Definition~\ref{eq:DefCoherence}-\eqref{eq:almost_pythagorian_formal}), it follows that
\begin{equation*}
\frac{2}{k}(1-c) \|\pi_{0}\|_{\kappa}^{2}
=
(1-c)
\sum\limits_{i =1}^{2k}\|u_{i}\pi_{\theta_i}\|_{\kappa}^2  \leq \|\mu_{1,k} - \mu_{2,k}\|_{\kappa}^2 
\leq 
(1+c)
\sum\limits_{i =1}^{2k}\|u_{i}\pi_{\theta_i}\|_{\kappa}^2 =\frac{2}{k}(1+c) \|\pi_{0}\|_{\kappa}^{2},
\end{equation*}
where we used that $u_{i}^{2} = 1/k^{2}$ for every $i$. Therefore
\begin{equation}\label{eq:bound_MMD_worst_example}
 \frac{4\|\pi_{0}\|_{\kappa}^{4}(1-c)^{2}}{k^2} \leq \|\mu_{1,k} - \mu_{2,k}\|_{\kappa}^4 \leq  \frac{4\|\pi_{0}\|_{\kappa}^{4}(1+c)^{2}}{k^2}.
\end{equation}

{\bf Bounding $(\bm{u}^{\otimes 2})^{\top}\bm{K}_w(\theta) \bm{u}^{\otimes 2} $.} Since
$u_{i_1}u_{i_2} = (-1)^{i_1+i_2}/k^2$ for each $i_1,i_2 \in [2k]$, we write
\begin{equation*}
(\bm{u}^{\otimes 2})^{\top}\bm{K}_w(\theta) \bm{u}^{\otimes 2} = \frac{1}{k^4} \sum\limits_{i_1,i_2,i_3,i_4 \in [2k]} (-1)^{i_1+i_2+i_3+i_4} \Big[\bm{K}_w(\theta)\Big]_{i_1,i_2,i_3,i_4}. 
\end{equation*}
Consider the sets
\begin{align*}
\mathcal{I}_{=} &:= \{ (i_{1},i_2,i_3,i_4) \in [2k]^4, \:\: i_1 - i_2 = i_3 -i_4 \},\\
\mathcal{I}_{+} & := \{ (i_1,i_2,i_3,i_4)\in [2k]^4, \: (-1)^{i_1+i_2+i_3+i_4} =1\}, \\
\mathcal{I}_{-} & := \{ (i_1,i_2,i_3,i_4)\in [2k]^4, \: (-1)^{i_1+i_2+i_3+i_4} =-1\},
\end{align*}
and observe that $\mathcal{I}_{=} \subset \mathcal{I}_{+}$, $\mathcal{I}_{+} \cup \mathcal{I}_{-} = [2k]^{4}$ and that the definition~\eqref{eq:DefTensorProductKernel} implies 
\begin{align}
(-1)^{i_1+i_2+i_3+i_4}\Big[\bm{K}_w(\theta)\Big]_{i_1,i_2,i_3,i_4} &
\geq 
\begin{cases}
\kappa_{w}(0,0), & \forall (i_{1},i_2,i_3,i_4) \in \mathcal{I}_{=}\\
-|\kappa_{w}(\theta_{i_{1}}-\theta_{i_{2}},\theta_{i_{4}}-\theta_{i_{3}})| & \forall (i_{1},i_2,i_3,i_4) \in \mathcal{I}_{-}\\
-|\kappa_{w}(\theta_{i_{1}}-\theta_{i_{2}},\theta_{i_{4}}-\theta_{i_{3}})| & \forall (i_{1},i_2,i_3,i_4) \in \mathcal{I}_{+}\backslash \mathcal{I}_{=}.
\end{cases}
\label{eq:KernelWLowerBoundDetail}
\end{align}
Further observe that  
since $i_1-i_2-i_3+i_4 \equiv i_1+i_2+i_3+i_4 [2]$, if $(i_1,i_2,i_3,i_4) \in \mathcal{I}_{-} \cup (\mathcal{I}_{+}\backslash \mathcal{I}_{=})$, then either $i_1-i_2-i_3+i_4 \equiv 1[2]$ or  $i_1-i_2-i_3+i_4 \neq 0$.
In both cases, we have $|i_1-i_2-i_3+i_4| \geq 1$, hence
$\|\theta_{i_1}-\theta_{i_2}-\theta_{i_3}+\theta_{i_4}\| = \|(i_1-i_2-i_3+i_4)\theta^{*}\| \geq 1,$ and
\[
|\kappa_{w}(\theta_{i_1}-\theta_{i_2}, \theta_{i_3}-\theta_{i_4})| \leq \sup_{\theta,\theta': \|\theta-\theta'\| \geq 1} |\kappa_{w}(\theta,\theta')|
\leq \frac{c_{w}}{2k} \kappa_{w}(0,0),
\]
where we used definition~\eqref{eq:TildeKernelCoherence}. 
Observe moreover that 
since $\mathcal{I}_{-}=  \{(i_{1},i_{2},i_{3},i_{4}); 
(i_{1},i_{2},i_{3}) \in [2k]^{3}, i_{4} \in [2k], i_{4} \equiv 1-(i_{1}+i_{2}+i_{3}) [2]\}$ we have $\sharp \mathcal{I}_{-} = (2k)^{3} \times k = 8k^{4}$. Similarly, $\sharp (\mathcal{I}_{+}\backslash \mathcal{I}_{=}) = \sharp\mathcal{I}_{+} - \sharp \mathcal{I}_{=} = 8k^{4} - \sharp \mathcal{I}_{=}$. Finally, as we will show below $\sharp \mathcal{I}_{=} \geq 16k^{3}/3$ (the proof is postponed at the end of the section) we obtain
\begin{align*}
\sum\limits_{(i_1,i_2,i_3,i_4) \in[2k]}  (-1)^{i_{1}+i_{2}+i_{3}+i_{4}} \Big[\bm{K}_w(\theta)\Big]_{i_1,i_2,i_3,i_4} 
& \geq \Big(\sharp \mathcal{I}_{=} -\sharp (\mathcal{I}_{+}\backslash \mathcal{I}_{=}) \cdot \frac{c_{w}}{2k}  - \sharp \mathcal{I}_{-} \cdot \frac{c_{w}}{2k}\Big) \cdot  \kappa_{w}(0,0) \\
& =
\Big(\sharp \mathcal{I}_{=}\Big(1+\frac{c_{w}}{2k}\Big)-8c_{w} k^{3}\Big) \cdot \kappa_{w}(0,0)\\
& \geq (16/3 - 8 c_{w})k^3 \cdot \kappa_{w}(0,0).
\end{align*}
Thus,
$
(\bm{u}^{\otimes 2})^{\top}\bm{K}_w(\theta) \bm{u}^{\otimes 2} \geq  \frac{8}{k} \Big(\frac{2}{3}-c_{w}\Big) \cdot \kappa_{w}(0,0)
$
and combining with~\eqref{eq:VarThm2} and~\eqref{eq:bound_MMD_worst_example} yields
$\mathbb{V}\| \mathcal{A}\nu_{k} \|^{2} \geq \frac{1}{m}(C_{w}k-1)$ with $$C_{w}k = \frac{8}{k}(2/3-c_{w}) \kappa_{w}(0,0) \frac{k^{2}}{4\|\pi_{0}\|_{\kappa}^{4}(1+c)^{2}},$$ i.e., $C_{w}=2(2/3-c_{w}) \frac{\kappa_{w}(0,0)}{\|\pi_{0}\|_{\kappa}^{4}(1+c)^{2}}$ as defined in~\eqref{eq:TildeKernelCoherence}.

To conclude the proof of~\eqref{eq:MainThm2Var} we now establish the claimed lower bound on $\sharp \mathcal{I}_{=}$. Since $\mathcal{I}_{=}$ is the disjoint union of $\mathcal{I}^{\ell}_{=} := \{(i_{1},i_{2},i_{3},i_{4}) \in [2k]^{4}, i_{1}-i_{2}=i_{3}-i_{4}=\ell\}$, $\ell \in \{-(2k-1),\ldots,0, \ldots, 2k-1\}$, and $\sharp \mathcal{I}^{\ell}_{=} = \sharp \mathcal{I}^{-\ell}_{=} = (2k-\ell)^{2}$ for $0 \leq \ell \leq 2k-1$, hence
\begin{align}\label{eq:lower_bound_I_equal_alternated_sum}
\sharp \mathcal{I}_{=} & = \sharp \mathcal{I}^{0}_{=}+2 \sum_{\ell=1}^{2k-1} \sharp \mathcal{I}^{\ell}_{=}
= (2k)^2 + \sum\limits_{\ell = 1}^{2k-1} 2(2k-\ell)^2 
 = \Big(2\sum\limits_{\ell = 0}^{2k-1} (2k-\ell)^2\Big) - (2k)^2\nonumber \\
& = \Big(2\sum\limits_{\ell' = 1}^{2k} (\ell')^2\Big) - 4k^2
 = \frac{1}{3}2k(2k+1)(4k+1) - \frac{12}{3}k^2 \geq \frac{16}{3}k^3. 
\end{align}
where we used the well known fact that $\sum_{\ell'=1}^{n} (\ell')^{2} = n(n+1)(2n+1)/6$ for every integer $n$.

We proceed to the second claim.  For $\theta,\theta'$ such that $\|\theta-\theta'\| \geq 1$, $\iota = \pi_{\theta}$ and $\iota' = \pi_{\theta'}$ are (non-normalized) $1$-separated dipoles with respect to $\varrho$. The mutual coherence assumption and Lemma~\ref{lem:TildeKernel} yield
\begin{equation*}\label{eq:upper_bound_k_coherence_double_difference}
\frac{\kappa_{w_{0}}(\theta,\theta')}{\|\pi_{0}\|_{\kappa}^{2}} 
= \langle \pi_{\theta}, \pi_{\theta'} \rangle_{\kappa}  \leq \frac{c}{2k}\|\pi_{\theta}\|_{\kappa}\|\pi_{\theta'}\|_{\kappa} = \frac{c}{2k} \|\pi_{0}\|_{\kappa}^{2}
 = \frac{c}{2k} \langle \pi_{0},\pi_{0}\rangle_\kappa 
 = \frac{c}{2k} \frac{\kappa_{w_{0}}(0,0)}{\|\pi_{0}\|_{\kappa}^{2}}.
\end{equation*}
This shows that $c_{w_{0}}$ defined in~\eqref{eq:TildeKernelCoherence} satisfies $c_{w_{0}} \leq c$.

\subsection{Proof of Proposition~\ref{prop:lower_bound_L_constant}}\label{proof:lower_bound_L_constant}

For any mixture model, kernel, and feature family, the set of normalized dipoles $\mathfrak{D}$ is included in the normalized secant set hence
\begin{align*}
\sup_{\nu \in \mathcal{S}_{k}} \|\nu\|_{\mathcal{F}} 
=   \sup_{\omega \in \mathbb{R}^{d}} \sup_{\nu \in \mathcal{S}_{k}}|\langle \phi^{w}_{\omega},\nu \rangle|
 \geq
\sup_{\omega \in \mathbb{R}^{d}} \sup_{\nu \in \mathfrak{D}}|\langle \phi^{w}_{\omega},\nu \rangle|.
\end{align*}
For $x \in \Theta$ s.t. $0 < \|x\| \leq 1$, define $\tilde{\iota}_{x} := \pi_{x}-\pi_{0}$ . Since $\Theta$ contains a neighborhood of zero, $\pi_{0}/\|\pi_{0}\|_{\kappa}, \tilde{\iota}_{x}/\|\tilde{\iota}_{x}\|_{\kappa} \in \mathfrak{D}$ as soon as $\|x\|$ is small enough hence there is $0<\delta \leq 1$ such that
\begin{align*}
\sup_{\nu \in \mathcal{S}_{k}} \|\nu\|_{\mathcal{F}}   \geq  \sup_{\omega \in \mathbb{R}^{d}} 
\max\left(
\frac{|\langle \phi^{w}_{\omega},\pi_{0} \rangle|}{\|\pi_{0}\|_{\kappa}}, 
\sup_{0<\|x\|\leq \delta}\frac{|\langle \phi^{w}_{\omega},\pi_{x}-\pi_{0} \rangle|}{\|\pi_{x}-\pi_{0}\|_{\kappa}}
\right).
\end{align*}
Since $\langle \phi^{w}_{\omega},\pi_{x}\rangle = e^{\imath \langle\omega,x\rangle} \langle \phi^{w}_{\omega},\pi_{0}\rangle$ and  $\|\pi_{x}-\pi_{0}\|_{\kappa}^{2} = 2\|\pi_{0}\|_{\kappa}^{2}(1-\bar{\kappa}(x))$  we have
\[
\sup_{0<\|x\|\leq \delta}\frac{|\langle \phi^{w}_{\omega},\pi_{x}-\pi_{0} \rangle|}{\|\pi_{x}-\pi_{0}\|_{\kappa}}
=
|\langle \phi^{w}_{\omega},\pi_{0} \rangle|
\cdot
\sup_{0<\|x\|\leq \delta}\frac{|e^{\imath \langle\omega,x\rangle}-1|}{\|\pi_{0}\|_{\kappa}\sqrt{2(1-\bar{\kappa}(x))}}.
\]
Now, since $\bar{\kappa}$ achieved its maximum at zero where it is $C^{2}$, we have for every $x \neq 0$ 
\begin{equation*}
\lim_{t \to 0} \sqrt{\frac{1-\bar{\kappa}(tx)}{t^{2}\|x\|^{2}}} = \sqrt{\frac{-x^{\mathrm{T}}\nabla^2\bar{\kappa}(0)x}{2\|x\|^2}}.
\end{equation*}
 By assumption, $\nabla^2\bar{\kappa}(0) \neq 0$ and we get
\begin{equation*}
C:=\sup_{0<\|x\|\leq \delta}\lim_{t \to 0} \sqrt{\frac{1-\bar{\kappa}(tx)}{t^{2}\|x\|^{2}}} = \sqrt{\frac{\|\nabla^2\bar{\kappa}(0)\|_{\mathrm{op}}}{2}}>0.
\end{equation*}

We obtain
\begin{align*}
\sup_{0<\|x\|\leq \delta}\frac{|e^{\imath \langle\omega,x\rangle}-1|}{
\sqrt{1-\bar{\kappa}(x)}
}
&\geq 
\sup_{0<\|x\|\leq \delta}\lim_{t \to 0}\frac{|e^{\imath t \langle\omega,x\rangle}-1|}{
\sqrt{1-\bar{\kappa}(tx)}
}
=
\sup_{0<\|x\|\leq \delta}\lim_{t \to 0}\frac{|\imath t \langle\omega,x\rangle|}{\sqrt{1-\bar{\kappa}(tx)}}\\
&=
\sup_{0<\|x\|\leq \delta}
\frac{|\langle\omega,x/\|x\|\rangle|}{\lim_{t \to 0} \sqrt{\frac{1-\bar{\kappa}(tx)}{t^{2}\|x\|^{2}}}}
 \geq \frac{1}{C}\sup_{0<\|x\|\leq \delta}
|\langle\omega,x/\|x\|\rangle| = \frac{1}{C} \|\omega\|_{\star}.
\end{align*}

Therefore 
\begin{equation*}
\sup_{\nu \in \mathcal{S}_{k}} \|\nu\|_{\mathcal{F}}   \geq  \sup_{\omega \in \mathbb{R}^{d}} \frac{|\langle \phi^{w}_{\omega},\pi_{0} \rangle|}{\|\pi_{0}\|_{\kappa}}
\max\left(1, \frac{1}{\sqrt{2}C} \|\omega\|_{\star}
\right).
\end{equation*}
To conclude we use that $(\sqrt{2}C)^{-1} = 1/\sqrt{\|\nabla^{2} \bar{\kappa}(0)\|_{\mathrm{op}}}$.

\subsection{Proof of~\Cref{prop:generic_upper_bound_worst_sketching_error}}
\label{sec:proof dipole decomp of secant bound}
Let $\nu \in \mathcal{S}_{k}$. By~\Cref{prop:embedding_normalized_secant_set}, there exist $2k$ normalized dipoles $\iota_{1}, \dots, \iota_{2k} \in \mathfrak{D}$, with $(\iota_{i},\iota_{j}) \in \mathfrak{D}_{\neq}^2$ when $i\neq j$, and coefficients $\alpha_{1}, \dots, \alpha_{2k}\geq 0$ such that
$(1+c)^{-1}\leq \sum_{i \in [2k]} \alpha_{i}^2 \leq (1-c)^{-1}$ that satisfy
$\nu = \sum_{i \in [2k]} \alpha_i \iota_i$.
We notice that
\[
\sum\limits_{i\neq j \in [2k]} \alpha_i\alpha_j  
= (\sum_{i \in [2k]} \alpha_{i})^{2}-
\sum_{i \in [2k]} \alpha_{i}^{2}
  \leq (\sqrt{2k}\|\alpha\|_{2})^2-\|\alpha\|_{2}^{2}  \leq \frac{2k-1}{1-c}.
\]
Since $\mathcal{A}\pi_{\theta}$ is well defined for probability distributions in the family $\mathcal{T}$, the action of $\mathcal{A}$ is well defined on $k$-mixtures, hence on elements of the normalized secant set. We have
\begin{align*}
\|\mathcal{A} \nu \|_{2}^2 - \sum\limits_{i \in [2k]} \alpha_{i}^2  & = \sum\limits_{i,j \in [2k]} \alpha_i \alpha_j \langle \mathcal{A}\iota_i, \mathcal{A}\iota_j \rangle - \sum\limits_{i \in [2k]}\alpha_i ^2 \nonumber\\
& = \sum\limits_{i \in [2k]}\alpha_i ^2( \| \mathcal{A}\iota_i\|_{2}^2 -1) +  \sum\limits_{i \neq j \in [2k]} \alpha_i \alpha_j \langle \mathcal{A}\iota_i, \mathcal{A}\iota_j \rangle,
\end{align*}
we obtain
\begin{align*}
\|\mathcal{A} \nu \|_{2}^2 -1 &=  \sum\limits_{i \in [2k]} \alpha_{i}^2-1 + \sum\limits_{i \in [2k]}\alpha_i ^2( \| \mathcal{A}\iota_i \|_{2}^2 -1) +
  \sum\limits_{i \neq j \in [2k]} \alpha_i \alpha_j \langle \mathcal{A}\iota_i, \mathcal{A}\iota_j \rangle.
\end{align*}
For $i \neq j$, since $\langle \mathcal{A}\iota_j, \mathcal{A}\iota_i \rangle$ is the complex conjugate of $\langle \mathcal{A}\iota_i, \mathcal{A}\iota_j \rangle$, we have $\langle \mathcal{A}\iota_i, \mathcal{A}\iota_j \rangle + \langle \mathcal{A}\iota_j, \mathcal{A}\iota_i \rangle = \mathfrak{Re}\langle \mathcal{A}\iota_i, \mathcal{A}\iota_j \rangle +  \mathfrak{Re}\langle \mathcal{A}\iota_j, \mathcal{A}\iota_i\rangle$ hence 
$\sum\limits_{i \neq j \in [2k]} \alpha_i \alpha_j \langle \mathcal{A}\iota_i, \mathcal{A}\iota_j \rangle
=  \sum\limits_{i \neq j \in [2k]} \alpha_i \alpha_j \mathfrak{Re}\langle \mathcal{A}\iota_i, \mathcal{A}\iota_j \rangle$.
As a result
\begin{align*}
\big| \|\mathcal{A} \nu \|_{2}^2 - 1 \big| &\leq \big| 1 - \sum\limits_{i \in [2k]} \alpha_{i}^2 \big| + \Big|\sum\limits_{i \in [2k]}\alpha_i ^2( \| \mathcal{A}\iota_i\|_{2}^2 -1)\Big| +  \Big|\sum\limits_{i \neq j \in [2k]} \alpha_i \alpha_j \mathfrak{Re}\langle \mathcal{A}\iota_i, \mathcal{A}\iota_j \rangle\Big| \nonumber \\
&\leq \big| 1 - \sum\limits_{i \in [2k]} \alpha_{i}^2 \big| + \sum\limits_{i \in [2k]}\alpha_i ^2 \sup\limits_{\iota \in \mathfrak{D}}\big| \| \mathcal{A}\iota \|_{2}^2 -1\big| +  \sum\limits_{i \neq j \in [2k]} \alpha_i \alpha_j \sup\limits_{(\iota,\iota') \in \mathfrak{D}_{\neq}^2}
\big|\mathfrak{Re}\langle \mathcal{A}\iota, \mathcal{A}\iota' \rangle\big|
. 
\end{align*}
Now, since $(1+c)^{-1} \leq \sum_{i \in [2k]} \alpha_{i}^2 \leq (1-c)^{-1}$
we have $\big| 1- \sum_{i \in [2k]} \alpha_{i}^2 \big|  \leq c/(1-c)$, hence
\begin{align*}
\big| \|\mathcal{A} \nu \|_{2}^2 - 1 \big|  & \leq \frac{1}{1-c} \Bigg(c+ \sup\limits_{\iota \in \mathfrak{D}}\big| \| \mathcal{A}\iota \|_{2}^2 -1\big| +  (2k-1) \sup\limits_{(\iota,\iota') \in \mathfrak{D}_{\neq}^2}
\big|\mathfrak{Re}\langle \mathcal{A}\iota, \mathcal{A}\iota' \rangle\big|
\Bigg).
\end{align*}
Since this holds for every $\nu \in \mathcal{S}_{k}$, this establishes~\eqref{eq:upper_bound_worst_sketching_error_two_sides} using the definitions of $\delta(\cdot|\mathcal{A})$ (see~\eqref{eq:DefDelta}) and of $\mu(\mathfrak{D}_{\neq}^{2}|\mathcal{A})$ (see~\eqref{eq:DefGamma}).
\qed

\subsection{Proof of~\Cref{thm:upper_bound_sup_monopoles_dipoles} }\label{sec:proofs_second_decomposition}
\label{app:pf_Th3-4}
\label{proof:upper_bound_sup_monopoles_dipoles}
To prove~\Cref{thm:upper_bound_sup_monopoles_dipoles} we rely on a generic formula expressing $\|\mathcal{A}\iota\|_{2}^2$ for $\iota \in \mathfrak{D}$ that depends on a scalar $\alpha \in [0,1]$, which reflects how {\em balanced} the normalized dipole $\iota$ is, and on a vector $x$, which reflects the {\em relative position} between the supports of the two monopoles that form $\iota$. We also exploit an expression of $\mathfrak{Re} \langle \mathcal{A}\iota, \mathcal{A}\iota' \rangle$ for $(\iota,\iota') \in \mathfrak{D}_{\neq}^2$ that depends on two scalars $\alpha, \alpha'$ reflecting how balanced $\iota$ and $\iota'$ are respectively, and on two relative-position vectors $x,x'$. 
The proof of the following proposition is deferred to~\Cref{sec:proof_param_delta}.

\begin{proposition}\label{prop:param_delta}
Consider $\mathcal{T}= (\Theta, \rho, \mathcal{I})$ a location-based family with base distribution $\pi_{0}$ where $\rho(\cdot,\cdot) = \|\cdot-\cdot\|$ for some norm $\|\cdot\|$, and $\kappa$ a normalized shift-invariant kernel that is locally characteristic with respect to $\mathcal{T}$. Consider
$\mathcal{A}$ a WFF sketching operator (\Cref{def:RFFsketching}) with frequencies $\omega_{1},\ldots,\omega_{m}$, $\Theta_{\mathrm{d}}$ and $\psi(\omega)$ as defined in~\eqref{eq:bound_gammma_d} and~\eqref{eq:DefRho}, and $\bar{\kappa}$ as defined in~\eqref{eq:DefNormalizedKernel}.
\begin{itemize}
\item For any normalized dipole $\iota \in \mathfrak{D}$, there exists $\alpha \in [0,1]$ and a vector $x \in \Theta_{\mathrm{d}}$ 
such that 
\begin{equation}\label{eq:param_delta_dipole}
\|\mathcal{A}\iota\|_{2}^2 = \frac{1}{m} \sum\limits_{j =1}^{m} \psi(\omega_j)  \frac{(1-\alpha)^2 + 2\alpha (1-\cos \langle \omega_j, x \rangle)}{(1-\alpha)^2+2\alpha(1-\bar{\kappa}(x))} .
\end{equation}
The case of a normalized monopole $\iota  \in \mathfrak{M}$ corresponds to $\alpha =0$ (and arbitrary $x$), while the case of a balanced normalized dipole corresponds to $\alpha =1$.
Vice-versa, for any $\alpha \in [0,1]$ and $x \in \Theta_{\mathrm{d}}$ there is $\iota \in \mathfrak{D}$ such that this equality holds.
\item For $(\iota,\iota') \in \mathfrak{D}_{\neq}^2$, there exist $s,s' \in \{-1,1 \}$, $\alpha, \alpha' \in [0,1]$, $\theta_{1}, \theta_{2}, \theta_{1}', \theta_{2}' \in \Theta$ that satisfy 
\begin{align}\label{eq:param_delta_1separateddipoles}
\forall i,j \in \{1,2\}, & \:\:\:\: \varrho(\theta_i, \theta_j') \geq 1, \nonumber \\
 0<\varrho( \theta_1,\theta_2 ) \leq 1, & \:\:\:\: 0<\varrho(\theta_1',\theta_2') \leq 1,
\end{align}
 such that
\begin{equation}\label{eq:cross_dipole_expression_using_alphas_and_x_proof}
\mathfrak{Re} \langle \mathcal{A}\iota, \mathcal{A}\iota' \rangle = ss'\frac{a - \alpha' b - \alpha c + \alpha \alpha' d}{\sqrt{(1-\alpha)^2 + 2\alpha(1-\bar{\kappa}(x))} \sqrt{(1-\alpha')^2 + 2\alpha'(1-\bar{\kappa}(x'))}},
\end{equation}
where $x:= \theta_{1}-\theta_2 \in \Theta_{\mathrm{d}}$ and $x':= \theta_{1}'-\theta_2' \in \Theta_{\mathrm{d}}$, and $a,b,c$ and $d$ are defined as follows
\begin{equation}\label{eq:cross_dipole_expression_using_alphas_and_x_paramsabcd}
\begin{aligned}
a &= \frac{1}{\|\pi_{0}\|_{\kappa}^2}\mathfrak{Re}\langle \mathcal{A}\pi_{\theta_1}, \mathcal{A}\pi_{\theta_1'} 
\rangle, &
 b &= \frac{1}{\|\pi_{0}\|_{\kappa}^2}\mathfrak{Re}\langle \mathcal{A}\pi_{\theta_1}, \mathcal{A}\pi_{\theta_2'} \rangle,
\\
c &= \frac{1}{\|\pi_{0}\|_{\kappa}^2}\mathfrak{Re}\langle \mathcal{A}\pi_{\theta_2}, \mathcal{A}\pi_{\theta_1'} \rangle,
& d &= \frac{1}{\|\pi_{0}\|_{\kappa}^2}\mathfrak{Re}\langle \mathcal{A}\pi_{\theta_2}, \mathcal{A}\pi_{\theta_2'} \rangle.
\end{aligned}
\end{equation}
The case where $\iota$ is a monopole (resp. a balanced dipole) corresponds to $\alpha=0$ (resp. $\alpha = 1$) and similarly for $\iota'$ and $\alpha'$.

\end{itemize}

\end{proposition}

\paragraph{Proof of~\eqref{eq:bound_diagonal_sketch_simplification_2}.}
Since normalized monopoles and balanced normalized dipoles are special cases of normalized dipoles, we have $\mathfrak{M} \subset \mathfrak{D}$ and $\hat{\mathfrak{D}} \subset \mathfrak{D}$ hence by the definition~\eqref{eq:DefDelta} of $\delta(\mathfrak{D}|\mathcal{A}) $ as a supremum we trivially have
\[
\delta(\mathfrak{D}|\mathcal{A}) \geq \max\big(\delta(\mathfrak{M}|\mathcal{A}), \delta(\hat{\mathfrak{D}}|\mathcal{A})\big).
\]
To establish~\eqref{eq:bound_diagonal_sketch_simplification_2} we show the converse inequality.
First observe that for any normalized monopole $\iota_{\mathfrak{M}} \in \mathfrak{M}$ we have
$\|\mathcal{A}\iota_{\mathfrak{M}}\|_{2}^2 = \frac{1}{m} \sum_{j =1}^{m}  \psi(\omega_j)$. Since $\varrho$ is a norm,  this is a direct consequence of Proposition~\ref{prop:param_delta} and shows that $\delta(\mathfrak{M}|\mathcal{A}) = |1-\|\mathcal{A}\iota_{\mathfrak{M}}\|_{2}^2|$ independently of the choice of the monopole $\iota_{\mathfrak{M}}$.

Now, consider an arbitrary normalized dipole  $\iota \in \mathfrak{D}$. 
If $\iota$ is either a normalized monopole or a balanced dipole, we trivially have $|1-\|\mathcal{A}\iota\|_{2}^{2}| \leq \max(\delta(\mathfrak{M}|\mathcal{A}),\delta(\hat{\mathfrak{D}}|\mathcal{A}))$. Otherwise, 
by Proposition~\ref{prop:param_delta} again, there exists $\alpha \in (0,1)$ such that \begin{equation*}
\|\mathcal{A}\iota\|_{2}^2 = \frac{1}{m} \sum\limits_{j =1}^{m} \psi(\omega_j) \frac{ (1-\alpha)^2 + 2\alpha (1-\cos \langle \omega_j, x \rangle) }{(1-\alpha)^2 + 2\alpha(1-\bar{\kappa}(x))}.
\end{equation*}
Define a {\em balanced} normalized dipole as
$\iota_{\hat{\mathfrak{D}}} := \frac{\pi_{0}- \pi_{x}}{\|\pi_{0}- \pi_{x}\|_{\kappa}} \in \hat{\mathfrak{D}}$.
Proposition~\ref{prop:param_delta} again yields
\begin{equation*}
\|\mathcal{A} \iota_{\hat{\mathfrak{D}}}\|_{2}^2 = \frac{1}{m} \sum\limits_{j =1}^{m}  \psi(\omega_j) \frac{1-\cos \langle \omega_j, x \rangle}{1-\bar{\kappa}(x)}
\end{equation*}
so that, with simple algebraic manipulations, we have
\begin{align*}
\|\mathcal{A}\iota\|_{2}^2 & =  \frac{(1-\alpha)^2 \frac{1}{m} \sum\limits_{j =1}^{m} \psi(\omega_j) + 2\alpha \frac{1}{m} \sum\limits_{j =1}^{m}  \psi(\omega_j)(1-\cos \langle \omega_j, x \rangle)}{(1-\alpha)^2 + 2\alpha(1-\bar{\kappa}(x))} \\
& =  \frac{(1-\alpha)^2 \|\mathcal{A}\iota_{\mathfrak{M}}\|_{2}^2 + 2\alpha (1-\bar{\kappa}(x)) \|\mathcal{A}\iota_{\hat{\mathfrak{D}}}\|_{2}^2}{(1-\alpha)^2 + 2\alpha(1-\bar{\kappa}(x))}.
\end{align*}
We deduce that
\begin{equation*} \min\big( \|\mathcal{A}\iota_{\mathfrak{M}}\|_{2}^2, \|\mathcal{A}\iota_{\hat{\mathfrak{D}}}\|_{2}^2  \big) \leq
\|\mathcal{A}\iota\|_{2}^2 \leq \max\big( \|\mathcal{A}\iota_{\mathfrak{M}}\|_{2}^2, \|\mathcal{A}\iota_{\hat{\mathfrak{D}}}\|_{2}^2  \big).
\end{equation*}
Moreover, since $\iota_{\hat{\mathfrak{D}}} \in \hat{\mathfrak{D}}$, we have
$
\max(1-\|\mathcal{A}\iota_{\hat{\mathfrak{D}}}\|_{2}^2,\|\mathcal{A}\iota_{\hat{\mathfrak{D}}}\|_{2}^2-1) = |1-\|\mathcal{A}\iota_{\hat{\mathfrak{D}}}\|_{2}^2| \leq \delta(\hat{\mathfrak{D}}|\mathcal{A})
$
hence
\begin{align*}
1- \|\mathcal{A}\iota\|_{2}^2  
& \leq 1- \min \big( \|\mathcal{A}\iota_{\mathfrak{M}}\|_{2}^2 , \|\mathcal{A}\iota_{\hat{\mathfrak{D}}}\|_{2}^2 \big) \\
&  =  \max \big( 1-\|\mathcal{A}\iota_{\mathfrak{M}}\|_{2}^2 , 1-\|\mathcal{A}\iota_{\hat{\mathfrak{D}}}\|_{2}^2  \big) 
 \leq \max\big( \delta(\mathfrak{M}|\mathcal{A}), \delta(\hat{\mathfrak{D}}|\mathcal{A})\big),\\
\|\mathcal{A}\iota\|_{2}^2 -1 
& \leq \max\big( \|\mathcal{A}\iota_{\mathfrak{M}}\|_{2}^2-1, \|\mathcal{A}\iota_{\hat{\mathfrak{D}}}\|_{2}^2-1  \big) 
 \leq \max \big( \delta(\mathfrak{M}|\mathcal{A}), \delta(\hat{\mathfrak{D}}|\mathcal{A})\big),
\end{align*}
This shows that $|1-\|\mathcal{A}\iota\|_{2}^2| \leq \max \big( \delta(\mathfrak{M}|\mathcal{A}), \delta(\hat{\mathfrak{D}}|\mathcal{A})\big)$ and establishes~\eqref{eq:bound_diagonal_sketch_simplification_2} as claimed. \qed

\paragraph{Proof of~\eqref{eq:bound_cross_sketch_simplification}.}
Recall that the families defined in~\eqref{eq:DefineMDfamilies} satisfy 
 $\mathfrak{M}_{\neq}^2,\mathfrak{M}\times\hat{\mathfrak{D}}_{\neq},\hat{\mathfrak{D}}_{\neq}^2 \subset \mathfrak{D}_{\neq}^2$. Hence, by the definition~\eqref{eq:DefGamma} of $\mu(\cdot|\mathcal{A})$ as a supremum we trivially have
\begin{equation*}
\mu(\mathfrak{D}_{\neq}^2|\mathcal{A}) \geq 
\max \big( \mu(\mathfrak{M}_{\neq}^2|\mathcal{A}), \mu(\mathfrak{M}\times\hat{\mathfrak{D}}_{\neq}|\mathcal{A}),\mu(\hat{\mathfrak{D}}_{\neq}^2|\mathcal{A}) \big).
\end{equation*}
This yields the lower bound in \eqref{eq:bound_cross_sketch_simplification}. To establish the upper bound we will use a  result which proof is postponed to Section~\ref{sec:proof_bound_on_ratio_2D}.

\begin{proposition}\label{prop:bound_on_ratio_2D}
Let $a,b,c,d \in \mathbb{R}$ and $e,f \in [0,1)$ and consider the function $g$ defined on $[0,1]\times[0,1]$ by
\begin{equation}\label{eq:funbound_on_ratio_2D}
h(u,v) = \frac{a - bu -cv +duv}{\sqrt{1+u^2-2eu} \sqrt{1+v^2-2fv}}.
\end{equation}
We have
\begin{equation}
\sup\limits_{ (u,v) \in [0,1]^2} \:\: \big|h(u,v)\big| \leq 3 \max \Big( |a|,|b|,|c|,|d|, \frac{|b-a|}{\sqrt{1-e}} , \frac{|d-c|}{\sqrt{1-e}} , \frac{|d-b|}{\sqrt{1-f}} , \frac{|c-a|}{\sqrt{1-f}} , \frac{|a - b  -c+d|}{\sqrt{1-e}\sqrt{1-f}} \Big). \nonumber
\end{equation}
\end{proposition}
Consider an arbitrary $1$-separated pair of normalized dipoles, $(\iota,\iota') \in \mathfrak{D}_{\neq}^2$, and denote $s,s' \in \{-1,1 \}$, $\alpha, \alpha' \in [0,1]$, $\theta_{1}, \theta_{2}, \theta_{1}', \theta_{2}' \in \Theta$, $x$, $a,b,c,d$ the parameters satisfying~\eqref{eq:param_delta_1separateddipoles}-\eqref{eq:cross_dipole_expression_using_alphas_and_x_proof}-\eqref{eq:cross_dipole_expression_using_alphas_and_x_paramsabcd} as given by Proposition~\ref{prop:param_delta}. As $x,x' \in \Theta_{\mathrm{d}}$ we have $0<\varrho(x,0) \leq 1$ and $0<\varrho(x',0) \leq 1$, and since $\kappa$ is \emph{locally characteristic}
\footnote{See Definition~\ref{def:kernel_assumptions}.} with respect to $\mathcal{T}$ and as $\kappa \geq 0$ implies $\bar{\kappa} \geq 0$ (cf~\eqref{eq:MMDInnerProd} and \eqref{eq:DefNormalizedKernel}) , this yields $e:=\bar{\kappa}(x'), f:=\bar{\kappa}(x) \in 
[0,1)$
so that
the expression~\eqref{eq:cross_dipole_expression_using_alphas_and_x_proof} reads as $h(\alpha',\alpha)$ with $h$ as in~\eqref{eq:funbound_on_ratio_2D}.
Since $\alpha, \alpha' \in [0,1]$ and $e,f \in [0,1)$, by Proposition~\ref{prop:bound_on_ratio_2D} applied to the absolute value of the expression~\eqref{eq:cross_dipole_expression_using_alphas_and_x_proof} we get
\begin{equation*}
\big| \mathfrak{Re}\langle \mathcal{A}\iota, \mathcal{A}\iota' \rangle \big| \leq 3 \max \Big( |a|,|b|,|c|,|d|, \frac{|b-a|}{\sqrt{1-e}} , \frac{|c-d|}{\sqrt{1-e}} , \frac{|b-d|}{\sqrt{1-f}} , \frac{|a-c|}{\sqrt{1-f}} , \frac{|b-a - c+d|}{\sqrt{1-e}\sqrt{1-f}} \Big).
\end{equation*}
Now, observe that \eqref{eq:cross_dipole_expression_using_alphas_and_x_paramsabcd} in Proposition~\ref{prop:param_delta} implies that every $t \in \{a,b,c,d\}$ can be written as $t = \mathfrak{Re}\langle \mathcal{A}\nu, \mathcal{A}\nu'\rangle$ where 
$
(\nu,\nu') \in \mathfrak{M}_{\neq}^2$, hence
\begin{equation*}
\max( |a|, |b|, |c|, |d| )  \leq \sup\limits_{(\xi,\xi') \in \mathfrak{M}_{\neq}^2} \big| \mathfrak{Re}\langle \mathcal{A}\xi, \mathcal{A}\xi' \rangle \big| = \mu(\mathfrak{M}_{\neq}^2|\mathcal{A}).
\end{equation*}
Similarly, observe that \eqref{eq:cross_dipole_expression_using_alphas_and_x_paramsabcd} implies also that
$|b-a|/\sqrt{1-e} = |\mathfrak{Re}\langle \mathcal{A}\nu, \mathcal{A} \nu' \rangle|$ where
\begin{equation*}
\nu:= \frac{\pi_{\theta_{1}}}{\|\pi_{\theta_{1}}\|_{\kappa}} \in \mathfrak{M}; \:\:\:\: \nu':= \frac{\pi_{\theta_{1}'}-\pi_{\theta_{2}'}}{\|\pi_{\theta_{1}'}-\pi_{\theta_{2}'}\|_{\kappa}} \in \hat{\mathfrak{D}},
\end{equation*}
and by~\eqref{eq:param_delta_1separateddipoles} we have $(\nu,\nu') \in \mathfrak{M} \times \hat{\mathfrak{D}}_{\neq}$ hence
\begin{equation*}
\frac{|b-a|}{\sqrt{1-e}} \leq \sup\limits_{(xi,\xi') \in \mathfrak{M} \times \hat{\mathfrak{D}}_{\neq}} \big|\mathfrak{Re}\langle \mathcal{A}\xi, \mathcal{A}\xi' \rangle \big| =: \mu(\mathfrak{M}\times\hat{\mathfrak{D}}_{\neq}|\mathcal{A}).
\end{equation*}
By symmetry, the same argument is valid for $|c-d|/\sqrt{1-e} , |b-d|/\sqrt{1-f}$ and $|a-c|/\sqrt{1-f}$. Therefore
\begin{equation*}
\max\Big( \frac{|b-a|}{\sqrt{1-e}} , \frac{|c-d|}{\sqrt{1-e}} , \frac{|b-d|}{\sqrt{1-f}} , \frac{|a-c|}{\sqrt{1-f}} \Big) \leq 
\mu(\mathfrak{M}\times\hat{\mathfrak{D}}_{\neq}|\mathcal{A}).
\end{equation*}
Finally, with $\nu := (\pi_{\theta_{1}}-\pi_{\theta_{2}})/\|\pi_{\theta_{1}}-\pi_{\theta_{2}}\|_{\kappa}$ and $\nu' := (\pi_{\theta'_{1}}-\pi_{\theta'_{2}})/\|\pi_{\theta'_{1}}-\pi_{\theta'_{2}}\|_{\kappa}$,  we have $(\nu,\nu') \in \hat{\mathfrak{D}}_{\neq}^2$ (by ~\eqref{eq:param_delta_1separateddipoles}) and as a result
\begin{equation*}
\frac{|b-a - c+d|}{\sqrt{1-e}\sqrt{1-f}}  = \big|\mathfrak{Re}\langle \mathcal{A}\nu, \mathcal{A}\nu' \rangle \big|
\leq \sup\limits_{(\xi,\xi') \in \hat{\mathfrak{D}}_{\neq}^2} \big| \mathfrak{Re}\langle \mathcal{A}\xi, \mathcal{A}\xi' \rangle \big|=: \mu(\hat{\mathfrak{D}}_{\neq}^2|\mathcal{A}) .
\end{equation*}
Combining all of the above yields \begin{equation*}
\big| \mathfrak{Re}\langle \mathcal{A}\iota, \mathcal{A}\iota' \rangle \big| \leq 
3\max \big( 
\mu(\mathfrak{M}_{\neq}^2|\mathcal{A}), 
\mu(\mathfrak{M}\times\hat{\mathfrak{D}}_{\neq}|\mathcal{A}),
\mu(\hat{\mathfrak{D}}_{\neq}^2|\mathcal{A}) \big).
\end{equation*}
As this holds for every $(\iota,\iota') \in \mathfrak{D}_{\neq}^2$ this establishes \eqref{eq:bound_cross_sketch_simplification}.\qed
\subsubsection{Proof of Proposition~\ref{prop:param_delta}}\label{sec:proof_param_delta}

Consider a normalized dipole $\iota \in \mathfrak{D}$. Since $\rho$ is a norm we can apply Lemma C.1 in \citep{GrBlKeTr20} hence there exists a dipole $\tilde{\iota}$ such that $\iota = \frac{\tilde{\iota}}{\|\tilde{\iota}\|_{\kappa}}$, with 
$\tilde{\iota} = \frac{s}{\|\pi_0\|_{\kappa}} ( \pi_{\theta_1} - \alpha \pi_{\theta_2})$,
where $s \in \{-1,1 \}$, $\alpha \in [0,1]$ and $x:= \theta_{1}-\theta_2 \in \Theta-\Theta$ satisfies $0<\|x\| \leq 1$. Since $\kappa$ is locally characteristic we have $\|\tilde{\iota}\|_{\kappa}>0$ hence the ratio $\tilde{\iota}/\|\tilde{\iota}\|_{\kappa}$ indeed makes sense. The case of a normalized monopole $\iota \in \mathfrak{M}$ corresponds to $\alpha=0$ and an arbitrary $x$, while the case of a balanced dipole corresponds to $\alpha = 1$.
Moreover, since $\kappa$ is translation invariant and $\mathcal{T}$ is a location-based family by~\eqref{eq:ShiftInvariantKernelNorm} we have
$\|\pi_{\theta_1}\|_{\kappa} = \|\pi_{\theta_2}\|_{\kappa} = \|\pi_{0}\|_{\kappa}$,
and 
$\langle \pi_{\theta_1}, \pi_{\theta_2} \rangle_{\kappa} = \bar{\kappa}(\theta_1, \theta_2)\cdot \|\pi_{\theta_{1}}\|_{\kappa}\|\pi_{\theta_{2}}\|_{\kappa} = \bar{\kappa}(x)\cdot \|\pi_{0}\|_{\kappa}^2$ where we recall that the $\mathcal{T}$-normalized kernel $\bar{\kappa}$ is defined in~\eqref{eq:DefNormalizedKernel}.
Therefore
\begin{align*}
\|\tilde{\iota}\|_{\kappa}^2 &= (1-\alpha)^2 + 2\alpha(1-\bar{\kappa}(x)),\\
\|\mathcal{A}\tilde{\iota}\|_{2}^2 &= \frac{1}{m} \sum\limits_{j =1}^{m} \psi(\omega_j) \Big( (1-\alpha)^2 + 2\alpha (1-\cos \langle \omega_j, x \rangle) \Big).
\end{align*}
Since $\|\mathcal{A}\iota\|_{2}^{2} = \|\mathcal{A}\tilde{\iota}\|_{2}^{2}/\|\tilde{\iota}\|_{\kappa}^{2}$, taking the quotient yields~\eqref{eq:param_delta_dipole} as claimed. 
Vice-versa for $\alpha \in [0,1]$ and $x \in \Theta_{\mathrm{d}}$, there are $\theta_{1},\theta_{2} \in \Theta^{2}$ such that $0<\|\theta_{1}-\theta_{2}\| \leq 1$ and setting $\iota = (\pi_{\theta_{1}}-\alpha\pi_{\theta_{2}})/\|\pi_{\theta_{1}}-\alpha\pi_{\theta_{2}}\|_{\kappa}$ yields a normalized dipole satisfying the desired expression.

Similarly, for any $1$-separated pair of normalized dipoles $(\iota,\iota') \in \mathfrak{D}_{\neq}^2$
we write $\iota = \frac{\tilde{\iota}}{\|\tilde{\iota}\|_{\kappa}}, \:\: \iota' = \frac{\tilde{\iota}'}{\|\tilde{\iota}'\|_{\kappa}}$,
with 
\begin{equation*}
\tilde{\iota} = \frac{s}{\|\pi_0\|_{\kappa}} ( \pi_{\theta_1} - \alpha \pi_{\theta_2}), \:\: \tilde{\iota}' = \frac{s'}{\|\pi_0\|_{\kappa}} ( \pi_{\theta_1'} - \alpha' \pi_{\theta_2'}),
\end{equation*}
where $s,s' \in \{-1,1 \}$, $\alpha, \alpha' \in [0,1]$ and $x:= \theta_{1}-\theta_2 \in \Theta-\Theta$ and $x':= \theta_{1}'-\theta_2' \in \Theta-\Theta$ satisfy $0<\|x\| \leq 1$ and  $0<\|x'\| \leq 1$. The $1$-separation assumption means that for every
$i,j \in \{1,2\}$ we have$\| \theta_i - \theta_j' \| \geq 1$. Since
$\|\pi_{\theta_1}\|_{\kappa} = \|\pi_{\theta_2}\|_{\kappa} = \|\pi_{\theta_1'}\|_{\kappa} = \|\pi_{\theta_2'}\|_{\kappa} = \|\pi_{0}\|_{\kappa}$
and 
$\langle \pi_{\theta_1}, \pi_{\theta_2} \rangle_{\kappa}  = \|\pi_0\|_{\kappa}^2\bar{\kappa}(x)$,
and
$\langle \pi_{\theta_{1}'}, \pi_{\theta_{2}'} \rangle_{\kappa} =  
\|\pi_0\|_{\kappa}^2\bar{\kappa}(x')$
we obtain
\begin{align*}
\|\tilde{\iota}\|_{\kappa}^2 &= (1-\alpha)^2 + 2\alpha(1-\bar{\kappa}(x)),\\
\|\tilde{\iota}'\|_{\kappa}^2 &= (1-\alpha')^2 + 2\alpha'(1-\bar{\kappa}(x'))\\
\mathfrak{Re} \big(\langle \mathcal{A}\tilde{\iota}, \mathcal{A}\tilde{\iota}' \rangle \big) 
& = \frac{ss'}{\|\pi_0\|_{\kappa}^2}
\mathfrak{Re} \Big(\big\langle \mathcal{A} (\pi_{\theta_1} - \alpha \pi_{\theta_2}) , \mathcal{A}(\pi_{\theta_1^{'}} - \alpha \pi_{\theta_2^{'}}) \big\rangle \Big)
 = ss'\left(a - \alpha' b - \alpha c + \alpha \alpha' d\right),
\end{align*}
with $a,b,c,d$ as in~\eqref{eq:cross_dipole_expression_using_alphas_and_x_paramsabcd}.
Since $\mathfrak{Re} \big(\langle \mathcal{A}\iota, \mathcal{A}\iota' \rangle \big) = \mathfrak{Re} \big(\langle \mathcal{A}\tilde{\iota}, \mathcal{A}\tilde{\iota}' \rangle \big)/(\|\tilde{\iota}\|_{\kappa}\|\tilde{\iota}'\|_{\kappa})$ we obtain~\eqref{eq:cross_dipole_expression_using_alphas_and_x_proof}. 
The special cases of monopoles and balanced dipoles, 
are proved similarly as above and left to the reader.
\qed

\subsubsection{Proof of Proposition~\ref{prop:bound_on_ratio_2D}}\label{sec:proof_bound_on_ratio_2D}

We will use a technical Lemma which proof is postponed to the end of the section.
\begin{lemma}\label{lemma:bound_on_ratio_1D}
Consider $\alpha,\beta \in \mathbb{R}$, $\gamma \in [0,1)$ and the function $\phi_{\alpha,\beta,\gamma}$ defined on $[0,1]$ by
\begin{equation*}\label{eq:DefIntermediatePhiFunction}
\phi_{\alpha,\beta,\gamma}(t) = \frac{\alpha - \beta t}{\sqrt{1+t^2-2\gamma t}}.
\end{equation*}
We have
\begin{equation*}
\forall t \in [0,1], \:\: |\phi_{\alpha,\beta,\gamma}(t)| \leq \sqrt{3} \max\Big(|\alpha|, |\beta|,\frac{|\beta-\alpha|}{\sqrt{2(1-\gamma)}}\Big).
\end{equation*}
\end{lemma}
Consider $(u,v) \in [0,1]^2$, $f \in [0,1)$, and define $ g := \frac{1}{\sqrt{1+v^2-2fv}}$. We have
\begin{equation*}
h(u,v)  = \frac{ga - gbu -gcv +gduv}{\sqrt{1+u^2-2eu}} = \frac{(ga -gcv) - (gb  -gdv)u }{\sqrt{1+u^2-2eu}} = \phi_{ga -gcv,gb  -gdv,e}(u)
\end{equation*}
hence, by Lemma~\ref{lemma:bound_on_ratio_1D} 
\begin{align*}|h(u,v)| & = |\phi_{ga -gcv,gb  -gdv,e}(u)|  \leq \sqrt{3} \max \Big(|ga -gcv|,|gb  -gdv|, \frac{|gb  -gdv - ga +gcv|}{\sqrt{1-e}} \Big).
\end{align*}
Now, we can use Lemma~\ref{lemma:bound_on_ratio_1D}  again to get
\begin{align*}|ga -gcv| & = \frac{|a-cv|}{\sqrt{1+v^2-2fv}}  = |\phi_{a,c,f}(v)|  \leq \sqrt{3} \max \Big(|a|,|c|, \frac{|c-a|}{\sqrt{1-f}} \Big),\\
|gb  -gdv| & = \frac{|b-dv|}{\sqrt{1+v^2-2fv}}  = |\phi_{b,d,f}(v)|  \leq \sqrt{3} \max \Big(|b|,|d|, \frac{|d-b|}{\sqrt{1-f}} \Big),\\
|gb  -gdv - ga +gcv| & = |gb - ga  +(gc-gd)v|  = \frac{|b - a  -(d-c)v|}{\sqrt{1+v^2-2fv}} 
 = |\phi_{b-a,d-c,f}(v)| \nonumber \\
& \leq \sqrt{3} \max \Big( |b-a|, |d-c|, \frac{|(d-c)-(b-a)|}{\sqrt{1-f}} \Big).
\end{align*}
Combining the above inequalities 
we obtain
\begin{equation*}
|h(u,v)| \leq 3 \max \Big( |a|,|b|,|c|,|d|, \frac{|b-a|}{\sqrt{1-e}} , \frac{|d-c|}{\sqrt{1-e}} , \frac{|d-b|}{\sqrt{1-f}} , \frac{|c-a|}{\sqrt{1-f}} , \frac{|a - b  -c+d|}{\sqrt{1-e}\sqrt{1-f}} \Big).\qed
\end{equation*}

\begin{proof}[Proof of Lemma~\ref{lemma:bound_on_ratio_1D}]

Equivalently, we bound $c := \sup_{t \in [0,1]} g(t)$ where $g(t) := |\phi_{\alpha,\beta,\gamma}(t)|^{2} = P(t)/Q(t)$, $P(t) := (\beta t-\alpha)^{2}$, and $Q(t) := 1+t^{2}-2\gamma t$. The bound $c \leq 3\max(\alpha^{2},\beta^{2}, (\beta-\alpha)^{2}/(1-\gamma))$ is trivial if $\alpha = \beta = 0$, so we now assume $(\alpha,\beta) \neq 0$. 
We have $g(0) = \alpha^{2}$ and $g(1) = (\beta-\alpha)^{2}/(2(1-\gamma))$ so the bound is also trivial if the maximum is achieved at a boundary point, so to conclude we now assume that $c$ is achieved at an interior point $t \in (0,1)$, which must satisfy $g'(t)=0$. Since $g' = (P'Q-PQ')/Q^{2}$ the fact that $g'(t) = 0$ reads as 
\begin{align*}
0 &= P'(t)Q(t)-P(t)Q'(t)\\
& = 2\beta(\beta t-\alpha)(1+t^{2}-2\gamma t)-(\beta t-\alpha)^{2} 2(t-\gamma)
= 2(\beta t-\alpha) \big(\beta(1+t^{2}-2\gamma t) -(\beta t-\alpha) (t-\gamma)\big)\\
&= 2(\beta t-\alpha) \big(\beta +\bcancel{\beta t^{2}}-2\beta\gamma t -\bcancel{\beta t^{2}}+\alpha t+\beta\gamma t-\alpha\gamma\big)
 = 2(\beta t-\alpha) \big( (\alpha-\beta\gamma)t-(\alpha\gamma-\beta)\big).
\end{align*} 
Since we assume $(\alpha,\beta) \neq (0,0)$, we have $P(t)/Q(t) = g(t) \geq \max(g(0),g(1)) = \max (\alpha^{2},   (\beta-\alpha)^{2}/(2(1-\gamma))) > 0$, hence $P(t) \neq 0$, i.e. $\beta t-\alpha\neq0$, thus the location of the maximum satisfies
\[
(\alpha-\beta\gamma)t = \alpha\gamma-\beta.
\]
This implies that $\alpha \neq \beta\gamma$ (otherwise we would have both $\alpha = \beta\gamma$ and $\alpha\gamma=\beta$, hence $\beta = \alpha\gamma = (\beta\gamma)\gamma= \beta\gamma^{2}$, and similarly $\alpha = \alpha \gamma^{2}$; since $0 \leq \gamma < 1$ this would contradict the fact that $(\alpha,\beta) \neq (0,0)$).
Moreover, since $P'(t)Q(t)-P(t)Q'(t) = 0$ we have $g(t) = P(t)/Q(t) = P'(t)/Q'(t) = 2\beta(\beta t-\alpha)/(2(t-\gamma))$. Since $g(t)>0$ this shows that $\beta \neq 0$, and we conclude that
\begin{align*}
g(t) &= \frac{\beta(\beta t-\alpha)}{t-\gamma} =
\beta
\frac{(\alpha-\beta\gamma)(\beta t-\alpha)}{(\alpha-\beta\gamma)(t-\gamma)}
=
\beta 
\frac{(\alpha\gamma-\beta)\beta-(\alpha-\beta\gamma)\alpha}
{(\alpha\gamma-\beta)-(\alpha-\beta\gamma)\gamma}\\
&=
\beta \frac{2\alpha\beta\gamma-\alpha^{2}-\beta^{2}}
{\beta(\gamma^{2}-1)}
=\frac{\alpha^{2}+\beta^{2}-2\alpha\beta\gamma}{1-\gamma^{2}}
= \frac{(\alpha-\beta)^{2}+2\alpha\beta(1-\gamma)}{(1+\gamma)(1-\gamma)}\\
&\leq \frac{(\alpha-\beta)^{2}}{1-\gamma}+2|\alpha\beta|
\leq \frac{(\alpha-\beta)^{2}}{1-\gamma} +2\max\left(\alpha^{2},\beta^{2}\right)
\leq 3\max\left(\alpha^{2},\beta^{2},\frac{(\alpha-\beta)^{2}}{1-\gamma}\right).
\end{align*}

\end{proof}

\subsection{Proof of Proposition~\ref{thm:parametrization_functions_cross_dipoles}}\label{proof:cor_parametrization_functions_cross_dipoles}

Equalities \eqref{eq:bound_gammma_m}, \eqref{eq:bound_gammma_d}, \eqref{eq:bound_gammma_mm}, \eqref{eq:bound_gammma_md} and ~\eqref{eq:bound_gammma_dd} are straightforward applications of the following lemma.

\begin{lemma}\label{lemma:existence_of_theta_l} 
Under the assumptions and notations of Proposition~\ref{thm:parametrization_functions_cross_dipoles}, there exist sets $\Theta_{\mathrm{md}} \subset \Theta_{\mathrm{d}} \times \Theta_{\mathrm{mm}}$ and $\Theta_{\mathrm{dd}} \subset  \Theta_{\mathrm{d}} \times \Theta_{\mathrm{d}} \times \Theta_{\mathrm{mm}}$ such that, 
\begin{enumerate}
\item \label{it:ProofItemM} For every $\iota \in \mathfrak{M}$, we have $1 - \|\mathcal{A}\iota\|_{2}^2 = 1-\Psi_{\mathrm{m}}(\bm{\Omega})$, for every $\bm{\Omega}$ and $\mathcal{A}=\mathcal{A}_{\bm{\Omega}}$.
\item \label{it:ProofItemD} For every $\iota \in \hat{\mathfrak{D}}$, there is $x \in \Theta_{\mathrm{d}}$ such that $1 - \|\mathcal{A}\iota\|_{2}^2 = 1-\Psi_{\mathrm{d}}(x|\bm{\Omega})$ for every $\bm{\Omega}$ and $\mathcal{A}=\mathcal{A}_{\bm{\Omega}}$; and vice-versa \rvs{for every $x \in \Theta_{\mathrm{d}}$, there is $\iota \in \hat{\mathfrak{D}}$ such that $1 - \|\mathcal{A}\iota\|_{2}^2 = 1-\Psi_{\mathrm{d}}(x|\bm{\Omega})$ for every $\bm{\Omega}$}.
\item \label{it:ProofItemMM} For $(\iota, \iota') \in \mathfrak{M}_{\neq}^2$, there are $s \in \{-1,1\}$, $y \in \Theta_{\mathrm{mm}}$ s.t. $\mathfrak{Re}\langle \mathcal{A} \iota, \mathcal{A} \iota' \rangle = s\Psi_{\mathrm{mm}}(y|\bm{\Omega})$ for every $\bm{\Omega}$ and $\mathcal{A}=\mathcal{A}_{\bm{\Omega}}$, and vice-versa.
\item \label{it:ProofItemMD} For $(\iota, \iota') \in \mathfrak{M}\times \hat{\mathfrak{D}}_{\neq}$ there are $s \in \{-1,1\}$, $(x,y) \in \Theta_{\mathrm{md}}$ s.t. $\mathfrak{Re}\langle \mathcal{A} \iota, \mathcal{A} \iota' \rangle = s\Psi_{\mathrm{md}}(x, y|\bm{\Omega})$ for every $\bm{\Omega}$ and $\mathcal{A}=\mathcal{A}_{\bm{\Omega}}$, and vice-versa.
\item \label{it:ProofItemDD} For $(\iota, \iota') \in \hat{\mathfrak{D}}_{\neq}^2$ there are $s \in \{-1,1\}$, $(x,x',y) \in \Theta_{\mathrm{dd}}$ s.t. $\mathfrak{Re} \langle \mathcal{A} \iota, \mathcal{A} \iota' \rangle = s\Psi_{\mathrm{dd}}(x,x', y|\bm{\Omega})$, for every $\bm{\Omega}$ and $\mathcal{A}=\mathcal{A}_{\bm{\Omega}}$, and vice-versa.
\end{enumerate}
\end{lemma}
\paragraph{Proof of \autoref{it:ProofItemM}:}
Proposition~\ref{prop:param_delta} yields
$\|\mathcal{A}\iota\|_{2}^2 = \frac{1}{m} \sum_{j =1}^{m} \psi(\omega_j)$
 for every $\iota \in \mathfrak{M}$. \qed
\paragraph{Proof of  \autoref{it:ProofItemD}:}
Proposition~\ref{prop:param_delta} yields that for every $\iota \in \hat{\mathfrak{D}}$ there is $x \in \Theta_{\mathrm{d}}$ such that
$1- \|\mathcal{A}\iota\|_{2}^2 = 1- \frac{1}{m} \sum_{j =1}^{m} \psi(\omega_j) \frac{  1-\cos \langle \omega_j, x \rangle }{1-\bar{\kappa}(x)} = 1-\frac{1}{m}\sum_{j =1}^{m} \psi(\omega_{j}) f_{d}(x|\omega_j)$, and vice-versa, where we used that for $t \in \mathbb{R}$ we have $1-\cos t = 2 \sin^{2}(t/2)$. \qed

The remaining items use the second part of Proposition~\ref{prop:param_delta} which gives a generic formula for $\mathfrak{Re} \langle \mathcal{A}\iota, \mathcal{A}\iota' \rangle$: when $(\iota,\iota') \in \mathfrak{D}_{\neq}^2$ there are
$s,s' \in \{-1,1 \}$, $\alpha, \alpha' \in [0,1]$, $\theta_{1}, \theta_{2}, \theta_{1}', \theta_{2}' \in \Theta$ satisfying \eqref{eq:param_delta_1separateddipoles}-\eqref{eq:cross_dipole_expression_using_alphas_and_x_proof} with 
$x:= \theta_{1}-\theta_2 \in \Theta_{\mathrm{d}}$, $x':= \theta_{1}'-\theta_2' \in \Theta_{\mathrm{d}}$, and $a,b,c$ and $d$ are defined by \eqref{eq:cross_dipole_expression_using_alphas_and_x_paramsabcd}.
We will also use that given that $\mathcal{A}$ is a $(\kappa,w)$-FF sketching operator (cf Definition~\ref{def:RFFsketching}) and $\pi_{\theta}$ is obtained by translation of $\pi_{0}$ we have
\(
\mathfrak{Re} \langle \mathcal{A}\pi_{\theta_{i}}, \mathcal{A}\pi_{\theta'_{j}} \rangle = 
\tfrac{1}{m} \sum_{j=1}^{m} |\langle \pi_{0},\phi_{\omega_{j}} \rangle|^{2} \cos \langle \omega_{j},\theta_{i}-\theta'_{j}\rangle,
\)
and that $y := \theta_1 - \theta_1'$ satisfies $y \in \Theta-\Theta$ and $\|y\| = \|\theta_{1}-\theta_{1}'\|   \geq 1$ by \eqref{eq:param_delta_1separateddipoles}, thus $y \in \Theta_{\mathrm{mm}}$.

\paragraph{Proof of \autoref{it:ProofItemMM}}

When $(\iota,\iota') \in \mathfrak{M}_{\neq}^2$, we have $\alpha  = \alpha'= 0$ 
hence~\eqref{eq:cross_dipole_expression_using_alphas_and_x_proof}-\eqref{eq:cross_dipole_expression_using_alphas_and_x_paramsabcd} yield 
\begin{align*}
\mathfrak{Re}\langle \mathcal{A}\iota, \mathcal{A}\iota' \rangle & = ss' a   = \frac{ss'}{\|\pi_0\|_{\kappa}^2}\frac{1}{m}\sum\limits_{j =1}^{m} |\langle \pi_{0}, \phi_{\omega_{j}} \rangle|^2 \cos \big(\langle \omega_{j}, \theta_1 - \theta_1' \rangle \big)  = \frac{ss'}{m}\sum\limits_{j =1}^{m} \psi(\omega_{j}) \cos\big( \langle\omega_{j}, y \rangle \big).
\end{align*}
Vice-versa for such $s,s',y$ it is easy to exhibit $(\iota,\iota') \in \mathfrak{M}_{\neq}^2$ satisfying the same expression.\qed

\paragraph{Proof of  \autoref{it:ProofItemMD}}
When $(\iota, \iota') \in \mathfrak{M}\times \hat{\mathfrak{D}}_{\neq}$, $\alpha = 0$ and $\alpha' =1$, which similarly yields
\begin{align*}
\mathfrak{Re}\langle \mathcal{A}\iota, \mathcal{A}\iota' \rangle & = ss' \frac{a - b }{ \sqrt{2(1-\bar{\kappa}(x'))}}
 = \frac{ss'}{m} \sum\limits_{j = 1}^{m} \psi(\omega_{j}) \frac{\cos \big(\langle \omega_j, \theta_1 - \theta_1'\rangle \big) - \cos \big(\langle \omega_j, \theta_1 - \theta_2' \rangle \big) }{ \sqrt{2(1-\bar{\kappa}(x'))}}.
 \end{align*}
 Now, using the identity $ \cos(u) - \cos(v) = -2\sin((u-v)/2)\sin((u+v)/2)$, we get
 \begin{align*}
 \cos \big(\langle \omega_j, \theta_1 - \theta_1'\rangle \big) - \cos \big(\langle \omega_j, \theta_1 - \theta_2' \rangle \big) 
& = \cos \big(\langle \omega_j, y\rangle \big) - \cos \big(\langle \omega_j, y + x' \rangle \big) \\
& =  2
\sin \big(\langle \omega_j, x'\rangle/2 \big)\sin \big(\langle \omega_j, y+x'/2\rangle \big).
\end{align*}
Thus 
\begin{equation}\label{eq:identity_cross_iota_md}
  \mathfrak{Re}\langle \mathcal{A}\iota, \mathcal{A}\iota' \rangle = ss'\Psi_{\mathrm{md}}((x',y)|\bm{\Omega}),
\end{equation}
where $x' = \theta'_{1}-\theta'_{2}$ and $y = \theta_{1}-\theta'_{1}$ satisfy $x' \in \Theta_{\mathrm{d}}$ and $y  \in \Theta_{\mathrm{mm}}$.
We define $\Theta_{\mathrm{md}}$ as the set of all couples $(x',y) \in \Theta_{\mathrm{d}} \times \Theta_{\mathrm{mm}}$ that satisfy~\eqref{eq:identity_cross_iota_md} for some $(\iota,\iota') \in (\mathfrak{M} \times \hat{\mathfrak{D}})_{\neq}$ and $s,s' \in \{-1,1\}$.

\paragraph{Proof of  \autoref{it:ProofItemDD}}
When $(\iota, \iota') \in \hat{\mathfrak{D}}_{\neq}^{2}$, $\alpha = 1$ and $\alpha' =1$, 
hence~\eqref{eq:cross_dipole_expression_using_alphas_and_x_proof}-\eqref{eq:cross_dipole_expression_using_alphas_and_x_paramsabcd} similarly yield
\begin{align*}
\mathfrak{Re}\langle \mathcal{A}\iota, \mathcal{A}\iota' \rangle & = ss'\frac{a  - b  - c  + d}{\sqrt{2(1-\bar{\kappa}(x))} \sqrt{ 2(1-\bar{\kappa}(x'))}} \nonumber \\
& = \frac{ss'}{m} \sum\limits_{j = 1}^{m} \psi(\omega_{j}) \frac{\cos \big( \langle \omega_j, \theta_1 - \theta_1' \rangle \big) - \cos\big( \langle \omega_j, \theta_1 - \theta_2' \rangle \big) - \cos\big( \langle \omega_j, \theta_2 - \theta_1' \rangle \big) + \cos\big( \langle \omega_j, \theta_2 - \theta_2' \rangle \big)}{ \sqrt{2(1-\bar{\kappa}(x))}\sqrt{2(1-\bar{\kappa}(x'))}}.
\end{align*}
Since
$\cos(u)-\cos(v) = -2\sin(\frac{u-v}{2})\sin(\frac{u+v}{2})$ and
$\sin(u)-\sin(v) = 2\sin(\frac{u-v}{2})\cos(\frac{u+v}{2})$, and denoting $y:= \theta_{1} - \theta_{1}'$, $x:= \theta_{1}-\theta_2$, and $x':= \theta_{1}' - \theta_{2}'$,
we get $x,x' \in \Theta_{\mathrm{d}}$, $y  \in \Theta_{\mathrm{mm}}$ and
\begin{align*}
\cos \big( \langle \omega_j, \theta_1 - \theta_1' \rangle \big) &- \cos\big( \langle \omega_j, \theta_1 - \theta_2' \rangle \big) - \cos\big( \langle \omega_j, \theta_2 - \theta_1' \rangle \big) + \cos\big( \langle \omega_j, \theta_2 - \theta_2' \rangle \big)\\
& = \cos \big( \langle \omega_j, y \rangle \big) - \cos\big( \langle \omega_j, y+x' \rangle \big) - \cos\big( \langle \omega_j, y-x \rangle \big) + \cos\big( \langle \omega_j, y-x+x' \rangle \big)\\
& = 2\sin \big(\langle \omega_j, x'/2 \rangle \big)\sin\big(\langle \omega_j, y+x'/2 \rangle \big) -2\sin \big( \langle\omega_j,x'/2 \rangle\big)\sin\big(\langle \omega_j, y-x+x'/2 \rangle  \big)\\
& = 2 \sin\big(\langle \omega_j, x'/2 \rangle \big)\Big(\sin\big( \langle \omega_j, y+x'/2 \rangle \big) -\sin\big( \langle \omega_j, y-x+x'/2 \rangle \big)\Big)\\
& = 4 \sin\big(\langle \omega_j, x/2 \big)\sin\big(\langle\omega_j, x'/2\rangle \big)\cos\big( \langle \omega_j, y+x'/2-x/2 \rangle \big).
\end{align*}
Thus 
\begin{equation}\label{eq:identity_cross_iota_dd}
  \mathfrak{Re}\langle \mathcal{A}\iota, \mathcal{A}\iota' \rangle = ss'\Psi_{\mathrm{dd}}((x,x',y)|\bm{\Omega}).
\end{equation}
We define $\Theta_{\mathrm{dd}}$ as the set of all triplets $(x,x',y)  \in \Theta_{\mathrm{d}}\times \Theta_{\mathrm{d}} \times \Theta_{\mathrm{mm}}$ that satisfy~\eqref{eq:identity_cross_iota_dd} for some $(\iota,\iota') \in \hat{\mathfrak{D}}_{\neq}^2$ and $s,s' \in \{-1,1\}$.

\subsection{Proof of Theorem~\ref{thm:delta_bounds}}
\label{sec:proofthemdeltabounds}

Consider $\ell \in \{\mathrm{d},\mathrm{md},\mathrm{mm},\mathrm{dd}\}$ and $z,z' \in \Theta_{\ell}$, and denote 
\[
\Delta_{\ell,\bm{\Omega}}(z,z') := \left|\Psi_{\ell}(z|\bm{\Omega})-\Psi_{\ell}(z'|\bm{\Omega})\right|.
\]
The result will follow if we exhibit $z_{1},\ldots,z_{T}$, $2 \leq T \leq 6$ such that
$z_{1}=z$, $z_{T}=z'$ such that
\begin{align}
\max_{1 \leq t \leq T-1}  \Delta_{\ell}(z_{t},z_{t+1})  \leq \Delta_{\ell}(z,z')  \label{eq:ReverseTriangle}
\end{align}
and if we can find smooth functions $u \in [0,1] \mapsto f_{\ell,t}(u|\omega)$ such that for every $\omega \in \mathbb{R}^{d}$ and every $t$
\begin{align}
f_{\ell,t}(0|\omega) = f_{\ell}(z_{t}|\omega), & \quad f_{\ell,t}(1|\omega) = f_{\ell}(z_{t+1}|\omega)\, ,\label{eq:FJKEndpoints}\\ 
\label{eq:BoundFJKDerivative}
\sup_{0<u < 1} |f'_{\ell,t}(u|\omega)| &\leq G(\omega) \Delta_{\ell}(z_{t},z_{t+1})\\
\text{where}\quad G(\omega) &:= C_{\kappa} \sum_{i=1}^{3}\|\omega\|_{a,\star}^{i}. \nonumber
\end{align}
Indeed, this will imply by the triangle inequality, the mean value theorem, and~\eqref{eq:M_Omega_def} that
\begin{align*}
\Delta_{\ell,\bm{\Omega}}(z,z') 
& = 
\Delta_{\ell,\bm{\Omega}}(z_{1},z_{T}) 
\leq \sum_{t=1}^{T-1} \Delta_{\ell,\bm{\Omega}}(z_{t},z_{t+1}) 
= \sum_{t=1}^{T-1} \left|
\frac{1}{m}\sum_{j=1}^{m}\psi_{m}(\omega_{j})  [f_{\ell,t}(1|\omega_{j})-f_{\ell,t}(0|\omega_{j})]
\right|
\\
& \leq 
\sum_{t=1}^{T-1} 
\frac{1}{m}\sum_{j=1}^{m}\psi_{m}(\omega_{j})  \left| [f_{\ell,t}(1|\omega_{j})-f_{\ell,t}(0|\omega_{j})]
\right|
 \leq 
 \frac{1}{m}
\sum_{t=1}^{T-1} 
\sum_{j=1}^{m}\psi_{m}(\omega_{j})  \sup_{0<u < 1} |f'_{\ell,t}(u|\omega_{j})|
\\
&\leq 
 \frac{1}{m}
\sum_{t=1}^{T-1} 
\sum_{j=1}^{m}\psi_{m}(\omega_{j})  G(\omega_{j}) \Delta_{\ell}(z_{t},z_{t+1})
=
\left( \frac{1}{m} \sum_{j=1}^{m}\psi_{m}(\omega_{j})  G(\omega_{j})\right) \cdot 
\sum_{t=1}^{T-1}  \Delta_{\ell}(z_{t},z_{t+1})
\\
& \leq
\left( \frac{1}{m} \sum_{j=1}^{m}\psi_{m}(\omega_{j})  G(\omega_{j})\right) \cdot 
6  \Delta_{\ell}(z,z')
= 6 \Psi_{0}(\bm{\Omega}) \cdot C_{\kappa} \cdot\ 
 \Delta_{\ell}(z,z').
\end{align*}
\subsubsection{Construction of $z_{t}$, $1 \leq t \leq T$}
First we focus on the construction of $z_{1},\ldots, z_{T}$ satifying the inequality~\eqref{eq:ReverseTriangle}:
\begin{itemize}
\item When $\ell=\mathrm{d}$, we have $z=x$ and $z'=x'$ where $x, x' \in \Theta_{\mathrm{d}}$ (cf~\eqref{eq:bound_gammma_d}). Observe that  $r:= \|x\|_{a}, r' := \|x'\|_{a}$ satisfy $r,r'>0$, and that $n :=x/r$, $n':= x'/r'$ satisfy $\|n\|_{a} = \|n'\|_{a}=1$.
We set $z_{1} = x$, $z_{2} = \tilde{x}$, $z_{3} = x'$ where $\tilde{x} := \|x'\|_{a} \cdot (x/\|x\|_{a})$ satisfies $\|\tilde{x}\|_{a} = \|x'\|_{a}$. 
Given the definition~\eqref{eq:DefMetricD} we have $\Delta_{\ell}(z_{1},z_{2}) = |r-r'| = |\|x\|_{a}-\|x'\|_{a}| \leq \Delta_{\ell}(x,x') = \Delta_{\ell}(z,z')$ and similarly $\Delta_{\ell}(z_{2},z_{3}) = \|n-n'\|_{a} \leq \Delta_{\ell}(z,z')$ as claimed.
\item When $\ell = \mathrm{mm}$, $z=y$, $z'=y'$, where $y,y' \in \Theta_{\mathrm{mm}}$ and we simply set $z_{1}=z$, $z_{2}=z'$.
\item When $\ell = \mathrm{md}$, $z = (x,y)$ and $z' = (x',y')$ where $x,x' \in \Theta_{\mathrm{d}}$, $y,y' \in \Theta_{\mathrm{mm}}$. Setting $\tilde{x}$ as above we define $z_{1} = z = (x,y)$, $z_{2} = (x,y')$, $z_{3} = (\tilde{x},y')$ and $z_{4} = (x',y')$. It is not difficult to check that the inequality~\eqref{eq:ReverseTriangle} holds given the definition~\eqref{eq:DefMetricMD} of $\Delta_{\mathrm{md}}$.
\item Finally, when $\ell = \mathrm{dd}$, $z=(x_1,x_2,y)$ and $z'=(x_1',x_2',y')$ where $x_{i} \in \Theta_{d}$ and $y,y' \in \Theta_{\mathrm{mm}}$. It is easy to check~\eqref{eq:ReverseTriangle} with $z_{1} = (x_{1},x_{2},y)$, $z_{2} = (x_{1},x_{2},y')$, $z_{3} = (\tilde{x}_{1},x_{2},y')$, $z_{4} = (x_{1}',x_{2},y')$, $z_{5} = (x_{1}',\tilde{x}_{2},y')$, $z_{6} = (x_{1}',x_{2}',y')$ given the definition~\eqref{eq:DefMetricDD} of $\Delta_{\mathrm{dd}}$.
\end{itemize}
To complete the proof of Theorem~\ref{thm:delta_bounds} we now exhibit $f_{\ell,t}$ satisfying
\eqref{eq:FJKEndpoints}-\eqref{eq:BoundFJKDerivative} by treating each case $\ell= \mathrm{d}, \ell=\mathrm{mm}, \ell=\mathrm{md}, \ell =\mathrm{dd}$ and pair $(z_{t},z_{t+1})$. 
First we build the functions and show that they satisfy~\eqref{eq:FJKEndpoints}. Then observing their common structure we establish the bound~\eqref{eq:BoundFJKDerivative}.

\subsubsection{Construction of $f_{\ell,t}$ satisfying~\eqref{eq:FJKEndpoints}}

\paragraph{The case $\ell=\mathrm{d}$.} For $u \in (0,1)$, we define $\bar{r}(u):= r+u(r'-r)$ and $\bar{n}(u):= n+u(n'-n)$, which satisfies $\|\bar{n}(u)\|_{a} = \|(1-u)n+un'\|_{a} \leq (1-u)\|n\|_{a}+u\|n'\|_{a} = 1$, and we set \begin{align}
f_{\mathrm{d},1}(u|\omega) & 
 := 2
\frac{\sin^{2}(\langle \omega, \bar{r}(u) n \rangle/2)}{1-\tilde{\kappa}(\bar{r}(u))}
\stackrel{\eqref{eq:DefAlphaFn} }{=}
2\alpha^{2}(\bar{r}(u)) \frac{\sin^{2}(\bar{r}(u) \langle \omega, n \rangle/2)}{(\bar{r}(u))^{2}} \, ,
\label{eq:DefFd1}\\
f_{\mathrm{d},2}(u|\omega) & 
:=
\frac{
2\sin^{2}(r'\langle \omega, \bar{n}(u) \rangle/2)}{1-\tilde{\kappa}(r')}
\stackrel{\eqref{eq:DefAlphaFn} }{=}
2\alpha^{2}(r') \frac{\sin^{2}(r'\langle \omega, \bar{n}(u) \rangle/2)}{(r')^{2}} 
\label{eq:DefFd2}\end{align}
Property~\eqref{eq:FJKEndpoints} follows from~\eqref{eq:DefPsiD} and \eqref{eq:DefRadialKernel} since $z_{1} = x = r \cdot n = \bar{r}(0)\cdot n$, $z_{2} = \tilde{x} = r' \cdot n = \bar{r}(1) \cdot n = r' \bar{n}(0)$, $z_{3} = x' 
= r' \cdot n' = r' \cdot \bar{n}(1)$ and $\|n\|_{a}= \|n'\|_{a}=1$ so that $\bar{\kappa}(z_{2}) = \bar{\kappa}(z_{3}) = \tilde{\kappa}(r')$.

\paragraph{The case $\ell=\mathrm{mm}$.}
For $u \in (0,1)$, we define $\bar{y}(u):= y+u(y'-y)$ and
\begin{align}
f_{\mathrm{mm},1}(u|\omega) 
& :=  \cos( \langle \omega, \bar{y}(u) \rangle).\label{eq:DefFmm1}
\end{align}
Property~\eqref{eq:FJKEndpoints} follows by~\eqref{eq:DefPsiMM} since $z_{1} = y=\bar{y}(0)$ and $z_{2}=y' = \bar{y}(1)$.

\paragraph{The case $\ell = md$.}
For any $x,y,\omega$, by~\eqref{eq:DefPsiMD},~\eqref{eq:DefRadialKernel},~\eqref{eq:DefAlphaFn}, since
$2\sin v \sin w = \cos(v-w)-\cos(v+w)$ 
\begin{align*}
f_{\mathrm{md}}(x,y|\omega) 
= \sqrt{2}\frac{\sin\big(\langle \omega, x \rangle/2 \big)\sin\big(\langle \omega, y +x/2 \rangle \big) 
}{\sqrt{1-\tilde{\kappa}(\|x\|_{a})}}
&=
\sqrt{2}\alpha(\|x\|_{a})  \frac{ \sin\big(\langle \omega, x \rangle/2 \big)\sin\big(\langle \omega, y +x/2 \rangle \big) }{\|x\|_{a}} \\
& = 
\frac{1}{\sqrt{2}}\alpha(\|x\|_{a}) \frac{\cos(\langle \omega,y\rangle) -\cos(\langle \omega, y\rangle+\langle \omega, x\rangle)}{\|x\|_{a}}.
\end{align*}
Since $z_{1} = (x,y)=(x,\bar{y}(0))$, 
$z_{2}=(x,y') = (x,\bar{y}(1)) = (\bar{r}(0) \cdot n,y')$, $z_{3}=(\tilde{x},y') = (\bar{r}(1)\cdot n,y') = (r' \cdot \bar{n}(0),y')$, $z_{4} 
= (x',y') = (r' \cdot \bar{n}(1),y')$, Property~\eqref{eq:FJKEndpoints} holds with
\begin{align}
f_{\mathrm{md},1}(u|\omega) &
:=
\sqrt{2}\alpha(r) \frac{\sin\big(r\langle \omega, n \rangle/2 \big)\sin\big(\langle \omega, \bar{y}(u) +x/2 \rangle \big)}{r}\, ,\label{eq:DefFmd1}\\
f_{\mathrm{md},2}(u|\omega) & 
:= 
\sqrt{2}\alpha(\bar{r}(u)) \frac{\sin\big(\bar{r}(u) \langle \omega, n \rangle/2\big)\sin\big(\langle \omega,y'+\bar{r}(u) \cdot n/2 \rangle\big)}{\bar{r}(u)} \, ,\label{eq:DefFmd2}\\
 f_{\mathrm{md},3}(u|\omega) & 
:= 
\frac{1}{\sqrt{2}}\alpha(r') \frac{\cos(\langle \omega,y'\rangle) -\cos(\langle \omega, y'\rangle+r' \langle \omega, \bar{n}(u)\rangle)}{r'}
\, .\label{eq:DefFmd3}
\end{align}

\paragraph{The case $\ell = \mathrm{dd}$.}
For any $x_{1},x_{2},y,\omega$,~\eqref{eq:DefPsiDD} yields using~\eqref{eq:DefRadialKernel}-\eqref{eq:DefAlphaFn}
\begin{align*}
f_{\mathrm{dd}}((x_{1},x_{2},y)|\omega) 
&= 2\frac{\sin\big(\langle \omega, x_{1} \rangle/2\big)\sin\big(\langle \omega, x_{2} \rangle/2\big)\cos\big(\langle \omega, y + x_{2}/2-  x_{1}/2 \rangle \big)}{\sqrt{1-\tilde{\kappa}(\|x_{1}\|_a)}\sqrt{1-\tilde{\kappa}(\|x_{2}\|_a)}}\\
& = 2\alpha(\|x_{1}\|_{a})\alpha(\|x_{2}\|_{a})
 \frac{\sin\big(\langle \omega, x_{1} \rangle/2\big)\sin\big(\langle \omega, x_{2} \rangle/2\big)\cos\big(\langle \omega, y + x_{2}/2-  x_{1}/2 \rangle \big)}{\|x_{1}\|_{a}\|x_{2}\|_{a}} 
\end{align*}
Reasoning as above establishes Property~\eqref{eq:FJKEndpoints} holds with $z_{1} = (x_{1},x_{2},y)$, $z_{2} = (x_{1},x_{2},y')$, $z_{3} = (\tilde{x}_{1},x_{2},y')$, $z_{4} = (x_{1}',x_{2},y')$, $z_{5} = (x_{1}',\tilde{x}_{2},y')$, $z_{6} = (x_{1}',x_{2}',y')$,  $\bar{y}(u) :=  y + u(y'-y)$, where $\tilde{x}_{i}, r_{i},r'_{i}, \bar{r}_{i}(u)$, $n_{i},n'_{i},\bar{n}_{i}(u)$, $i \in \{1,2\}$ were defined in the same way as in the case $\ell = \mathrm{d}$, and
\begin{align}
f_{\mathrm{dd},1}(u|\omega) & := 2\alpha(r_{1})\alpha(r_{2})
\frac{\sin\big(r_{1}\langle \omega, n_{1} \rangle/2\big)\sin\big(r_{2}\langle \omega, n_{2} \rangle/2\big)\cos\big(\langle \omega, \bar{y}(u) + r_{2}n_{2}/2-  r_{1}n_{1}/2 \rangle \big)}{r_{1}r_{2}}\label{eq:DefFdd1}\\
f_{\mathrm{dd},2}(u|\omega) & :=  2\alpha(\bar{r}_{1}(u))\alpha(r_{2})
\frac{\sin\big(\bar{r}_{1}(u)\langle \omega, n_{1} \rangle/2\big)\sin\big(r_{2}\langle \omega, n_{2} \rangle/2\big)\cos\big(\langle \omega, y' + r_{2}n_{2}/2-  \bar{r}_{1}(u)\cdot n_{1}/2 \rangle \big)}{\bar{r}_{1}(u) r_{2}}\label{eq:DefFdd2}\\
f_{\mathrm{dd},3}(u|\omega) & :=  2\alpha(r'_{1})\alpha(r_{2})
\frac{\sin\big(r'_{1}\langle \omega, \bar{n}_{1}(u) \rangle/2\big)\sin\big(r_{2}\langle \omega, n_{2} \rangle/2\big)\cos\big(\langle \omega, y' + r_{2}n_{2}/2-  r'_{1} \cdot \bar{n}_{1}(u)/2 \rangle \big)}{r'_{1}r_{2}}\label{eq:DefFdd3}\\
f_{\mathrm{dd},4}(u|\omega) & :=  2\alpha(r'_{1})\alpha(\bar{r}_{2}(u))
\frac{\sin\big(r'_{1}\langle \omega, n'_{1} \rangle/2\big)\sin\big(\bar{r}_{2}(u)\langle \omega, n_{2} \rangle/2\big)\cos\big(\langle \omega, y' + \bar{r}_{2}(u)n_{2}/2-  r'_{1}n'_{1}/2 \rangle \big)}{r'_{1}\bar{r}_{2}(u)}\label{eq:DefFdd4}\\
f_{\mathrm{dd},5}(u|\omega) & :=  2\alpha(r'_{1})\alpha(r'_{2})
\frac{\sin\big(r'_{1}\langle \omega, n'_{1} \rangle/2\big)\sin\big(r'_{2}\langle \omega, \bar{n}_{2}(u) \rangle/2\big)\cos\big(\langle \omega, y' + r'_{2}\bar{n}_{2}(u)/2-  r'_{1}n'_{1}/2 \rangle \big)}{r'_{1}r'_{2}}\label{eq:DefFdd5}
\end{align}

\subsubsection{Proof of the bound~\eqref{eq:BoundFJKDerivative}}
To continue we gather a few observations. First, since $\sinc(t) := \sin(t)/t = \int_{0}^{1} \cos(xt)dx$ for every $t \neq 0$ (and $\sinc(0)=1$) we have $\sinc'(t) = \int_{0}^{1} -x\sin(xt)dx$ hence $\max(|\sinc(t)|,|\sinc'(t)|) \leq 1$. Now, by definition of the dual norm $\|\cdot\|_{a,\star}$, for every $\omega, v \in \mathbb{R}^{d}$ we have $|\langle \omega,v\rangle| \leq \|\omega\|_{a,\star} \|v\|_{a}$.
Thus, by definition~\eqref{eq:conditions_on_alpha} of $C_{\kappa}$ for every $0<t \leq R$ we have 
\begin{equation}\label{eq:SincBound}
\forall v \in \mathbb{R}^{d},\quad
|\alpha(t) \sin(t\langle \omega,v\rangle/2)/t| = |\alpha(t)\tfrac{\langle \omega,v\rangle}{2} \sinc(t\langle \omega,v\rangle/2)| \leq \frac{\sqrt{C_{\kappa}}}{2} \|\omega\|_{a,\star} \|v\|_{a}.
\end{equation}
Recalling that $\bar{y}(u) := y+u(y'-y)$, $\bar{n}(u) := n+u(n'-n)$, $\bar{r}(u) := r+u(r'-r)$, and $\omega \in \mathbb{R}^{d}$ we now bound the following auxiliary functions and their derivatives, with arbitrary $\phi \in \mathbb{R}$
\begin{align*}
g_{0,\phi}(u) := \cos(\langle \omega, \bar{y}(u)\rangle+\phi)\, ,\ 
g_{1}(u) := \alpha(\bar{r}(u)) \frac{\sin(\bar{r}(u) \langle \omega, n\rangle/2)}{\bar{r}(u)}\, ,\ 
g_{2}(u) := \alpha(r') \frac{\sin(r' \langle \omega,\bar{n}(u)\rangle/2)}{r'}\, .
\end{align*}
Since $0<r' \leq R$ we have
\begin{align*}
|g'_{0,\phi}(u)| &= |\sin(\langle \omega, \bar{y}(u)\rangle+\phi) \cdot \langle \omega,y'-y\rangle | \leq  |\langle \omega,y'-y\rangle| \leq \|\omega\|_{a,\star}\cdot  \|y'-y\|_{a}\, ,\\
|g'_{2}(u)| & = |\alpha(r') \cos(r'\langle \omega,\bar{n}(u)\rangle/2) \cdot  \tfrac{\langle \omega,n'-n\rangle}{2} |
\leq \frac{\sqrt{C_{\kappa}} \|\omega\|_{a,\star}}{2}  \cdot \|n'-n\|_{a}\, .
\end{align*}
As $g_{1}(u) = \alpha(\bar{r}(u)) \frac{\langle \omega,n\rangle}{2} \sinc(\bar{r}(u) \langle \omega, n\rangle/2)$, $\|n\|_{a}=1$, and $\bar{r}(u) \leq \max (r,r') \leq R$ we get 
\begin{align*}
|g'_{1}(u)| &=  \Big|\alpha'(\bar{r}(u)) \sinc(\bar{r}(u) \langle \omega, n\rangle/2) + 
\alpha(\bar{r}(u)) \frac{\langle \omega, \cdot n\rangle}{2} \sinc'(\bar{r}(u) \langle \omega, n\rangle/2)\Big|
\cdot \Big|\frac{\langle \omega,n\rangle}{2} \cdot (r'-r)\Big|\\
&\leq \frac{\sqrt{C_{\kappa}}\|\omega\|_{a,\star}}{2} \cdot (1+\|\omega\|_{a,\star}/2)  \cdot |r'-r|
\end{align*}
Since $\|\bar{n}(u)\|_{a} \leq \max(\|n\|_{a},\|n'\|_{a}) = 1$ we also get using~\eqref{eq:SincBound}
\[
\max(|g_{1}(u)|,|g_{2}(u)|) \leq \frac{\sqrt{C_{\kappa}}}{2} \|\omega\|_{a,\star} \max( \|n\|_{a} ,\|\bar{n}(u)\|_{a} ) = \frac{\sqrt{C_{\kappa}}}{2} \|\omega\|_{a,\star}\, .
\]
We are now equipped to proceed.
\paragraph{The case $\ell=\mathrm{d}$.} 
Expressions~\eqref{eq:DefFd1}-\eqref{eq:DefFd2} yield $f_{\mathrm{d},1}(u|\omega) = 2 g_{1}^{2}(u)$, $f_{\mathrm{d},2}(u|\omega) = 2g_{2}^{2}(u)$. By~\eqref{eq:DefMetricD} and the choice of $z_{1}=x,z_{2}=\tilde{x},z_{3}=x'$ we have $\Delta_{d}(z_{1},z_{2}) = |r'-r|$,  $\Delta_{d}(z_{2},z_{3}) = \|n'-n\|_{a}$ and we obtain  the bound~\eqref{eq:BoundFJKDerivative} since
\begin{align*}|f'_{\mathrm{d},1}(u)|  &= |4g_{1}(u)g'_{1}(u)| \leq C_{\kappa} (\|\omega\|_{a,\star}^{2}+\|\omega\|_{a,\star}^{3}/2)  
\cdot |r'-r| \leq G(\omega) \cdot  \Delta_{d}(z_{1},z_{2})\, ,\\
|f'_{\mathrm{d},2}(u)|  &= |4g_{2}(u)g'_{2}(u)| \leq C_{\kappa} \|\omega\|_{a,\star}^{2} \cdot  \|n'-n\|_{a} \leq G(\omega) \cdot  \Delta_{d}(z_{2},z_{3})\, .
\end{align*}
\paragraph{The case $\ell=\mathrm{mm}$.}
By the expression~\eqref{eq:DefFmm1} we have $f_{\mathrm{mm},1} = g_{0}$. The bound~\eqref{eq:BoundFJKDerivative} follows from
\begin{align*}|f_{\mathrm{mm},1}'(u)| 
&=  |g'_{0,0}(u)| 
\leq \|  \omega_j \|_{a,\star} \|y'-y\|_{a}  \stackrel{\eqref{eq:conditions_on_alpha} \& \eqref{eq:DefMetricMM}}{\leq} 
G(\omega) \Delta_{\mathrm{mm}}(z_{1},z_{2}).
\end{align*}
\paragraph{The case $\ell = \mathrm{md}$.}
By the identity $\cos(\theta-\pi/2) = \sin(\theta)$ we have $f_{\mathrm{md},1}(u) \stackrel{\eqref{eq:DefFmd1}}{=} \sqrt{2}g_{1}(0) g_{0,\phi}(u)$ with $\phi :=\langle \omega,x\rangle/2 -\pi/2$, and by~\eqref{eq:DefFmd2} $f_{\mathrm{md},2}(u) = \sqrt{2} g_{1}(u) \sin(\psi(u))$ with $\psi(u) := \langle \omega, y' + \bar{r}(u)n/2\rangle$, and $\psi'(u) = \langle \omega,n\rangle (r'-r)/2$. Combining with~\eqref{eq:DefFmd3} we obtain the bound
~\eqref{eq:BoundFJKDerivative} since
\begin{align*}|f_{\mathrm{md},1}'(u)| 
&=  |\sqrt{2}g_{1}(0)g'_{0,\phi}(u)|  
\leq \frac{\sqrt{C_{\kappa}}}{\sqrt{2}} \|  \omega_j \|_{a,\star}^{2} \|y'-y\|_{a}  \stackrel{C_{\kappa} \geq 1 \& \eqref{eq:DefMetricMD}}{\leq} 
G(\omega) \Delta_{\mathrm{md}}(z_{1},z_{2})\, , \\
|f'_{\mathrm{md},2}(u)| & = \sqrt{2} |g'_{1}(u) \sin(\psi(u)) + g_{1}(u) \cos(\psi(u))\langle \omega,n\rangle(r'-r)/2 |\nonumber\\
& \leq \sqrt{2}\left(\frac{\sqrt{C_{\kappa}}\|\omega\|_{a,\star}}{2}(1+\|\omega\|_{a,\star}/2)|r'-r|+\frac{\sqrt{C_{\kappa}}\|\omega\|_{a,\star}}{2} \|\omega\|_{a,\star}|r'-r|/2\right)\\
& \leq 
\frac{C_{\kappa} \|\omega\|_{a,\star}}{\sqrt{2}} (1+\|\omega\|_{a,\star})
\cdot |r'-r|
\leq G(\omega) \Delta_{\mathrm{md}}(z_{2},z_{3})\, ,\\
|f'_{\mathrm{md},3}(u)|
& = \frac{1}{\sqrt{2}}|\alpha(r')\sin(\langle\omega, y'\rangle+r' \langle \omega, \bar{n}(u)\rangle) \cdot \langle\omega,n'-n\rangle|\nonumber\\
& \leq \frac{\sqrt{C_{\kappa}}\|\omega\|_{a,\star}}{\sqrt{2}} \cdot \|n'-n\|_{a}
\leq G(\omega) \Delta_{\mathrm{md}}(z_{3},z_{4})\, .
\end{align*}

\paragraph{The case $\ell = \mathrm{dd}$}
Denote $g_{i,j}$, $i,j=1,2$ the functions defined as $g_{i}$ with $r_{j},r'_{j}$, etc. instead of $r,r'$ etc.
By~\eqref{eq:DefFdd1}
we have $f_{\mathrm{dd},1}(u) = 2g_{1,1}(0)g_{1,2}(0) g_{0,\phi}(u)$ with $\phi := r_{2}n_{2}/2-r_{1}n_{1}/2$, and by~\eqref{eq:DefFdd2}, $f_{\mathrm{dd},2}(u) = 2 g_{1,1}(u)g_{1,2}(0)\cos(\psi(u))$ with $\psi_{2}(u) := \langle \omega,y'+r_{2}n_{2}/2\rangle-\bar{r}_{1}(u) \langle \omega, n_{1}\rangle/2$, $\psi_{2}'(u) = -\langle \omega, n_{1}\rangle (r'_{1}-r_{1})/2$ hence
\begin{align*}
|f'_{\mathrm{dd},1}(u)| & = |2g_{1,1}(0)g_{1,2}(0)g'_{0,\phi}(u)| \leq \frac{C_{\kappa}}{2}\|\omega\|_{a,\star}^{3} \cdot \|y'-y\|_{a} \leq G(\omega) \Delta_{\mathrm{dd}}(z_{1},z_{2})\, ,\\
|f'_{\mathrm{dd},2}(u)| & = |2g_{1,2}(0)| \cdot |g_{1,1}'(u)\cos(\psi(u))+g_{1,1}(u)\sin(\psi(u))\langle \omega, n_{1}\rangle (r'_{1}-r_{1})/2| \\
& \leq \sqrt{C_{\kappa}}\|\omega\|_{a,\star} \left(\frac{\sqrt{C_{\kappa}}\|\omega\|_{a,\star}}{2}(1+\|\omega\|_{a,\star}/2)|r'_{1}-r_{1}|+\frac{\sqrt{C_{\kappa}}\|\omega\|_{a,\star}}{2}\|\omega\|_{a,\star}|r'_{1}-r_{1}|/2\right)\\
& \leq \frac{C_{\kappa}}{2}\|\omega\|_{a,\star}^{2} 
\left(1+3\|\omega\|_{a,\star}/4\right) |r'_{1}-r_{1}|
 \leq G(\omega) \Delta_{\mathrm{dd}}(z_{2},z_{3})\, .
\end{align*}
Similarly by~\eqref{eq:DefFdd4} $f_{\mathrm{dd},4}(u) = 2 g_{1,1}(1)g_{1,2}(u)\cos(\psi_{4}(u))$ with $\psi_{4}(u) := 
\langle \omega, y' -  r'_{1}n'_{1}/2 \rangle+ \bar{r}_{2}(u)\langle \omega,n_{2}\rangle/2$ and the same reasoning yields the same bound. This establishes the bound~\eqref{eq:BoundFJKDerivative}  for $t=1,2,4$.

By~\eqref{eq:DefFdd3} and the identity $2\sin v \cos w = \sin(v+w)+\sin(v-w)$ we write 
\begin{align*}
f_{\mathrm{dd},3}(u) 
&= 
g_{1,2}(0) \frac{\alpha(r'_{1})}{r'_{1}} 
2\sin\big(r'_{1}\langle \omega, \bar{n}_{1}(u) \rangle/2\big)\cos\big(\langle \omega, y' + r_{2}n_{2}/2-  r'_{1} \cdot \bar{n}_{1}(u)/2 \rangle \big)\\
&= g_{1,2}(0) \frac{\alpha(r'_{1})}{r'_{1}} 
\Big(
\sin(\langle \omega,y'+r_{2}n_{2}/2\rangle) + 
\sin(\underbrace{r'_{1} \langle \omega,\bar{n}_{1}(u)\rangle -\langle \omega, y'+r_{2}n_{2}/2\rangle}_{\psi_{3}(u)})
\Big)\\
|f'_{\mathrm{dd},3}(u)|
& = |g_{1,2}(0) \frac{\alpha(r'_{1})}{r'_{1}} \cos(\psi_{3}(u))\psi_{3}'(u)| 
\leq \frac{\sqrt{C_{\kappa}} \|\omega\|_{a,\star} }{2}\frac{\sqrt{C_{\kappa}}}{r'_{1}} |\psi_{3}'(u)|
= \frac{C_{\kappa}\|\omega\|_{a,\star}}{2} |\langle \omega,n'_{1}-n_{1}\rangle|\\
& \leq \frac{C_{\kappa}\|\omega\|_{a,\star}^{2}}{2} \cdot \|n'_{1}-n_{1}\|_{a} \leq G(\omega) \Delta_{\mathrm{dd}}(z_{3},z_{4}).
\end{align*}
The same reasoning works from \eqref{eq:DefFdd5} for $t=5$. This establishes the bound~\eqref{eq:BoundFJKDerivative}  for $t=3,5$.

\subsubsection{Proof of Proposition~\ref{prop:covering_delta_metric}}
\label{pf:prop:covering_delta_metric}
We denote by $\mathbb{B}_{a}$ (resp. $\mathbb{S}_{a}$) the unit ball (resp. unit sphere) with respect to $\|\cdot\|_{a}$, $R_{d} := \sup_{x \in \Theta_{d}}\|x\|_{a}$, and $R_{\mathrm{mm}} := \sup_{x \in \Theta_{\mathrm{mm}}} \|x\|_{a}$. Observe that 
\[
\max(R_{d},R_{\mathrm{mm}}) = \sup_{x \in \Theta_{d} \cup \Theta_{\mathrm{mm}}} \|x\|_{a} \stackrel{\eqref{eq:bound_gammma_d}-\eqref{eq:bound_gammma_mm}}{=} \sup_{x \in \Theta-\Theta} \|x\|_{a} = \operatorname{diam}_{a}(\Theta) =: D.
\]

\paragraph{An upper bound of $\mathcal{N}_{d}(\tau)$.}
Denote $I := (0,R_{d}] \subseteq \mathbb{R}$. We will soon show that 
\begin{equation}\label{eq:CoverProd1}
\mathcal{N}_{d}(\tau) := \mathcal{N}(\Theta_{d}, \Delta_{d}, \tau)
\leq 
\mathcal{N}( I,|\cdot|,\tau/2) 
\times
\mathcal{N}\big(\mathbb{S}_{a},\|\cdot\|_{a},\tfrac{\tau}{2(R_{d}+1)}\big).
\end{equation}
By~Lemma A.1 in \citep{GrBlKeTr20}, since  $\mathbb{S}_{a} \subset \mathbb{B}_{a}$, we have
$\mathcal{N}\big(\mathbb{S}_{a},\|\cdot\|_{a},\tfrac{\tau}{2(R_{d}+1)}\big) \leq \mathcal{N}\big(\mathbb{B}_{a},\|\cdot\|_a,\tfrac{\tau}{4(R_{d}+1)}\big)$.
Moreover, by Lemma 5.7 in \citep{Wain19}, for every $\tau>0$ the inequality $\mathcal{N}\big(\mathbb{B}_{a},\|\cdot\|_{a},\frac{\tau}{(4(R_{d}+1)}\big) \leq  \big(1+8(R_{d}+1)/\tau\big)^d$ is valid.
Finally, since
$\mathcal{N}(I,|\cdot|,\tau/2) \leq 1+2R_{d}/\tau \leq 1+8(R_{d}+1)/\tau$ we obtain
\[
\mathcal{N}_{d}(\tau) \leq \big(1+8(R_{d}+1)/\tau\big)^{d+1}
\leq \big(1+8(D+1)/\tau\big)^{d+1}.
\]
We now establish~\eqref{eq:CoverProd1}. Denote $N_{1} := \mathcal{N}( I,|\cdot|,\tau/2)$, $N_{2} := \mathcal{N}(\mathbb{S}_{a},\|\cdot\|_{a},\tau/(2(R_{d}+1))) $, and consider $(r_i)_{i \in [N_{1}]}$ a covering of $I$ with respect to $|\cdot|$ at scale $\tau/2$ and $(s_i)_{i \in [N_{2}]}$ a covering of $\mathbb{S}_{a}$ with respect to $\|\cdot\|_{a}$ at scale $\tau/(2(R_{d}+1))$. We show that the family $(r_i \cdot s_j)_{(i,j) \in [N_{1}] \times [N_{2}]}$ is a covering of $\Theta_{d}$ with respect to the metric $\Delta_{d}$ at scale $\tau$.
For this, consider an arbitrary $x \in \Theta_{d} $ (recall the definition~\eqref{eq:bound_gammma_d}) and define $r := \|x\|_{a}$ and $n := x/r$. By definition of $R_{d}$ and $I$ we have $r \in I$, and $\|n\|_{a} \in \mathbb{S}_{a}$, hence there are $i \in [N_{1}]$, $j \in [N_{2}]$ such that $|r - r_{i}| \leq \tau/2$ and $\|n - s_{j}\|_{a} \leq \tau/(2(R_{d}+1))$.
To reach the conclusion we show that $\Delta_{d}(x,r_{i}s_{j}) \leq \tau$. Indeed
\begin{align*}
\Delta_{d}(x,r_{i}s_{j}) &= \Delta_{d}(rn,r_{i}s_{j})  \stackrel{\eqref{eq:DefMetricD}}{=} 
 \|rn-r_{i}s_{j}\|_{a} + \|n - s_{j}\|_{a} 
  \leq \|rn- r_{i}n \|_{a} + \| r_{i}n - r_{i}s_{j}\|_{a} + \|n - s_{j}\|_{a} \\
& \leq |r-r_i|\|n\|_{a} + (r_i+1)\| n - s_{j}\|_{a} 
 \leq \tau/2 + (R_{d}+1) \tau/(2(R_{d}+1)) \leq \tau.
\end{align*}

\paragraph{An upper bound of $\mathcal{N}_{\mathrm{mm}}(\tau)$.}
By definition of $R_{\mathrm{mm}}$ we have $\Theta_{\mathrm{mm}} \subset R_{\mathrm{mm}}\cdot \mathbb{B}_{a}$. By  \citep[Lemma A.1]{GrBlKeTr20} \citep[Lemma 5.7]{Wain19}  and \eqref{eq:DefMetricMM} we have
\begin{align*}
\mathcal{N}_{\mathrm{mm}}(\tau) 
:= \mathcal{N}(\Theta_{\mathrm{mm}}, \|\cdot\|_{a}, \tau)
& \leq \mathcal{N}\big(R_{\mathrm{mm}} \cdot \mathbb{B}_{a}, \|\cdot\|_{a}, \tau/2 \big)
= \mathcal{N}\big(\mathbb{B}_{a}, \|\cdot\|_{a}, \tau/(2R_{\mathrm{mm}}) \big)\\
&\leq \big(1+4R_{\mathrm{mm}}/\tau\big)^d
\leq \big(1+4D/\tau\big)^{d}.
\end{align*}
\paragraph{An upper bound of $\mathcal{N}_{\mathrm{md}}(\tau)$.}
By~\eqref{eq:bound_gammma_md}, we have  $\Theta_{\mathrm{md}} \subset \Theta_{d} \times \Theta_{\mathrm{mm}}$. Thus by \citep[Lemma A.1]{GrBlKeTr20}, we have
\begin{equation*}
\mathcal{N}_{\mathrm{md}}(\tau) = \mathcal{N}(\Theta_{\mathrm{md}}, \Delta_{\mathrm{md}}, \tau) \leq \mathcal{N}(\Theta_{d} \times \Theta_{\mathrm{mm}}, \Delta_{\mathrm{md}}, \tau/2).
\end{equation*}
Now, given the definition~\eqref{eq:DefMetricMD} of $\Delta_{\mathrm{md}}$
we have
\begin{equation*}
\mathcal{N}_{\mathrm{md}}(\tau) \leq  \mathcal{N}(\Theta_{d} \times \Theta_{\mathrm{mm}}, \Delta_{\mathrm{md}}, \tau/2) 
 \leq \mathcal{N}(\Theta_{d} , \Delta_{d}, \tau/4) \mathcal{N}( \Theta_{\mathrm{mm}}, \Delta_{\mathrm{mm}}, \tau/4).
\end{equation*}
Indeed, with $(x_i)_{i \in [\mathcal{N}_{d}( \tau/4)]}$ a covering of $\Theta_{d}$ with respect to $\Delta_{d}$ at scale $\tau/4$ and $(y_j)_{j \in [\mathcal{N}_{\mathrm{mm}}( \tau/4)]}$ a covering of $\Theta_{\mathrm{mm}}$ with respect to $\Delta_{\mathrm{mm}}$ at scale $\tau/4$ it is straightforward to show using~\eqref{eq:DefMetricMD}  that $(x_i,y_j)_{(i,j) \in [\mathcal{N}_{d}( \tau/4)] \times [\mathcal{N}_{\mathrm{mm}}( \tau/4)]}$ covers $\Theta_{d} \times \Theta_{\mathrm{mm}}$ with respect to $\Delta_{\mathrm{md}}$ at scale $\tau/2$. Combined with the above estimates we get
\[
\mathcal{N}_{\mathrm{md}}(\tau) 
\leq \mathcal{N}_{d}(\tau/4) \mathcal{N}_{\mathrm{mm}}(\tau/4)
\leq \big(1+32(D+1)/\tau \big)^{d+1} \big(1+16D/\tau\big)^d. 
\]
\paragraph{An upper bound of $\mathcal{N}_{\mathrm{dd}}(\tau)$}
By~\eqref{eq:bound_gammma_dd}, we have $\Theta_{\mathrm{dd}} \subset \Theta_{d} \times \Theta_{d} \times \Theta_{\mathrm{mm}}$, thus a similar argument yields
\begin{align*}
\mathcal{N}_{\mathrm{dd}}(\tau)  \leq 
 \mathcal{N}(\Theta_{d} \times \Theta_{d} \times \Theta_{\mathrm{mm}}, \Delta_{\mathrm{dd}}, \tau/2) 
& \stackrel{\eqref{eq:DefMetricDD}}{\leq} 
\mathcal{N}(\Theta_{d}, \Delta_{d},\tau/8)^{2} \mathcal{N}(\Theta_{\mathrm{mm}}, \Delta_{\mathrm{mm}},\tau/4) \nonumber\\
& = \mathcal{N}_{d}(\tau/8)\mathcal{N}_{d}(\tau/8)\mathcal{N}_{\mathrm{mm}}(\tau/4) \\
& \leq \big(1+64(D+1)/\tau\big)^{2(d+1)}  \big(1+16D/\tau\big)^d\\
& \leq  \big(1+64(D+1)/\tau\big)^{3d+2} . 
\end{align*}
To conclude, observe that the last bound dominates all the previous ones.
 \subsection{Proof of Theorem~\ref{thm:RIP_without_weights}}\label{sec:proof_RIP_without_weights}
Since the average marginal density of the $\omega_{j}$'s satisfies~\eqref{eq:ImportanceSamplingStructured}, we have
\begin{align}
\mathbb{E} \Psi_{\mathrm{m}}(\bm{\Omega})&=1\label{eq:ex_psi_m_mu}\\
\mathbb{E} \Psi_{\mathrm{d}}(z|\bm{\Omega})&=1,\quad \forall z \in \Theta_{\mathrm{d}}\label{eq:ex_psi_d_mu}\\
|\mathbb{E} \Psi_{\mathrm{\ell}}(z|\bm{\Omega})| &\leq \mu\quad \forall \ell \in \{\mathrm{mm},\mathrm{md},\mathrm{dd}\},\ \forall z \in \Theta_{\ell}.\label{eq:ex_psi_ell_mu}
\end{align}
The proof of~\eqref{eq:ex_psi_ell_mu} (\eqref{eq:ex_psi_m_mu} and~\eqref{eq:ex_psi_d_mu} are obtained similarly) is postponed to the end of this section. By~\eqref{eq:ex_psi_ell_mu} we have
\begin{equation*}
|\Psi_{\mathrm{\ell}}(z|\bm{\Omega})| \leq |\Psi_{\mathrm{\ell}}(z|\bm{\Omega}) - \mathbb{E} \Psi_{\mathrm{\ell}}(z|\bm{\Omega})|+\mu ,\quad \forall \ell \in \{\mathrm{mm},\mathrm{md},\mathrm{dd}\},\ \forall z \in \Theta_{\ell}
\end{equation*}
hence \eqref{eq:assumption_psi_m_Omega},\eqref{eq:assumption_psi_d_Omega},\eqref{eq:assumption_psi_ell_Omega} yield
\begin{align}
\label{eq:assumption_psi_m_Omega__}
\mathbb{P}\Big(|\Psi_{\mathrm{m}}(\bm{\Omega})-1| > \frac{\tau}{4} \Big) & \leq 2\exp\Big(-\frac{m}{v}\Big)\, ,\\
\label{eq:assumption_psi_d_Omega__}
\forall z \in \Theta_{\mathrm{d}}, \:\:  \mathbb{P}\Big(|\Psi_{\mathrm{d}}(z|\bm{\Omega})-1| > \frac{\tau}{8} \Big) 
&\leq 2\exp\Big(-\frac{m}{v}\Big)\, ,\\
\label{eq:assumption_psi_ell_Omega__}
\forall \ell \in \{\mathrm{mm,md,dd}\}, \:\: \forall z \in \Theta_{\ell}, \:\:  
\mathbb{P}\Big(|\Psi_{\ell}(z|\bm{\Omega})| > \mu+ \frac{\tau}{16k}\Big) 
& \leq 2\exp\Big(-\frac{m}{v}\Big)\, .
\end{align}
Now, \rvs{as mentioned in \Cref{sec:kernel_coherence} (after \Cref{eq:DefCoherenceSkOp}),} the mutual incoherence assumption implies that $\kappa$ is locally characteristic with respect to $\mathcal{T}$ and that its $2k$-coherence is bounded by $c = (2k-1)\mu$. Since $\kappa \geq 0$, all assumptions of~\Cref{prop:generic_upper_bound_worst_sketching_error,thm:upper_bound_sup_monopoles_dipoles,thm:parametrization_functions_cross_dipoles}  and~\Cref{thm:delta_bounds} thus hold.
For each $\ell \in \{\mathrm{d,mm,md,dd}\}$ consider a covering $(z_i^{\ell})_{i \in [\mathcal{N}_{\ell}(\tau')]}$ of $\Theta_{\ell}$ with respect to $\Delta_{\ell}$ at scale 
\begin{equation}\label{eq:IntermediateScaleMainThmPf}
\tau' := \frac{\tau}{96 Mk \cdot C_{\kappa}}.
\end{equation}
By~\Cref{thm:delta_bounds}, if
\begin{align}
\label{eq:MainThmObjective1}
\Psi_{0}(\bm{\Omega}) &\leq M\, ,\\
\label{eq:MainThmObjective2}
\max_{i \in [\mathcal{N}_{\mathrm{d}}(\tau')]} |\Psi_{\mathrm{d}}(z_{i}^{\mathrm{d}}|\bm{\Omega})-1| &\leq \tau/8\\
\label{eq:MainThmObjective3}
\max_{\ell \in \{\mathrm{mm,md,dd}\}} \max_{i \in [\mathcal{N}_{\ell}(\tau')]} |\Psi_{\ell}(z_{i}^{\ell}|\bm{\Omega})| &\leq 
\mu+\frac{\tau}{16k}
\end{align}
then by~\eqref{eq:MetricDomination}, for each $\ell \in \{\mathrm{mm,md,dd}\}$ and $z \in \Theta_{\ell}$ there is $i \in [\mathcal{N}_{\ell}(\tau')]$ such that 
\begin{align*}
|\Psi_{\ell}(z|\bm{\Omega})| 
\leq 
|\Psi_{\ell}(z_{i}^{\ell}|\bm{\Omega})|+|\Psi_{\ell}(z|\bm{\Omega})-\Psi_{\ell}(z_{i}^{\ell}|\bm{\Omega})| 
&\leq \mu+\frac{\tau}{16k} + 6 M \cdot C_{\kappa} \cdot \underbrace{\Delta_{\ell}(z,z_{i}^{\ell})}_{\leq \tau'}\\
& \leq \mu+\frac{\tau}{16k} + \frac{\tau}{16k} \leq 
\mu+\frac{\tau}{8k}\, ,
\end{align*}
and similarly for each $z \in \Theta_{\mathrm{d}}$ there is $i \in [\mathcal{N}_{\mathrm{d}}(\tau')]$ such that 
\begin{align*}
|\Psi_{\mathrm{d}}(z|\bm{\Omega})-1| 
\leq 
|\Psi_{\ell}(z_{i}^{\mathrm{d}}|\bm{\Omega})-1|+|\Psi_{\mathrm{d}}(z|\bm{\Omega})-\Psi_{\ell}(z_{i}^{\mathrm{d}}|\bm{\Omega})| 
&\leq \frac{\tau}{8} + \frac{\tau}{8} \leq  \frac{\tau}{4}.
\end{align*}
Thus, when~\eqref{eq:MainThmObjective1}-\eqref{eq:MainThmObjective2}-\eqref{eq:MainThmObjective3} hold we have
\begin{align*}
\sup_{z \in \Theta_{\mathrm{d}}}|\Psi_{\mathrm{d}}(z|\bm{\Omega})-1|  \leq \frac{\tau}{4} \quad\text{and}\quad
\max_{\ell \in \{\mathrm{mm,md,dd}\}} \sup\limits_{z\in \Theta_{\ell}}|\Psi_{\ell}(z|\bm{\Omega})|  \leq 
\mu+\frac{\tau}{8k}.
\end{align*}
If in addition we have 
\begin{align}
\label{eq:MainThmObjective4}
|\Psi_{\mathrm{m}}(\bm{\Omega})-1| & \leq \frac{\tau}{4}
\end{align}
then by~\Cref{thm:parametrization_functions_cross_dipoles}  and the bound~\eqref{eq:MainDeterministicBound}, which follows from~\Cref{prop:generic_upper_bound_worst_sketching_error,thm:upper_bound_sup_monopoles_dipoles}, 
we obtain
\begin{align*}
\max\big(\delta(\mathfrak{M}|\mathcal{A}),\delta(\hat{\mathfrak{D}}|\mathcal{A})\big) &\leq \tau/4,\\
(2k-1) \max\big(\mu(\mathfrak{M}_{\neq}^2| \mathcal{A}),\mu(\hat{\mathfrak{D}}_{\neq}^2| \mathcal{A}),\mu(\mathfrak{M}\times\hat{\mathfrak{D}}_{\neq}| \mathcal{A}))\big) & \leq (2k-1)\mu+\tau/4 = c+\tau/4,\\
 \delta(\mathcal{S}_{k}|\mathcal{A}) & \leq \frac{1}{1-c}\big(c+\frac{\tau}{4} + 3(c+\frac{\tau}{4})\big) = \frac{4c+\tau}{1-c}.
\end{align*}
Observe that, since $0<\tau<1-5c$, we have $4c+\tau < 1-c$ hence $(4c+\tau)/(1-c) < 1$. 
 
To conclude, we bound the probability $p$ that one of the inequalities~\eqref{eq:MainThmObjective1}-\eqref{eq:MainThmObjective2}-\eqref{eq:MainThmObjective3}-\eqref{eq:MainThmObjective4} fails to hold. 
Using~\eqref{eq:IntermediateScaleMainThmPf}, 
denoting  $D := \operatorname{diam}_{a}(\Theta)$, we have $1+64(D+1)/\tau' = 1+6144 M k (D+1)C_{\kappa}/\tau = 1+C/\tau$.
By a union bound combining~\eqref{eq:MOmega_leq_M}-\eqref{eq:assumption_psi_m_Omega__}-\eqref{eq:assumption_psi_d_Omega__}-\eqref{eq:assumption_psi_ell_Omega__}
and by~\Cref{prop:covering_delta_metric} we obtain
\begin{align*}
p & \leq 2 \exp\Big(-\frac{m}{v}\Big) \cdot \Big(
\frac{\gamma}{2}+1+\sum_{\ell \in \{\mathrm{d,mm,md,dd}\}}\mathcal{N}_{\ell}(\tau')
\Big)
\\
& \leq 2\exp\Big(-\frac{m}{v}\Big) \cdot \Big(\frac{\gamma}{2}+1+4\big(1+64(D+1)/\tau'\big)^{3d+2} \Big)
 \leq 
(\gamma+10) \exp\Big(-\frac{m}{v}\Big) \cdot \big(1+C/\tau\big)^{3d+2}. \qedhere
\end{align*}

{\bf Proof of~\eqref{eq:ex_psi_ell_mu}.}
Let $\ell \in \{\mathrm{mm}, \mathrm{md}, \mathrm{dd} \}$ and let $z \in \Theta_{\ell}$. First, by \Cref{lemma:existence_of_theta_l}, there exists $(\iota,\iota') \in \mathfrak{D}_{\neq}^{2}$  and $\xi \in \{-1,1\}$
such that $\Psi_{\ell}(z|\bm{\Omega}) = \xi \mathfrak{Re}\langle \mathcal{A}\iota, \mathcal{A}\iota'\rangle$ for any choice of $\bm{\Omega}$ (and of the corresponding sketching operator $\mathcal{A}=\mathcal{A}_{\bm{\Omega}}$). In the following, we show that $ \mathbb{E}_{\bm{\Omega}} \mathfrak{Re} \langle \mathcal{A}\iota, \mathcal{A}\iota'\rangle = \langle \iota, \iota' \rangle_{\kappa}$, so that, by the definition~\eqref{eq:DefMutualCoherenceKernel} of the mutual coherence with respect to $\mathcal{T}$, we get
$|\mathbb{E}_{\bm{\Omega}}\Psi_{\ell}(z|\bm{\Omega})| = |\langle \iota, \iota' \rangle_{\kappa}| \leq \mu$. \rvs{ Since $\langle \iota, \iota' \rangle_{\kappa} \in \mathbb{R}$, it is enough to prove that $ \mathbb{E}_{\bm{\Omega}} \langle \mathcal{A}\iota, \mathcal{A}\iota'\rangle = \langle \iota, \iota' \rangle_{\kappa}$. Now,} remember that $ \langle \mathcal{A}\iota, \mathcal{A}\iota'\rangle = \sum_{j=1}^{m} \langle  \phi_{\omega_j}, \iota \rangle \overline{\langle \phi_{\omega_j}, \iota' \rangle}/m$, so that
\begin{align*}
\mathbb{E}_{\bm{\Omega}}\langle \mathcal{A}\iota, \mathcal{A}\iota'\rangle = \frac{1}{m}\sum\limits_{j=1}^{m} \mathbb{E}_{\bm{\Omega}} \langle  \phi_{\omega_j}, \iota \rangle \overline{\langle \phi_{\omega_j}, \iota' \rangle }  & = \frac{1}{m}\sum\limits_{j=1}^{m} \int_{\mathbb{R}^{d}} \Lambda_{j}(\omega)\langle  \phi_{\omega}, \iota \rangle \overline{\langle \phi_{\omega}, \iota' \rangle }  \mathrm{d}\omega \\
& = \int_{\mathbb{R}^{d}}\frac{1}{m}\sum\limits_{j=1}^{m}  \Lambda_{j}(\omega)\langle  \phi_{\omega}, \iota \rangle \overline{\langle \phi_{\omega}, \iota' \rangle } \mathrm{d}\omega \\
& \stackrel{\eqref{eq:ImportanceSamplingStructured}}{=}  \int_{\mathbb{R}^{d}} w(\omega)^2 \hat{\kappa}(\omega) \langle  \phi_{\omega}, \iota \rangle \overline{\langle \phi_{\omega}, \iota' \rangle } \mathrm{d}\omega  = \langle \iota, \iota' \rangle_{\kappa}.
\end{align*}

\subsection{Proof of Theorem~\ref{thm:RIP_without_weights_random_bis} and its corollaries}\label{sec:proof_strong_result}
In this section, we prove Theorem~\ref{thm:RIP_without_weights_random_bis}, Corollary~\ref{cor:bounds_gaussian} and~Corollary~\ref{cor:bounds_dirac}. We start by establishing \Cref{lem:BoundFellOmegaP},  and a few lemmas to deal with sub-exponential random variables.

\subsubsection{Proof of~\Cref{lem:BoundFellOmegaP}}\label{proof:BoundFellOmegaP}
{\bf The case $\ell=\mathrm{d}$.} With $x= z \in \Theta_{d}$, $x' := x/\|x\|_{a}$ we have $\|x'\|_{a}=1$. As $|\sin t| \leq |t|$ for all $t$
\begin{align*}
|f_{\mathrm{d}}(z|\omega)| 
\stackrel{\eqref{eq:DefPsiD}}{=} 2\frac{\sin^2(\langle \omega,x \rangle/2)}{1-\tilde{\kappa}(\|x\|_{a})}
& \stackrel{\eqref{eq:DefAlphaFn}}{=} 2\alpha^{2}(\|x\|_{a})\frac{\sin^2(\langle \omega,x\rangle/2)}{\|x\|_{a}^2} 
\stackrel{\eqref{eq:conditions_on_alpha}}{\leq} \frac{1}{2} C_{\kappa}\langle \omega,x'\rangle^2  \leq C_{\kappa} |\langle \omega,x'\rangle|^2. 
\end{align*}
This establishes~\eqref{eq:PfThm5MainStep} with $p=p_{\mathrm{d}}=2$.\\
{\bf The case $\ell = \mathrm{mm}$.}
Denote $y = z \in \Theta_{\mathrm{mm}}$. We have $|f_{\mathrm{mm}}(z|\omega)| \stackrel{\eqref{eq:DefPsiMM}}{=} |\cos(\langle \omega, y \rangle)| \leq 1$.
This establishes~\eqref{eq:PfThm5MainStep} with $p=p_{\mathrm{mm}}=0$.\\
{\bf The case $\ell = \mathrm{md}$.}
With $(x,y) = z \in \Theta_{\mathrm{md}}$ and $x' := x/\|x\|_{a}$ we have $\|x'\|_{a}=1$ and
\begin{align}
| f_{\mathrm{md}}(z|\omega)| & \stackrel{\eqref{eq:DefPsiMD}}{=} 
2 \Bigg|\frac{\sin\big(\langle \omega, x/2 \rangle \big)\sin\big( \langle \omega, y
+x/2 \rangle \big)}{ \sqrt{2(1-\bar{\kappa}(x))}}\Bigg| 
\stackrel{\eqref{eq:DefAlphaFn}}{\leq}
\sqrt{2} \alpha(\|x\|_{a}) \Bigg|\frac{\sin\big(\langle \omega, x/2 \rangle \big)}{ \|x\|_{a} }\Bigg| \nonumber\\
&  \stackrel{\eqref{eq:conditions_on_alpha}}{\leq}  
\sqrt{2}\sqrt{C_{\kappa}}\bigg|\Big\langle \omega,\frac{  x   }{2 \|x\|_{a}  } \Big\rangle\bigg|
\leq \sqrt{C_{\kappa}} |\langle \omega,x' \rangle|. \nonumber
\end{align}
This establishes~\eqref{eq:PfThm5MainStep} with $p=p_{\mathrm{md}}=1$.\\
{\bf The case $\ell = \mathrm{dd}$.}
With $(x_1,x_2,y) = z \in \Theta_{\mathrm{dd}}$ and $x'_{i} = x_{i}/\|x_{i}\|_{a}$ we have $\|x'_{i}\|_{a}=1$ and
\begin{align*}
|f_{\mathrm{dd}}(z|\omega)| & \stackrel{\eqref{eq:DefPsiDD}}{=} 4\Bigg|\frac{\sin\big(\langle \omega, x_1/2 \rangle \big)\sin\big(\langle\omega, x_2/2\rangle \big)\cos\big( \langle \omega, y+x_2/2-x_1/2 \rangle \big)}{ \sqrt{2(1-\bar{\kappa}(x_1))}\sqrt{2(1-\bar{\kappa}(x_2))}}\Bigg| \nonumber \\
& \stackrel{\eqref{eq:DefAlphaFn}}{\leq}
 2 \alpha(\|x_1\|_{a})\alpha(\|x_2\|_{a})  \Bigg|\frac{\sin\big(\langle \omega, x_1/2 \rangle \big)\sin\big(\langle\omega, x_2/2\rangle \big)}{ \|x_1\|_{a} \|x_2\|_{a}}\Bigg| \nonumber\\
& \stackrel{\eqref{eq:conditions_on_alpha}}{\leq} \frac{C_{\kappa}}{2} \big| \langle \omega, x_1/\|x_1\|_{a} \rangle \langle \omega, x_2/\|x_2\|_{a} \rangle \big|   
= \frac{C_{\kappa}}{2}  \big| \langle \omega, x'_1 \rangle \langle \omega, x'_2 \rangle \big|   
 \leq \frac{C_{\kappa}}{4} \Big(\langle \omega,x'_1\rangle^2+\langle \omega,x'_2 \rangle^2\Big).
\end{align*}

\subsubsection{Some properties of sub-exponential random variables}
\label{sec:proofsubexps}

\begin{proof}[Proof of~\Cref{lemma:sub_exp_2}]
Denote $E := 2 \mathbb{E}Y$. If $E=0$ then $Y=X=0$ almost surely, hence $X$ (and $Y$) are both sub-exp($\nu',\beta'$) for any choice of $\nu',\beta' \geq 0$ so the result is trivial. Assume now that $E>0$. 
Since $|X-\mathbb{E}X| \leq |X|+|\mathbb{E}X| \leq Y+\mathbb{E}Y$ almost surely, we get \begin{equation*}
\mathbb{E}e^{\lambda(X-\mathbb{E}X)} = \sum_{q=0}^{+\infty}\frac{\lambda^{q}}{q!}\mathbb{E}(X-\mathbb{E}X)^{q} \leq 1+\sum_{q=2}^{+\infty}\frac{|\lambda|^{q}}{q!}\mathbb{E}|X-\mathbb{E}X|^{q} \leq 1+\sum_{q=2}^{+\infty}\frac{|\lambda|^{q}}{q!}\mathbb{E}(Y+\mathbb{E}Y)^{q}.
\end{equation*}
Using the binomial formula this yields
\begin{align*}
\mathbb{E}e^{\lambda(X-\mathbb{E}X)}  
& \leq 1+\sum_{q=2}^{+\infty}\frac{|\lambda|^{q}}{q!}\sum\limits_{j=0}^{q}{q \choose j} \mathbb{E}(Y- \mathbb{E}Y)^{j}E^{q-j}\\
& = 1 + \sum_{j=0}^{\infty} \frac{|\lambda|^{j}}{j!} \mathbb{E}(Y- \mathbb{E}Y)^{j} \cdot \sum_{q \geq \max(j,2)} 
\frac{|\lambda|^{q-j} j!}{q!}{q \choose j} E^{q-j}.
\end{align*}
Observe that $\mathbb{E}(Y- \mathbb{E}Y)^{0} = 1$, $\mathbb{E}(Y- \mathbb{E}Y)^{1} = 0$, and 
\begin{align*}
\sum_{q \geq \max(j,2)} 
\frac{|\lambda|^{q-j} j!}{q!}{q \choose j} E^{q-j} & =
\sum_{q \geq \max(j,2)} 
\frac{|\lambda|^{q-j}}{(q-j)!} E^{q-j} = 
\begin{cases}
\sum_{k \geq 0} \frac{(|\lambda| E)^{k}}{k!} = e^{|\lambda| E}, & j \geq 2\\
\sum_{k \geq 2} \frac{(|\lambda| E)^{k}}{k!} = e^{|\lambda| E}-1-|\lambda|E, & j =0.
\end{cases}
\end{align*}
Since $Y$ is sub-exp($\nu,\beta$), when $|\lambda| \leq 1/\beta$ we obtain (using again that $\mathbb{E}(Y-\mathbb{E}Y)^{j} = 0$ for $j=1$)
\begin{align*}
\mathbb{E}e^{\lambda(X-\mathbb{E}X)}  & \leq 
1 + 
(e^{|\lambda|E}-1-|\lambda|E) + 
\sum_{j=2}^{\infty} \frac{|\lambda|^{j}}{j!} \mathbb{E}(Y- \mathbb{E}Y)^{j} e^{|\lambda| E}\nonumber\\
& =  \Big(1+\sum_{j=2}^{\infty} \frac{|\lambda|^{j}}{j!} \mathbb{E}(Y- \mathbb{E}Y)^{j}\Big)e^{|\lambda| E}-|\lambda|E
= \mathbb{E} e^{|\lambda|(Y-\mathbb{E}Y)}e^{|\lambda|E}  -|\lambda|E \nonumber\\
& \leq e^{\nu^{2}\lambda^{2}/2 + |\lambda|E} -|\lambda|E.
\end{align*}
To conclude we use a technical lemma which proof is postponed to  \Cref{sec:subexpas}.
\begin{lemma}\label{lemma:very_useful_bound_on_exp}
For $\alpha >0$, we have
\begin{equation*}\label{eq:ineq_exp_alpha}
\forall t \geq 0, \:\: e^{\alpha t^2/2+t}-t\leq e^{(\alpha+2) t^2}.
\end{equation*}
\end{lemma}
Applying the lemma with $\alpha = (\nu/E)^{2}$,  $t = |\lambda|E$, we obtain 
\begin{equation*}
|\lambda| \leq 1/\beta\ \Longrightarrow
\mathbb{E}e^{\lambda(X-\mathbb{E}X)}  \leq e^{\frac{\nu^{2}\lambda^{2}}{2} + |\lambda|E} -|\lambda|E
= e^{\frac{\alpha t^{2}}{2}+t}-t \leq e^{(\alpha+2)t^{2}} = e^{\frac{2(\nu^{2}+2E^{2})\lambda^{2}}{2}}.
\label{eq:TowardsSubExpXDominatedByY0}
\end{equation*}
This shows that $X$ is sub-exp($\nu',\beta$) with $(\nu')^{2} = 2(\nu^{2}+2E^{2})=2\nu^{2}+16(\mathbb{E}Y)^{2}$. 
\end{proof}

\begin{lemma}\label{lemma:sub_exp_example}
If $X \sim \mathcal{N}(0,1)$ then $X^{2}$ is sub-exp($2,4$) and $|X|$ is sub-exp($4,4$).
\end{lemma}
\begin{proof}
Since $Y := X^{2}$ follows the chi-squared distribution of one degree of freedom, by~\citep{EvHaPeFo11} we have
\begin{equation*}
\mathbb{E}e^{\lambda Y} = \frac{1}{\sqrt{1-2\lambda}}, \:\: \forall |\lambda|<\frac{1}{2}.
\end{equation*}
In particular, since $\mathbb{E}Y = 1$, then 
\begin{equation*}
\mathbb{E}e^{\lambda (Y-\mathbb{E}Y)} = \mathbb{E}e^{\lambda (Y-1)} = \frac{e^{-\lambda}}{\sqrt{1-2\lambda}} \leq e^{\frac{2^2\lambda^2}{2}} , \:\: \forall |\lambda|\leq\frac{1}{4}.
\end{equation*}
This is the definition of a sub-exp($2,4$) random variable.

Now, considering $Z := |X|$, since $|t| \leq t^2+1/4$ for each $t \in \mathbb{R}$, we have
\begin{equation*}
Z = |X| \leq \frac{1}{4}+X^2. \:\: \text{a.s.}
\end{equation*} 
Since $X^2$ is sub-exp($2,4$), $X^2+1/4$ is also sub-exp($2,4$). As $\mathbb{E}(X^2 +1/4) = 5/4 \leq \sqrt{2}$, by \Cref{lemma:sub_exp_2}, $X$ is sub-exp($\nu,4$), with 
$\nu^{2} := 2 \times 2^2 + 4(5/4)^2 \leq 16$, hence $Z$ is sub-exp($4,4$).

\end{proof}

\begin{lemma}\label{lemma:sub_exp_sum}
Consider $X_{i}$, $i=1,2$ two real-valued random variables (possibly non independent), assumed to be respectively sub-exp($\nu_{i},\beta_{i}$). Then $X_{1}+X_{2}$ is sub-exp($\nu,\beta$) where
\begin{equation*}
\nu:= \nu_1 +\nu_2, \:\: \text{and} \:\:\beta:= \max\Big(\frac{\beta_{1}(\nu_1 + \nu_2)}{\nu_1},\frac{\beta_{2}(\nu_1 + \nu_2)}{\nu_2}\Big).
\end{equation*}
\end{lemma}

\begin{proof}
Let $p = (\nu_1 + \nu_2)/\nu_1$ and $q = (\nu_1 + \nu_2)/\nu_2$, so that $1/p+1/q = 1$. By H{\"o}lder's inequality and~\Cref{eq:DefSubExp}, if $|\lambda| \leq \min(\frac{1}{p\beta_1},\frac{1}{q\beta_2}) = \min\Big(\frac{\nu_1}{\beta_{1}(\nu_1 + \nu_2)},\frac{\nu_2}{\beta_{2}(\nu_1 + \nu_2)}\Big)$ then
\begin{align*}
\mathbb{E}e^{\lambda(X_{1}+X_{2})} 
& \leq (\mathbb{E}e^{\lambda p X_{1}})^{1/p}(\mathbb{E}e^{\lambda q X_{2}})^{1/q} 
\leq e^{\frac{\lambda^2 p^2\nu_{1}^2}{2p}}e^{\frac{\lambda^2 q^2\nu_{2}^2}{2q}}
 \leq e^{\frac{\lambda^2 (p\nu_{1}^2+q\nu_{2}^2)}{2}}
 = e^{\frac{\lambda^2 (\nu_{1}+\nu_{2})^2}{2}}. 
\end{align*}
\end{proof}

\subsubsection{Proof of \Cref{thm:RIP_without_weights_random_bis}}
\label{proof:RIP_without_weights_random_bis}

The proof of \Cref{thm:RIP_without_weights_random_bis} leverages \Cref{thm:RIP_without_weights}. 
Before exploiting this theorem we check that the basic required assumptions are met: 
i) $\mathcal{T}= (\Theta, \rho, \mathcal{I})$, $\kappa \geq 0$ and $\|\cdot\|_{a}$ are satisfying the assumptions required in~\Cref{thm:delta_bounds}, 
ii) $\kappa$ is assumed to have its mutual coherence with respect to $\mathcal{T}$ bounded by $0<\mu<1/10$, and $k$ satisfies $1 \leq k <\frac{1}{10 \mu}$, 
iii) $w$ is a $\kappa$-compatible weight function and the average marginal density of the $\omega_{j}$'s satisfies~\eqref{eq:ImportanceSamplingStructured},
iv) the assumption \eqref{eq:MOmega_leq_M} holds. 

Now, we move to check the more technical assumptions: \eqref{eq:assumption_psi_m_Omega},\eqref{eq:assumption_psi_d_Omega},\eqref{eq:assumption_psi_ell_Omega} with $v$ defined in~\eqref{eq:variance_structured}. 
For this purpose, we prove that $\Psi_{m}(\bm{\Omega})$ is sub-exp($\nu'/\sqrt{m/b},\beta' / (m/b)$) and for $\ell \in \{ \mathrm{d}, \mathrm{mm},\mathrm{md},\mathrm{dd} \}$, and for $z \in \Theta_{\ell}$, $\Psi_{\ell}(z|\bm{\Omega})$ is sub-exp($\nu'\sqrt{m/b},\beta' /(m/b)$), where we recall that $b$ is the block size, and $\nu',\beta'$ will be specified in due time to derive \eqref{eq:assumption_psi_m_Omega},\eqref{eq:assumption_psi_d_Omega},\eqref{eq:assumption_psi_ell_Omega} using \eqref{eq:ConcentrationSubExp}.

First, consider $\ell \in \{ \mathrm{d}, \mathrm{mm},\mathrm{md},\mathrm{dd} \}$ and $z \in \Theta_{\ell}$, and observe that
\begin{equation*}
\Psi_{\ell}(z|\bm{\Omega}) = \frac{1}{m/b}\sum\limits_{j=1}^{m/b} \underbrace{\frac{1}{b} \sum\limits_{i = 1}^{b}\psi(\omega_{(j-1)b+i})f_{\ell}(z|\omega_{(j-1)b+i})}_{=: X_{j}},
\end{equation*}
where by assumption the random variables $X_{j}$ are independent and identically distributed.
A well-known property of sub-exp random variables is that if $X_{1}, \dots, X_{n}$ are independent sub-exp($\nu,\beta$)  then  $\frac{1}{n}\sum_{j=1}^{n}X_j$ is sub-exp($\nu/\sqrt{n},\beta/n$). Thus, in order to prove that $\Psi_{\ell}(z|\bm{\Omega})$ is sub-exp($\nu'/\sqrt{m/b},\beta' / (m/b)$), it is enough to prove that the random variable 
\[
X := \frac{1}{b} \sum_{i = 1}^{b}\psi(\omega_i)f_{\ell}(z|\omega_{i})
\]
is sub-exp($\nu',\beta'$). 
For this purpose, we make use of \Cref{lem:BoundFellOmegaP,lemma:sub_exp_2}.

Consider arbitrary $x'_{t} \in \mathbb{R}^{d}$, $t \in \{0,1,2\}$, s.t. $\|x'_{t}\|_{a}=1$, and for $t,p \in \{0,1,2\}$ denote, $Z_{t,p} := \frac{1}{b}\sum_{i = 1}^{b}\psi(\omega_i) (\sqrt{C_{\kappa}}|\langle \omega_i,x'_{t}\rangle|)^p$. By assumption, each $Z_{t,p}$
 is sub-exp($\nu,\beta$) with $|\mathbb{E}(Z_{t,p})| \leq B$. 
We distinguish two cases
\begin{itemize}
\item if $\ell \in \{\mathrm{d},\mathrm{mm},\mathrm{md}\}$, since $\psi(.) \geq 0$, the first claim of \Cref{lem:BoundFellOmegaP}
 implies that $|X| \leq Z_{0,p}$ for some choice of $x'_{0}$, hence by \Cref{lemma:sub_exp_2} $X$ is sub-exp($\nu',\beta$),
where
\begin{equation}\label{eq:DefNuPrimeThm5}
\nu':= 
\sqrt{2\nu^{2}+16B^{2}}.
\end{equation}
\item if $\ell = \mathrm{dd}$ by the second claim of \Cref{lem:BoundFellOmegaP} we similarly get the existence of $x'_{t}\in \mathbb{R}^{d}$ satisfying $\|x'_{t}\|_{a}=1$, $t =1,2$, such that  
$|X| \leq Z'$ with $Z' := (Z_{1,2}+Z_{2,2})/4$.
By Lemma~\ref{lemma:sub_exp_sum}, $Z_{1,2}+Z_{2,2}$ is sub-exp($2\nu,2\beta$), and $|\mathbb{E}(Z_{1,2}+Z_{2,2})| \leq 2B$, hence $Z' := (Z_{1,2}+Z_{2,2})/4$
is also sub-exp($\nu,\beta$) with $|\mathbb{E}(Z')| \leq B$. By \Cref{lemma:sub_exp_2} we also get that $X$ is sub-exp($\nu',\beta$) with $\nu'$ as in~\eqref{eq:DefNuPrimeThm5}. 
\end{itemize}
Now, consider $\ell = \mathrm{m}$. Similarly, to prove that $\Psi_{m}(\bm{\Omega})$ is sub-exp($\nu'/\sqrt{m/b},\beta' /(m/b)$), it is enough to prove that $X := \frac{1}{b} \sum_{i = 1}^{b}\psi(\omega_{i})$ is sub-exp($\nu',\beta'$).  Since $\psi(\omega) = Z_{0,0}$ for any choice of $x'_{0}$, this is indeed true with $\nu'$ as in~\eqref{eq:DefNuPrimeThm5} and $\beta'=\beta$. 

Now, we use~\eqref{eq:ConcentrationSubExp} applied to $\Psi_{\mathrm{m}}(\bm{\Omega})$ with $t = \frac{\tau}{4}$, and to $\Psi_{\mathrm{d}}(z|\bm{\Omega})$ with $t = \frac{\tau}{8}$, and to $\Psi_{\ell}(z|\bm{\Omega})$ with $t = \frac{\tau}{16k}$ for $\ell \in \{ \mathrm{d},\mathrm{mm},\mathrm{md},\mathrm{dd}\}$ and $z \in \Theta_{\ell}$, to get 
that the left hand side of~\eqref{eq:assumption_psi_m_Omega},\eqref{eq:assumption_psi_d_Omega},\eqref{eq:assumption_psi_ell_Omega} is bounded from above by
\begin{align*}
\max_{t \in \{\frac{\tau}{4},\frac{\tau}{8},\frac{\tau}{16k}\}}
 2 \exp\left(-\frac{(m/b)t^{2}}{2\nu'^{2}+\beta t}\right)
 =2 \exp\left(-\frac{(m/b)\left(\frac{\tau}{16k}\right)^{2}}{2\nu'^{2}+\beta \left(\frac{\tau}{16k}\right)}\right)
 =2 \exp\left(-\frac{m \tau^{2}}{256bk^{2}(2\nu'^{2}+\beta \left(\frac{\tau}{16k}\right)}\right)
\end{align*}
where we used that $t \mapsto t^{2}/(2(\nu')^{2}+\beta t)$ is an increasing function.
We conclude by observing that, with $v$ as defined in \eqref{eq:variance_structured}, we have
\begin{equation*}
256bk^{2}(2\nu'^{2}+\beta \frac{\tau}{16k}) \leq
256bk^2(2\nu'^{2}+\beta \tau) =   \tau^{2}v,
\end{equation*}

To prove the variant of the theorem, we first reason as above to show that the modified assumptions imply that 
for arbitrary $x'_{t} \in \mathbb{R}^{d}$, $t \in \{0,1,2\}$, s.t. $\|x'_{t}\|_{a}=1$ the random variable
$Y_{p} :=  \frac{1}{b}\sum_{i = 1}^{b}B_{\psi}
 (\sqrt{C_{\kappa}}|\langle \omega,x'_{0}\rangle|)^{p}$  is sub-exp($B_{\psi}\nu,B_{\psi}\beta$) with $|\mathbb{E}Y_{p}| \leq B_{\psi}\beta$, and similarly for $Y' :=  \frac{1}{b}\sum_{i = 1}^{b}B_{\psi} C_{\kappa} (\langle \omega,x'_{1}\rangle^{2}+\langle \omega,x'_{2}\rangle^{2})/4$. Then, using the definition of $B_{\psi}$ we obtain that:
for $\ell = \mathrm{m}$, $X :=  \frac{1}{b}\sum_{i = 1}^{b}\psi(\omega_{i})$ satisfies $|X| \leq Y_{0}$; 
for $\ell \in \{\mathrm{d},\mathrm{mm},\mathrm{md}\}$ and $z \in \Theta_{\ell}$,
$X := \frac{1}{b}\sum_{i = 1}^{b} \psi(\omega_{i})f_{\ell}(z|\omega_{i})$ satisfies $|X| \leq Y_{p}$ for an appropriate choice of $x'_{0}$, $p \in \{0,1,2\}$;
for $\ell = \mathrm{dd}$ and $z \in \Theta_{\ell}$, 
$X :=  \frac{1}{b}\sum_{i = 1}^{b}\psi(\omega_{i})f_{\ell}(z|\omega_{i})$ satisfies $|X| \leq Y'$ for an appropriate choice of $x'_{1},x'_{2}$. The same reasoning as above yields that $X$ is sub-exp($B_{\psi}\nu',B_{\psi}\beta$) with $\nu'$ as in~\eqref{eq:DefNuPrimeThm5}. We conclude similarly once we observe that $256bk^{2}(2B_{\psi}^{2}\nu'^{2}+B_{\psi}\beta \tau) = \tau^2 v'$.

\subsubsection{Proofs of Corollary~\ref{cor:bounds_gaussian} and Corollary~\ref{cor:bounds_dirac}}\label{proof:bounds_gaussian}

\Cref{cor:bounds_gaussian} and \Cref{cor:bounds_dirac} are direct consequences of \Cref{thm:RIP_without_weights_random_bis}. 
To see why, we check that the assumptions of these theorems are met in these settings.

\paragraph{Checking that the assumptions on $\mathcal{T}$ and $\kappa$ hold, and controlling $C_{\kappa}$, $\mu$.}

First, observe that, in the setting of Gaussian (resp. Dirac) mixtures of \Cref{ex:KernelExamplesIncoherent} (resp. of \Cref{ex:DiracMixture}), the kernel satisfies $\kappa \geq 0$, and~\eqref{eq:DefRadialKernel} holds with $\|\cdot\|_{a} = \|\cdot\|_{\bm{\Sigma}}$ (resp. with $\|\cdot\|_{a} = \|\cdot\|_{2}$) and $\tilde{\kappa}(r) = e^{-r^{2}/\sigma^{2}}$
with $\sigma := \sqrt{2(2+s^2)}$ (resp. with $\sigma := \sqrt{2}s$). In both settings we have $\varrho = \|\cdot\|_{a}/\epsilon$, 
and by the definition of $\Theta_{\mathrm{d}}$ (see~\eqref{eq:bound_gammma_d}, with $\|\cdot\| := \|\cdot\|_{a}/\epsilon$), we have $R := \sup_{x \in \Theta_{\mathrm{d}}} \|x\|_{a}\leq \epsilon$. 
Thus, by Lemma~\ref{lemma:sup_alpha_alpha_prime}, the constant $C_{\kappa}$ from~\eqref{eq:conditions_on_alpha} satisfies
$
C_{\kappa} \leq \max\big(1,\sqrt[4]{3}R,\sqrt[4]{3}\sigma \big)^2
\leq \max\big(1 , \sqrt{3}\epsilon^2, \sqrt{3}\sigma^2 \big).
$ 
Now we proceed separately for the two settings.
\begin{itemize}
\item For Gaussian mixtures, as $\sqrt{2+s^2} \leq \epsilon(4\sqrt{\log(ek)})^{-1}$, we get $\sigma = \sqrt{2(2+s^{2})} \leq \epsilon$, and
 by \citep[Theorem 5.16, Lemma 6.10]{GrBlKeTr20} $\kappa$ is locally characteristic with mutual coherence with respect to $\mathcal{T}$ bounded by $\mu \leq 12/(16(10k-1))$.
\item For mixtures of Diracs, since $s \leq \epsilon(4\sqrt{\log(5ek)})^{-1}$, we get $\sigma :=\sqrt{2}s \leq \epsilon$ and the bound on the coherence holds by \citep[Theorem 5.16, Lemma 6.10]{GrBlKeTr20}.
\end{itemize}
In both cases, we get $\mu \leq 12/(16(10k-1))m \leq 1/(10k)$ and $\sigma \leq \epsilon$. The latter implies
\begin{equation}\label{eq:bound_on_C_kappa_final}
C_{\kappa} \leq
\max\big(1 , \sqrt{3}\epsilon^2, \sqrt{3}\sigma^2 \big)
= \max(1,\sqrt{3}\epsilon^{2}).
\end{equation}
Since $B_{\psi} :=\sup_{\omega \in \mathbb{R}^d} \psi(\omega)$ is finite in both settings, we may use the 
variant of \Cref{thm:RIP_without_weights_random_bis}. Indeed, by \eqref{eq:psi_expression_gaussian} (resp. by~\eqref{eq:psi_expression_dirac}) $B_{\psi} = (1+2s^{-2})^{d/2}$ for Gaussian mixtures (resp. $B_{\psi} =1$ for mixtures of Diracs). Now, we check that the assumptions expressed in \Cref{cond:new_thm_1} (in its variant 
involving $Z'_{p}$) and \Cref{cond:new_thm_2} of \Cref{thm:RIP_without_weights_random_bis}
 hold in both settings. 

\paragraph{Checking the variant of \Cref{cond:new_thm_1} in \Cref{thm:RIP_without_weights_random_bis}} 

We show the existence of $\nu,\beta,B>0$ such that the random variables $Z'_{p}$ defined in \eqref{eq:Z_def_structured_bis} are sub-exp($\nu,\beta$) with $|\mathbb{E}(Z_{p})| \leq B$. Since the frequencies $\omega_{1}, \dots, \omega_m$ are i.i.d., we consider a block size $b = 1$ and 
$Z'_p = (\sqrt{C_{\kappa}} |\langle \omega, x \rangle|)^p$
with $\omega \sim \Lambda$.
We will indeed prove that for each $x \in \mathbb{R}^d$ s.t. $\|x\|_{a} =1$, the random variables $ |\sqrt{C_{\kappa}}\langle \omega, x \rangle|^{p}$, $p \in \{0,1,2\}$ are sub-exp$(\nu,\beta)$ with $|\mathbb{E}Z_p| \leq B$, 
where
\begin{equation}\label{eq:properties_on_B_and_nu}
B := \max(1,C_{\kappa}s^{-2}),
\quad \textit{and} \quad \nu = \beta = 4 B.
\end{equation}
This is done in the following. To handle both settings in a common framework, define $\bm{\Sigma} = \mathbb{I}_{d}$ for the setting of a mixture of Diracs, so that in both cases we have $\|\cdot\|_{a} = \|\cdot\|_{\bm{\Sigma}}$. 
Thus, a vector $x \in \mathbb{R}^{d}$ satisfies $\|x\|_{a} = \|x\|_{\bm{\Sigma}} = 1$ if, and only if, $\|\bm{\Sigma}^{-1/2}x\|_{2} =1$. 
Since $\omega \sim \mathcal{N}(0,s^{-2} \bm{\Sigma}^{-1})$ we have $s\bm{\Sigma}^{1/2}\omega \sim \mathcal{N}(0,\mathbb{I}_d)$ hence $s \langle \omega,x \rangle = s\langle\bm{\Sigma}^{1/2}\omega,\bm{\Sigma}^{-1/2}x\rangle \sim \mathcal{N}(0,1)$, so that $(\mathbb{E}|\langle\omega,x\rangle|)^{2} \leq \mathbb{E}|\langle\omega,x\rangle|^{2} = 1/s^{2}$, hence using also that $|\langle\omega,x\rangle|^{0} \equiv 1$ we obtain
\begin{equation*}
\max_{p \in \{0,1,2\}} \mathbb{E}|\sqrt{C_{\kappa}}\langle\omega,x\rangle|^{p} \leq \max(1,\sqrt{C_{\kappa}}s^{-1},C_{\kappa}s^{-2}) =  B.
\end{equation*}

By Lemma~\ref{lemma:sub_exp_example}, $s|\langle\omega,x\rangle|$ is sub-exp($4,4$) and $s^2\langle\omega,x\rangle^2$ is sub-exp($2,4$), hence $|\sqrt{C_{\kappa}}\langle\omega,x\rangle|$ is sub-exp($4\sqrt{C_{\kappa}}/s,4\sqrt{C_{\kappa}}/s$), and $C_{\kappa}\langle\omega,x\rangle^2$ is sub-exp($2C_{\kappa}/s^2,4C_{\kappa}/s^2$). Observe that $\max(4\sqrt{C_{\kappa}}/s,2C_{\kappa}/s^{2}) \leq 4 B$ and $\max(4\sqrt{C_{\kappa}}/s,4C_{\kappa}/s^{2}) \leq 4B$ to conclude that $|\sqrt{C_{\kappa}}\langle\omega,x\rangle|$ and $C_{\kappa}\langle\omega,x\rangle^2$ are indeed both sub-exp($\nu,b$) with $\nu=b =4B$. 
The same holds for $p=0$ since $|\sqrt{C_{\kappa}}\langle\omega,x\rangle|^{0} \equiv 1$.

NB: in the setting of Gaussian mixtures, for $x\in \mathbb{R}^d$ such that $\|x\|_{\bm{\Sigma}}=1$, $\psi(\omega)$ and $\psi(\omega)|\langle\omega,x\rangle|$ and $\psi(\omega)\langle\omega,x\rangle^2$ are bounded and we may alternatively have used Hoeffding's inequality \citep{Hoef94} instead of Theorem~\ref{thm:RIP_without_weights_random_bis}.
We chose to use the latter as it allows to encompass both settings under the same reasoning.

We move now to check that~\Cref{cond:new_thm_2} holds in both settings.

\paragraph{Identifying $M$ such that~\Cref{cond:new_thm_2} in \Cref{thm:RIP_without_weights_random_bis} holds with $\gamma=1$ for Gaussian mixtures with any $v>0$.}
We show that 
\begin{equation}
\label{eq:MGaussianProof}
|\psi(\omega) f_{0}(\omega)| \leq M := 4B_{\psi},\ \forall \omega,
\end{equation}
hence~\eqref{eq:assumption_on_psi0_thm5} holds 
for any $v>0$. In particular, \eqref{eq:assumption_on_psi0_thm5} holds for $v = v_{k}(\tau)$, with $v_{k}(\tau)$ defined in \eqref{eq:DefVKTau}. Indeed, for every $t \geq 0$ we have $t^2 \leq (t+t^3)/2$ (since $t(t-1)^2 \geq 0$), hence given the definition~\eqref{eq:M_Omega_def} of $f_{0}$ and since $\|\cdot\|_{a,\star} = \|\cdot\|_{\bm{\Sigma}^{-1}}$ we have $f_{0}(\omega) \leq \frac{3}{2}(\|\omega\|_{\bm{\Sigma}^{-1}}+ \|\omega\|_{\bm{\Sigma}^{-1}}^{3})$ for every $\omega$, so that
\begin{align*}
|\psi(\omega) f_{0}(\omega)| 
 \stackrel{\eqref{eq:psi_expression_gaussian}}{=} B_\psi e^{-\omega^{\mathrm{T}} \bm{\Sigma} \omega/2} f_{0}(\omega)
& \leq \frac{3}{2}B_\psi e^{-\|\omega\|_{\bm{\Sigma}^{-1}}^{2}/2} (\|\omega\|_{\bm{\Sigma}^{-1}}+\|\omega\|_{\bm{\Sigma}^{-1}}^{3})\\
& \leq \frac{3}{2}B_\psi \varphi(\|\omega\|_{\bm{\Sigma}^{-1}}) \leq 4B_{\psi} \end{align*}
since $\varphi(t):= (t+t^3)e^{-t^2/2}$ satisfies $\varphi(t) \leq 8/3$ for $t \in \mathbb{R}$. 
Indeed, we have 
\[
\sup_{t \in \mathbb{R}}\varphi(t) \leq \sup_{t \in \mathbb{R}}\varphi_{1}(t)+\sup_{t \in \mathbb{R}}\varphi_{2}(t),
\]
where $\varphi_{1}(t):= te^{-t^2/2}$, and $\varphi_{2}(t):= t^3e^{-t^2/2}$, and it is easy to prove that $\sup_{t \in \mathbb{R}}\varphi_{1}(t) = \varphi_{1}(1) = e^{-1/2}  \leq 1$, 
and $\sup_{t \in \mathbb{R}}\varphi_{2}(t) = \varphi_{2}(\sqrt{3}) = 3\sqrt{3}e^{-3/2} \approx 1.16 \leq 5/3 \approx 1.66$.

\paragraph{Identifying $M$ such that~\Cref{cond:new_thm_2} holds with $\gamma=1$ for mixtures of Diracs with $v = v_{k}(\tau)$.} 
Since $\psi(\omega) = 1$ (cf \eqref{eq:psi_expression_dirac}) we have $\Psi_{0}(\bm{\Omega}) = \frac{1}{m} \sum_{j=1}^{m} f_{0}(\omega_{j})$, with $f_{0}(\omega)$  defined in~\eqref{eq:M_Omega_def}. Since $\|\cdot\|_{a}=\|\cdot\|_{2}$ we have $\|\cdot\|_{a,\star}=\|\cdot\|_{2}$ hence $f_{0}(\omega) = \sum_{t=1}^{3}\|\omega\|_{2}^{t}$. As $\omega \sim \mathcal{N}(0,s^{-2} \mathbb{I}_{d})$, \Cref{prop:tail_sum_norm_omega_alpha} yields
\begin{align*}
\forall t \geq 0,\quad \mathbb{P}\Big( \Psi_{0}(\bm{\Omega}) >  \tilde{M}(t) \Big) \leq \exp\bigg( -\frac{(m t)^{2/3}}{2} \bigg),
\end{align*}
with $ \tilde{M}(t) :=1+ \frac{8}{s^3}\big(\sqrt{\frac{8}{\pi}}d^{3/2} + t  \big)$. 
We define
\begin{equation}\label{eq:DefC0}
C_{0} := 7B_{\psi}B \end{equation}
and recall that $C_{0} \geq 7$ since $B_{\psi} \geq 1$ and $B \geq 1$. 
Therefore, for $0<\tau \leq 1$, and $v_{k}(\tau)$ as defined in~\eqref{eq:DefVKTauBis}, we have 
\begin{align*}
 \mathbb{P}\Big( \Psi_{0}(\bm{\Omega}) >  \tilde{M}(\tau^{3/2}m^{1/2}) \Big) \leq \exp\bigg( -\frac{m \tau}{2} \bigg) \leq \exp(-m/v_k(\tau)),
\end{align*}
since $2/\tau \leq 2/\tau^2 \leq 2(C_{0}/\tau)^{2} \leq v_{k}(\tau)$. 

In other words, ~\eqref{eq:assumption_on_psi0_thm5} holds with $M = \tilde{M}(\tau^{3/2}m^{1/2}) $ and $v = v_{k}(\tau)$.\\
{\bf Wrapping up the proof.}

To complete the proof of \Cref{cor:bounds_gaussian} and \Cref{cor:bounds_dirac} we use Theorem~\ref{thm:RIP_without_weights_random_bis} and the last step is to give explicit upper bounds of the constants $C$ and $v'$ respectively defined in~\eqref{eq:C_def_thm5} and~\eqref{eq:variance_structured_bis}. 

We start with $v'$, and we show that it is upper bounded by $v_{k}(\tau)$ defined in~\eqref{eq:DefVKTau} (resp. in~\eqref{eq:DefVKTauBis}). By \eqref{eq:properties_on_B_and_nu}, $\nu = \beta = 4B$ hence, by~\eqref{eq:variance_structured} $\nu' = \sqrt{2}\sqrt{\nu^2 + 8 B^2} = \sqrt{48 B^2} \leq \sqrt{49B^2} = 7B$. Hence, $B_{\psi}^{2}\nu'^{2} \leq (7B_{\psi}B)^{2} = C_{0}^{2}$ and $B_{\psi}\beta = 4B_{\psi}B =\frac{4}{7}C_{0}$ where $C_{0}$ is defined in~\eqref{eq:DefC0}. Since the block size is $b=1$, by~\eqref{eq:variance_structured_bis} we obtain
\begin{align*}
v' = \frac{256k^2b(2B_{\psi}^{2}\nu'^2+B_{\psi}\beta\tau)}{\tau^2} &\leq512k^2  \left((C_{0}/\tau)^{2}+\frac{2}{7}(C_{0}/\tau)\right) \\
&\leq512k^2  \left((C_{0}/\tau)^{2}+\frac{1}{3}(C_{0}/\tau)\right). \end{align*}
This matches the expressions of $v_{k}(\tau)$ used in~\eqref{eq:DefVKTau} and ~\eqref{eq:DefVKTauBis}. Soon we will also prove that $C_{0}$ satisfies the bounds expressed in~\eqref{eq:DefVKTau} and ~\eqref{eq:DefVKTauBis}.

As for $C$, observe that by definition~\eqref{eq:C_def_thm5}, 
$C = 6144MC_{\kappa} \cdot k(1+\mathrm{diam}_{a}(\Theta))$, so that we only need to give an upper bound of $6144 MC_{\kappa}$ to get the expressions that appear in~\eqref{eq:DefCofVKTau}  and \eqref{eq:DefCofVKTauBis}.

To conclude we bound $C_{0}$ and $6144MC_{\kappa}$. First, by~\eqref{eq:bound_on_C_kappa_final}, we have $C_{\kappa} \leq \max(1,\sqrt{3}\epsilon^2)$, and by~\eqref{eq:properties_on_B_and_nu} $B = \max(1,C_{\kappa}s^{-2})$. We study separately the two settings:
\begin{itemize}
\item For mixtures of Gaussians, by~\eqref{eq:SepVsScaleGMM} we have $\epsilon \geq \max(s,1)$ hence $C_\kappa \leq \sqrt{3}\epsilon^{2}$, $B \leq \sqrt{3}(\epsilon/s)^{2}$. Since $B_{\psi} = (1+2s^{-2})^{d/2}$, 
we get by~\eqref{eq:DefC0}
$C_{0} = 7BB_{\psi} \leq 7\sqrt{3}\epsilon^{2}s^{-2}(1+2s^{-2})^{d/2}$ as claimed in~\eqref{eq:DefVKTau}. Since $M  = 4 B_{\psi}$ by~\eqref{eq:MGaussianProof}, we get $6144 M C_{\kappa} \leq (4\times 6144\sqrt{3}) \epsilon^{2}B_{\psi}\leq 43000 \epsilon^{2} B_{\psi}$ as claimed in~\eqref{eq:DefCofVKTau}.
\item For mixtures of Diracs, by \eqref{eq:SepVsScaleDiracs} we have $\epsilon \geq s$ hence 
 $\sqrt{3}\epsilon^{2}s^{-2}\geq 1$. As $C_{\kappa}s^{-2} \leq \max(s^{-2},\sqrt{3}\epsilon^{2}s^{-2})$, it follows that $B = \max(1,C_{\kappa}s^{-2}) \leq  \max(1,s^{-2},\sqrt{3}\epsilon^{2}s^{-2}) 
  = \max(s^{-2},\sqrt{3}\epsilon^{2}s^{-2}) = s^{-2} \max(1,\sqrt{3}\epsilon^{2})$. Since $B_{\psi} = 1$ it follows by~\eqref{eq:DefC0} that
$C_{0} = 7BB_{\psi} =7B \leq 7 s^{-2}\max(1,\sqrt{3}\epsilon^{2})$ as claimed in \eqref{eq:DefVKTauBis}.
Finally observe that $c:= \sqrt{8/\pi} \approx 1.6$ so that $cd^{3/2}+t \geq 1$ for every $t>0$ hence $\tilde{M}(t) = 1+\frac{8}{s^{3}}(c d^{3/2}+t) \leq (1+8/s^{3})(cd^{3/2}+t) \leq (1+2/s)^{3}(c d^{3/2}+t) \leq (1+2/s)^{3}(2d^{3/2}+t)$. 
Since $M  = \tilde{M}(m^{1/2}\tau^{3/2})$ it follows that
$6144 M C_{\kappa} \leq 6144 (1+2/s)^{3} \max(1,\sqrt{3}\epsilon^{2}) (2d^{3/2}+\sqrt{m}\tau^{3/2})$ as claimed in~\eqref{eq:DefCofVKTauBis}.
\end{itemize}

\subsubsection{Proof of~\Cref{cor:bounds_str}}\label{proof:bounds_str}

The proof follows the same steps as the proofs of \Cref{cor:bounds_gaussian} and \Cref{cor:bounds_dirac} given in \Cref{proof:bounds_gaussian}, even though the $\omega_{j}$ are not independent, we can still prove that \Cref{cond:new_thm_1} and \Cref{cond:new_thm_2} of \Cref{thm:RIP_without_weights_random_bis} hold.

Indeed, by \Cref{lemma:sub_exp_sum}, the random variable $d \times Z_{p} = \sum_{j=1}^{d}\psi(\omega_{j})(\sqrt{C_{\kappa}}|\langle \omega_{j},x \rangle|)^p$ is sub-exp($d\nu, d\beta$), since it is a sum of $d$ random variables that are sub-exp($\nu,\beta$), thus $Z_{p}$ (and similarly $Z'_{p}$) is sub-exp($\nu,\beta$), and \Cref{cond:new_thm_1} holds.

As for \Cref{cond:new_thm_2}, we distinguish the two cases:
\begin{itemize}
\item 
For mixtures of Gaussians, the variables $\psi(\omega_j)f_{0}(\omega_j)$ are bounded and the same reasoning
as in the proof of \Cref{cor:bounds_gaussian} establishes \eqref{eq:assumption_on_psi0_thm5} with the same constant $M$ and any $v>0$. Gathering all of the above shows that for Gaussian mixture models, sketching with the considered structured random Fourier features satisfies the RIP: \eqref{eq:the_RIP}
holds with $\gamma = 1$, $v = d v_{k}(\tau)$, where $v_{k}(\tau)$ given in \eqref{eq:DefVKTau}.
\item For mixtures of Diracs, we  decompose $\bm{\Omega}$ into $d$ (non-adjacent) blocks of $m/d$ i.i.d. random variables, 
$\bm{\Omega}_i:=\{ \omega_{i}, \omega_{i+d},\ldots, \omega_{i+(m/d-1)d} \}$, $i \in [d]$ and write
\begin{equation*}
\Psi_{0}(\bm{\Omega}) = \frac{1}{d}\sum\limits_{i=1}^{d} \frac{1}{m/d} \sum\limits_{j =0}^{m/d-1}\psi(\omega_{i+d(j-1)})f_{0}(\omega_{i+d(j-1)}) = \frac{1}{d} \sum_{i=1}^{d}\Psi_{0}(\bm{\Omega}_{i}),
\end{equation*}
and observe that, since the frequencies are block-i.i.d., the columns of each $\bm{\Omega}_{i}$ are independent so that we can use ~\Cref{prop:tail_sum_norm_omega_alpha} to obtain with the same reasoning as in the proof of \Cref{cor:bounds_dirac} 
\begin{equation*}
\mathbb{P}\Big( \Psi_{0}(\bm{\Omega}_i) >M \Big) \leq \exp\Big( -
\frac{m/d}{v_{k}(\tau)}
\Big),
\end{equation*}
where $M:= \tilde{M}(\tau^{3/2}(m/d)^{1/2})$.
A union bound yields \begin{align*}
\mathbb{P}\Big( \Psi_{0}(\bm{\Omega})>M \Big) & \leq \mathbb{P}\Big( \cup_{i = 1}^{d} \{ \Psi_{0}(\bm{\Omega}_i)>M \} \Big)
 \leq d \exp\Big( -\frac{m}{dv_{k}(\tau)} \Big),
\end{align*}
and \Cref{cond:new_thm_2} holds with $\gamma = d$, $v = d v_{k}(\tau)$ and $M:= \tilde{M}(\tau^{3/2}(m/d)^{1/2})$. \qedhere
\end{itemize}

\subsubsection{Some helpful results}

\begin{lemma}\label{lemma:sup_alpha_alpha_prime}
Consider $\sigma>0$ and for $r  \geq 0$ define $\alpha_{\sigma}(r):= \frac{r}{\sqrt{1-e^{-r^2/\sigma^2}}}$.
For each $R>0$ we have
\begin{equation*}
\sup\limits_{r \in [0,R]}|\alpha_{\sigma}(r)| \leq \sqrt[4]{3}\max(\sigma,R),
\quad \text{and}\quad 
\sup\limits_{r \in (0,R]}|\alpha'_{\sigma}(r)| \leq 1 .
\end{equation*}
\end{lemma}

\begin{proof}
First we show that $\alpha_{\sigma}(r) = \sigma \alpha_{1}(r/\sigma) \leq (1-e^{-1})^{-1/2}\max(\sigma,r)$ when $r \geq 0$. This implies the first bound as $(1-e^{-1})^{-1/2} \approx 1.26 \leq 1.316 \approx \sqrt[4]{3}$. When $\sigma \leq r$ we have
\[
|\alpha_{\sigma}(r)| =  \frac{r}{\sqrt{1-e^{-r^2/\sigma^2}}} \leq \frac{r}{\sqrt{1-e^{-1}}} \leq (1-e^{-1})^{-1/2} r 
= (1-e^{-1})^{-1/2} \max(\sigma,r).
\]
When $0 \leq r \leq \sigma$, we prove below that $|\alpha_{1}(t)| \leq (1-e^{-1})^{-1/2}$ for every $t \in [0,1]$, so that
\[
|\alpha_{\sigma}(r)| = \sigma |\alpha_{1}(r/\sigma)| \leq (1-e^{-1})^{-1/2}\sigma = (1-e^{-1})^{-1/2} \max(\sigma,r).
\]
To show that $|\alpha_{1}(t)| \leq (1-e^{-1})^{-1/2}
$ on $[0,1]$ observe that since $u \mapsto e^{-u}$ is convex, it has non-decreasing slopes so that $(e^{-u}-1)/u \leq (e^{-1}-1)/1$, $\forall u \in [0,1]$. This reads $u/(1-e^{-u}) \leq 1/(1-e^{-1})$ and implies $|\alpha_{1}(t)|^{2} = t^{2}/(1-e^{-t^{2}}) \leq 1/(1-e^{-1})$ for $t \in [0,1]$, thus $|\alpha_{1}(t)| \leq 1/\sqrt{1-e^{-1}}$.

We now prove that $|\alpha'_{1}(t)| \leq 1$, $\forall t >0$. This implies the second bound since $\alpha_{\sigma}'(r) = \alpha_{1}'(r/\sigma)$. Writing $\alpha_{1}(t) = t [v(t)]^{-1/2}$ 
with $v(t) := 1-e^{-t^{2}}$ we get 
\begin{align*}
\alpha'_{1}(t) &= [v(t)]^{-1/2}-\frac{t}{2}v'(t)[v(t)]^{-3/2} = [v(t)]^{-3/2}\left(v(t)-tv'(t)/2\right)
= [v(t)]^{-3/2} \left(1-e^{-t^{2}}(1+t^{2})\right)
\end{align*}
For each $t \geq 0$, since $e^{t^{2}} \geq 1+t^{2}$, it follows that $|\alpha'_{1}(t)| = \alpha'_{1}(t)$.
When $t\geq 1$, since $0 < x := e^{-t^{2}} \leq 1/e \approx 0.368 < 1/2 < 0.618 \approx (\sqrt{5}-1)/2$ we have 
$(x-1)^{3}+(1-2x)^{2} = x^{3}+x^{2}-x = x(x^{2}+x-1) = x[x+(1+\sqrt{5})/2][x+(1-\sqrt{5})/2] \leq 0$ hence $(1-2x)^{2} \leq (1-x)^{3}$, and since $1-2x >0$ it follows that $1-2x\leq (1-x)^{3/2}$. Thus, we obtain
\[
|\alpha'_{1}(t)| = \frac{1-e^{-t^{2}}(1+t^{2})}{\sqrt{1-e^{-t^2}}^{3}} \leq \frac{1-2e^{-t^{2}}}{\sqrt{1-e^{-t^2}}^{3}} 
=\frac{1-2x}{(1-x)^{3/2}}
\leq 1,
\]
Now, when $0<t \leq 1$, since $[v(t)]^{-3/2} = (\alpha_{1}(t)/t)^{3}$, 
using that $|\alpha_{1}(t)| \leq (1-e^{-1})^{-1/2} \max(1,t) = (1-e^{-1})^{-1/2}$ and $(1-e^{-1})^{-3/2} \approx 1.99 \leq 2$ we get
$$|\alpha'_{1}(t)| \leq (1-e^{-1})^{-3/2} \frac{1-e^{-t^{2}}(1+t^{2})}{t^3} \leq 2 \frac{1-e^{-t^{2}}(1+t^{2})}{t^3}. $$ 
It is enough to show that $g(t) := (1-e^{-t^{2}}(1+t^{2}))/t^{3} \leq 1/2$, $\forall t \in (0,1]$. We have $g'(t) =  t^{-4}(-3+e^{-t^2}(3+3t^2+2t^4))$ hence $\mathtt{sign}(g'(t)) = -\mathtt{sign}(e^{t^{2}}-(1+t^{2}+\frac{2}{3}t^{4}))$. Thus there is a neighborhood of zero in which $g$ is increasing, since $\mathtt{sign}(g'(t)) = -\mathtt{sign}(1+t^{2}+\frac{1}{2}t^{4}+O(t^{6})-(1+t^{2}+\frac{2}{3}t^{4})) = -\mathtt{sign}(-\frac{1}{6}t^{4}+O(t^{6})) = +1$ for $t$ small enough. 
Since $g$ is continuously differentiable, its supremum on $(0,1]$ is thus either equal to $g(1) = 1-2/e \approx 0.264 < 1/2$ or to $g(t_{*})$, for some local maximum $0< t_{*} \leq 1$ which must satisfy $g'(t_{*}) = 0$. To conclude without further characterizing the existence or value of such a root, we establish that any such root must satisfy $g(t_{*}) \leq 1/2$.  Indeed using that $g'(t_{*})=0$ if, and only if $e^{t_{*}^{2}} = 1+t_{*}^{2}+\frac{2}{3}t_{*}^{4}$, we obtain
\begin{equation*}
g(t_{*}) = 
\frac{1-e^{-t_{*}^{2}}(1+t_{*}^{2})}{t_{*}^{3}} 
=
\frac{e^{t_{*}^{2}} -(1+t_{*}^{2})}{e^{t_{*}^{2}} t_{*}^{3}} 
=
\frac{1+t_{*}^{2}+\frac{2}{3}t_{*}^{4}-(1+t_{*}^{2})}{(1+t_{*}^{2}+\frac{2}{3}t_{*}^{4})t_{*}^{3}}
= \frac{\frac{2}{3}t_{*}}{1+t_{*}^2+\frac{2}{3}t_{*}^4}.
\end{equation*}
We finally distinguish two cases: 
i) if $t_{*} <1/\sqrt{2}$ then $g(t_{*}) \leq \frac{2}{3}t_{*} \leq \frac{\sqrt{2}}{3} \leq 1/2$; 
ii) if $1/\sqrt{2}  \leq t_{*} \leq 1$ then $g(t_{*}) \leq \frac{\frac{2}{3}}{1+\frac{1}{2} + \frac{2}{3} (1/2)^{2}} = 0.4 < 1/2$.
\end{proof}

\begin{proposition}\label{prop:tail_sum_norm_omega_alpha}
Consider $s >0$, and let $\omega_1, \dots, \omega_m$ be i.i.d. samples from $\mathcal{N}(0, \frac{1}{s^2}\mathbb{I}_d)$. Then
\begin{equation}\label{eq:fluctuation_omega3}
\forall \tau \geq 0, \:\:  \mathbb{P}\Big( \frac{1}{m}\sum\limits_{j=1}^{m} \|\omega_j\|^3 > \frac{4}{s^3}\big(d^{3/2}\sqrt{8/\pi} + \tau \big) \Big) \leq \exp\bigg( -\frac{(m\tau)^{2/3}}{2} \bigg).
\end{equation}
Moreover, as a consequence we have for $\tau > 0$
\begin{equation}\label{eq:bound_prop_psi_0_finalone}
\mathbb{P}\Big( \frac{1}{m}\sum\limits_{j=1}^{m} \sum_{t = 1}^{3} \|\omega_j\|_{2}^t > 1+ \frac{8}{s^3}\big(d^{3/2}\sqrt{8/\pi} + \tau \big) \Big) \leq \exp\bigg( -\frac{(m\tau)^{2/3}}{2} \bigg) .
\end{equation}

\end{proposition}

\begin{proof}
First,  to prove~\eqref{eq:fluctuation_omega3} when $\omega_1, \dots, \omega_m$ are i.i.d. samples from $\mathcal{N}(0, \frac{1}{s^2}\mathbb{I}_d)$, it is enough to deal with the case $s=1$. Next, for $s=1$, denoting $\Omega \in \mathbb{R}^{dm}$ the concatenation of $\omega_{1}, \dots, \omega_{m} \in \mathbb{R}^{d}$, the vector $\Omega$ has independent standard normal random entries when $\omega_{1}, \dots, \omega_{m}$ are i.i.d. samples from $\mathcal{N}(0,\mathbb{I}_d)$, and for any $2 \leq p < \infty$ the function $f_{p}: \mathbb{R}^{dm} \rightarrow \mathbb{R}$ defined by
$f_{p}(\Omega) := (\sum_{j=1}^{m}\|\omega_{j}\|_{2}^{p})^{1/p}$
is (as we will soon show) $1$-Lipschitz with respect to the Euclidean norm in $\mathbb{R}^{dm}$. Therefore, we may use the Tsirelson-Ibragimov-Sudakov inequality~\citep{TsIbSu76}, a.k.a. concentration of a random variable that writes as a Lipschitz function of a Gaussian vector \citep[Theorem 5.6]{BoLuMa13}, to obtain $\mathbb{P}( f_{p}(\Omega) - \mathbb{E}f_{p}(\Omega) \geq t ) \leq \exp( -t^2/2)$, for each $t \geq 0$, or equivalently
\begin{equation*} \forall t \geq 0, \:\: \mathbb{P}\Big( \big[f_{p}(\Omega)\big]^{p} \geq \big[\mathbb{E}f_{p}(\Omega) + \tau \big]^p \Big) \leq \exp\big( -t^2/2 \big).
\end{equation*}
By convexity we have $(a+b)^{p} = 2^{p}(a/2+b/2)^{p} \leq 2^{p-1}(a^{p}+b^{p})$ for every $a,b \in \mathbb{R}_{+}$, hence  
\begin{equation*}
\forall t \geq 0, \:\:  \mathbb{P}\Big( \big[f_{p}(\Omega)\big]^{p} \geq 2^{p-1}\big(\big[\mathbb{E}f_{p}(\Omega)\big]^p + t^p \big) \Big) \leq \exp\big( -t^2/2 \big). 
\end{equation*}
We now show that $\big[\mathbb{E}f_{p}(\Omega)\big]^p \leq md^{p/2}\mathbb{E}|g|^p$. The convexity of $t \mapsto t^{p/2}$ ($p \geq 2$) on $\mathbb{R}_{+}$ yields
\begin{align*}
\mathbb{E}\|\omega\|_{2}^p 
&= d^{p/2} \mathbb{E}\Big( \frac{1}{d}\sum\limits_{i=1}^{d} \omega_{i}^2 \Big)^{p/2} 
\leq d^{p/2} \mathbb{E} \left[\frac{1}{d}\sum\limits_{i=1}^{d} |\omega_{i}|^p\right]  = d^{p/2}\mathbb{E}|g|^{p}\\
\intertext{where $\omega \sim \mathcal{N}(0,\mathbb{I}_{d}),\ g \sim \mathcal{N}(0,1)$. B convexity of $t \mapsto t^{p}$ and Jensen's inequality, it follows that}
\big[\mathbb{E}f_{p}(\Omega)\big]^p 
& \leq \mathbb{E} \big[f_{p}(\Omega)\big]^p =
 \sum_{j=1}^{m}\mathbb{E}\|\omega_{j}\|_{2}^p=m\mathbb{E}\|\omega\|_{2}^p \leq md^{p/2} \mathbb{E}|g|^{p}.
\end{align*}
As a result
\begin{equation*} \forall t \geq 0, \:\: \mathbb{P}( \big[f_{p}(\Omega)\big]^{p} \geq 2^{p-1}(md^{p/2}\mathbb{E}|g|^{p} + t^p ) ) \leq \exp( -t^2/2 ),
\end{equation*} 
or equivalently $\mathbb{P}\Big( \frac{1}{m}\sum_{j=1}^{m} \|\omega_j\|^p \geq 2^{p-1}\big(d^{p/2}\mathbb{E}|g|^{p} + \tau \big) \Big) \leq \exp( -(m\tau)^{2/p}/2)$ for each $\tau \geq0$.
Since $\mathbb{E}|g|^{3} = \sqrt{8/\pi}$ \citep[Chapter 11]{EvHaPeFo11}, considering $p=3$ yields~\eqref{eq:fluctuation_omega3} (for $s=1$) as claimed.
Now, observe that
for $t \geq 0$ we have $t^{3}-t-(t^{2}-1) = (t^{2}-1)(t-1) = (t-1)^2(t+1) \geq 0$ hence $t^3+1 \geq t^2 +t$ and $t+t^2+t^3 \leq 1+2t^3$.  Thus $ \sum_{i = 1}^{3} \|\omega_j\|_{2}^i \leq 1 + 2\|\omega_j\|_{2}^3$, and we deduce
~\eqref{eq:bound_prop_psi_0_finalone}  from~\eqref{eq:fluctuation_omega3}.

To complete the proof, we show that $f_{p}$ is $1$-Lipschitz with respect to $\|\cdot\|_{2}$. Denoting $v_{\Omega}: = (\|\omega_j\|_{2})_{j\in [m]}  \in \mathbb{R}^{m}$, observe that $f_{p}(\Omega) = \|v_{\Omega}\|_{p}$. Thus, for $\Omega, \Omega' \in \mathbb{R}^{dm}$, since $p \geq 2$, we have
\begin{equation*}
\big|f_{p}(\Omega) - f_{p}(\Omega') \big|  = \big|\|v_{\Omega}\|_{p} - \|v_{\Omega'}\|_{p}\big|  \leq \big\|v_{\Omega} - v_{\Omega'}\big\|_{p}  \leq \big\|v_{\Omega} - v_{\Omega'}\big\|_{2}.
\end{equation*}
Finally,
\begin{equation*}
\big\|v_{\Omega} - v_{\Omega'}\big\|_{2}^{2}  
= \sum\limits_{j = 1}^{m} (\|\omega_j\|_{2}- \|\omega'_j\|_{2})^2 
\leq \sum\limits_{j = 1}^{m} \|\omega_j-\omega'_j\|_{2}^2  
= \|\Omega - \Omega' \|_{2}^{2}.
\end{equation*}

\end{proof}

\subsubsection{Proof of \Cref{lemma:very_useful_bound_on_exp}}
\label{sec:subexpas}

Denote $c  := 2(\alpha+2)$ and $\varphi(t):= e^{-(c-\alpha)t^{2}/2 +t}-te^{-ct^2/2}$.
Since $\varphi(0) = 1$, it is enough to prove that $\varphi$ is non-increasing on $\mathbb{R}_{+}$. Since $\varphi$ is $\mathcal{C}^{1}$, we study the sign of 
$\varphi'(t) = \Big( -(c-\alpha)t +1 + (ct^2 - 1)e^{-\alpha t^2/2 -t} \Big)e^{-(c-\alpha)t^2/2 +t}$
which is the sign of 
\(
\psi(t):= 1-(c-\alpha)t + (ct^2 - 1)e^{-\alpha t^2/2 -t}.
\)
To show that $\psi(t) \leq 0$ for each $t \in \mathbb{R}_{+}$ we study its sign on the intervals $(0,\frac{1}{c-\alpha})$ and $(\frac{1}{c-\alpha}, +\infty)$. As a preliminary we record that since $\alpha >0$, we have 
\begin{equation}\label{eq:RecordAlphaProp}
\sqrt{c}/(c-\alpha) = \sqrt{2(2+\alpha)}/(\alpha+4) \leq 1/2.
\end{equation}
\paragraph{Case of $t \in (0,\frac{1}{c-\alpha})$.} 
Since $1-(c-\alpha)t > 0$, we will get that $\psi(t) \leq 0$ if we show that
\begin{equation*}\label{eq:TmpPsiAlpha2}
 \frac{ 1-ct^2}{ 1-(c-\alpha)t}  \geq e^{\alpha t^2/2 +t}.
\end{equation*}
Since $t \in (0, 1/(c-\alpha))$, using~\eqref{eq:RecordAlphaProp} we have $\sqrt{c}t < (c-\alpha)t<1$ hence
\begin{equation*}\label{eq:alpha_1_simplification}
\frac{ 1-ct^2}{ 1-(c-\alpha)t}   = \frac{(  1-\sqrt{c}t)(1+\sqrt{c}t)}{ 1-(c-\alpha)t}  \geq 
1+\sqrt{c}t.
\end{equation*}
Denoting $h(u):= (1+\sqrt{c}u)e^{-\alpha u^2/2 -u}$, it is enough to prove that $h(t) \geq 1 = h(0)$, which will follow if we establish that $h'(u) = \Big(\sqrt{c}(1 -u -\alpha u^2 )-
(\alpha u +1)\Big)e^{-\alpha u^2/2 -u} \geq 0$ on $(0,1/(c-\alpha))$, or equivalently that the quadratic function 
$\sqrt{c}(1 -u -\alpha u^2 )-(\alpha u +1)$ takes non-negative values at $u=0$ and at $u=1/(c-\alpha)$. Indeed, its  evaluation at $u=0$ yields $\sqrt{c}-1 = \sqrt{2(2+\alpha)} - 1 > 0$, while its evaluation on $1/(c-\alpha) = 1/(\alpha+4)$ is lower bounded by $8/(\alpha+4)^2 >0$.

\paragraph{Case of $t \in (\frac{1}{c-\alpha}, +\infty)$.}
Since $1-(c-\alpha)t < 0$, we get  that $\psi(t) \leq 0$ as soon as
\begin{equation*}\label{eq:ineq_t_larger_alpha4}
\frac{ct^2 - 1}{(c-\alpha)t -1} \leq e^{\alpha t^2/2 +t}.
\end{equation*}
Since $t > 1/(c-\alpha)$, we have $(c-\alpha)t >1$, and using~\eqref{eq:RecordAlphaProp} we get
\begin{equation*}
\frac{\sqrt{c}t - 1}{(c-\alpha)t -1} \leq \frac{\sqrt{c}t}{(c-\alpha)t} \leq \frac{1}{2}.
\end{equation*}
Therefore
\begin{equation*}\label{eq:alpha_2_simplification}
\frac{ct^2 - 1}{(c-\alpha)t -1}   = \frac{(\sqrt{c}t - 1)(\sqrt{c}t + 1)}{(c-\alpha)t -1}  
\leq \frac{1}{2} + \frac{1}{2}\sqrt{c}t 
\leq \frac{e^{\alpha t^2/2 +t}}{2}+ \frac{\sqrt{c}t}{2}.
\end{equation*}
Denoting $g(u):= u e^{-\alpha u^2/2 -u}$, it is thus enough to show
that $\sqrt{c} g(t) \leq 1$ to conclude. Since $g'(u) = -(-1 +u +\alpha u^2 )e^{-\alpha u^2/2 -u}$,
the unique $u \geq 0$ such that $g'(u) = 0$ is $u_{\alpha} := 2/(\sqrt{4\alpha+1} + 1)$, which satisfies $\alpha u_{\alpha}^2+ u_{\alpha}-1=0 $, and
the maximum of  $g(u)$ on $\mathbb{R}_{+}$ is at $u=u_{\alpha}$. As a result
\begin{equation}\label{eq:g_alpha_12}
\sqrt{c}g(t) \leq  \sqrt{c}g(u_{\alpha}) 
= 2\frac{\sqrt{2(2+\alpha)}}{\sqrt{4\alpha+1} + 1}e^{-1/2-u_{\alpha}/2}.
\end{equation}
To conclude, we show that the r.h.s. is bounded by one, by distinguishing two cases.
On the one hand, if $\alpha \geq 2$, we have $2\sqrt{2(2+\alpha)}/(\sqrt{4\alpha+1} + 1) \leq \sqrt{2}$, and since $u_{\alpha} \geq 0$ the r.h.s. of~\eqref{eq:g_alpha_12} is upper bounded by $e^{-1/2}\sqrt{2} \leq 1$. On the other hand, if $\alpha \leq 2$, we 
have $u_{\alpha}/2 \geq 1/4$ and $2\sqrt{2(2+\alpha)}/(\sqrt{4\alpha+1} + 1) \leq 2$ (the latter inequality holds for any 
$\alpha \geq 0$), so that the r.h.s. of~\eqref{eq:g_alpha_12} is upper bounded by $2e^{-3/4} \leq 1$. In both cases, we get as claimed that $\sqrt{c}g_{\alpha}(t) \leq 1$. \qed

\newpage
\bibliographystyle{alpha}
\bibliography{references}
\end{document}